\documentclass{article}
\usepackage[utf8]{inputenc}
\usepackage[parfill]{parskip}    
\usepackage{graphicx,subfigure}
\usepackage[colorlinks=false]{hyperref} 
\usepackage{amssymb, amsmath, amsthm}
\usepackage{epstopdf}
\usepackage{multibib}
\usepackage{mathrsfs}
\usepackage[marginratio=1:1,height=600pt,width=460pt,tmargin=100pt]{geometry}

\usepackage{layout, color}
\usepackage{bbm}

\usepackage{authblk}

\usepackage{algorithm}
\usepackage{algpseudocode}

\usepackage{setspace}
\usepackage{lineno}

\newcommand{\R}{\mathbb{R}}
\newcommand{\E}{\mathbb{E}}

\newtheorem{theorem}{Theorem}[section]
\newtheorem{proposition}[theorem]{Proposition}
\newtheorem{assumption}{Assumption}

\newtheorem{lemma}[theorem]{Lemma}

\newtheorem{remark}[theorem]{Remark}
\newtheorem{conjecture*}{Conjecture}

\theoremstyle{plain}

\begin{document}

\title{Two-sample Statistics Based on Anisotropic Kernels}

\author[1]{Xiuyuan Cheng\footnote{The majority of this work was done while the first two authors were Gibbs Assistant Professors at Yale.}}
\author[2]{Alexander Cloninger$^*$}
\author[3]{Ronald R. Coifman}
\affil[1]{Department of Mathematics, Duke University}
\affil[2]{Department of Mathematics, University of California, San Diego}
\affil[3]{Department of Mathematics, Yale University}
\date{}

\maketitle

\begin{abstract}
	
The paper introduces a new kernel-based Maximum Mean Discrepancy (MMD) statistic 
for measuring the distance between two distributions given finitely-many multivariate samples. 
When the distributions are locally low-dimensional,
the proposed test can be made more powerful to distinguish certain alternatives
by incorporating local covariance matrices and constructing an anisotropic kernel.
The kernel matrix is asymmetric; it computes the affinity between $n$ data points and a set of $n_R$ reference points,
where $n_R$ can be drastically smaller than $n$.
While the proposed statistic can be viewed as a special class of Reproducing Kernel Hilbert Space MMD, 
the consistency of the test is proved, under mild assumptions of the kernel, 
as long as $\|p-q\| \sqrt{n} \to \infty $, 
and a finite-sample lower bound of the testing power is obtained.
Applications to flow cytometry and diffusion MRI datasets are demonstrated,
which motivate the proposed approach to compare distributions.

\end{abstract}

\section{Introduction}

We address the problem of comparing two probability distributions $p$ and $q$ from finite samples in $\R^d$, 
where both distributions are assumed to be continuous (with respect to Lebesgue measure) and compactly supported. 
We consider the case where each distribution is observed from i.i.d. samples,
called $X$ ($\sim p$) and $Y$ ($\sim q$)  respectively,
and the two datasets $X$ and $Y$ are independent. 
The methodology has applications in a variety of fields, 
particularly in bio-informatics.
It can be used, for example,
to test genetic similarities between subtypes of cancers,  
to compare patient groups to determine potential treatment propensity, 
and to detect small anomalies in medical images that are symptomatic of a certain disease.
We will cover applications to flow cytometry and diffusion MRI data sets in this paper. 

Due to the complicated nature of the datasets we would like to study, we are interested in the general alternative hypothesis test ${\cal H}_0: p = q $ against ${\cal H}_1: p\neq q $.  
This goes beyond tests of possible shifts of finite moments, for example, that of 
mean-shift alternatives namely $\E_{X\sim p}[X]\stackrel{?}{=} \E_{Y\sim q}[Y]$.
We also focus on the medium dimensional setting, in which $1<d \ll \min(n_1,n_2)$, where $n_1$ ($n_2$) is the number of samples in $X$ ($Y$).  
As $n \to \infty $, the dimension $d$ is assumed to be fixed. 
(For the scenario where $d \sim O(n)$, the convergence of kernel matrices to the limiting integral operator needs to be treated quite differently, and a new analysis is needed.)
We are particularly interested in the situation where data is sampled from distributions which are locally low-rank,
which means that  the local covariance matrices are generally of rank much smaller than $d$. 
As will be clear in the analysis, the approach of constructing anisotropic kernels is most useful when data
exhibits such characteristics in a high dimensional ambient space.

We are similarly concerned with a $k$-sample problem, in which the question is to determine the global relationships between $k$ distributions, each of which has a finite set of i.i.d. samples. 
This can be done by measuring a pairwise distance between any two distributions and combining these pairwise measurements to build a weighted graph between the $k$ distributions. Thus we focus on the two-sample test as the primary problem.

In the two-sample problem, the two data sets do not have any point-to-point correspondence, 
which means that they need to be compared in terms of their probability densities.
One way to do this is to construct ``bins" at many places in the domain, compute the histograms of the two datasets at these bins, and then compare the histograms. 
This turns out to be a good description of our basic strategy, 
and the techniques beyond this include using ``anisotropic gaussian" bins at every local point and ``smoothing" the histograms when needed. 
Notice that there is a trade-off in constructing these bins: 
if a bin is too small, which may leave no points in it most of the time, the histogram  will have a large variance compared to its mean. 
When a bin is too large, it will lose the power to distinguish $p$ and $q$ when they only differ slightly within the bin.
In more than one dimension, one may hope to construct anisotropic bins which are large in some directions and small in others, 
so as maximize the power to differentiate $p$ and $q$ while maintaining small variance.
It turns out to be possible when the deviation $(q-p)$ has certain preferred (local) directions. 
We illustrate this idea in a toy example below, 
and the analysis of testing consistency, including the comparison of different ``binning" strategies, 
is carried out by analyzing the spectrum of the associated kernels in Section \ref{sec:theory}.

At the same time, we are not just interested in whether $p$ and $q$ deviate, but how and where they deviate. 
The difference of the histograms of $X$ and $Y$ measured at multiple bins,
introduced as above, can surely be used as an indication of where $p$ and $q$ differ.  
The formal concept is known as the \emph{witness function} in literature,
which we introduce in Section \ref{subsec:witness} and use throughout the applications.

The idea of using bins and comparing histograms 
dates back to measuring distribution distances based on kernel density estimation, 
and is also closely related to two-sample statistics by Reproducing Kernel Hilbert Space (RKHS) Maximum-mean discrepancy (MMD).
We discuss these connections in more detail in Section \ref{sec:review}.
While anisotropic kernels have been previously studied in manifold learning and image processing,
kernel-based two-sample statistics with anisotropic kernels have not been examined in the situation where the data is locally low-rank,
which is the theme of the current paper.

This paper yields several contributions to  the two sample problem.  
Methodologically,
we introduce a kernel-based MMD statistic that increases the testing power against certain deviated distributions by adopting an anisotropic kernel, 
and reduces both the computation and memory requirements. 
The proposed method can be combined with spectral smoothing of the histograms
in order to reduce variability and possibly optimize the importance of certain regions of data, so that the power of the test maybe furtherly improved.
Theoretically, asymptotic consistency is proved for any fixed deviation beyond the critical regime $\|p-q\| \sim O(n^{-1/2})$
under generic assumptions. 
Experimentally, 
we provide two novel applications of two-sample and $k$-sample problems for biological datasets.

The rest of the paper is organized as follows: 
we begin with a sketch of the main idea and motivating example in the remainder of this section, 
together with a review of previous studies.
Section \ref{sec:2} formalizes the definition of the MMD statistics being proposed.
Asymptotic analysis  is given in Section \ref{sec:theory}.
The algorithm and other implementation matters, including computation complexity, 
are discussed in Section \ref{sec:practical}.
Section \ref{sec:applications} covers numerical experiments on synthetic and real-world datasets.

\subsection{Main Idea}\label{subsec:main-idea}

\begin{figure}
\footnotesize
	\begin{center}
		\begin{tabular}{ccccc}
			\includegraphics[height=.15\textwidth]{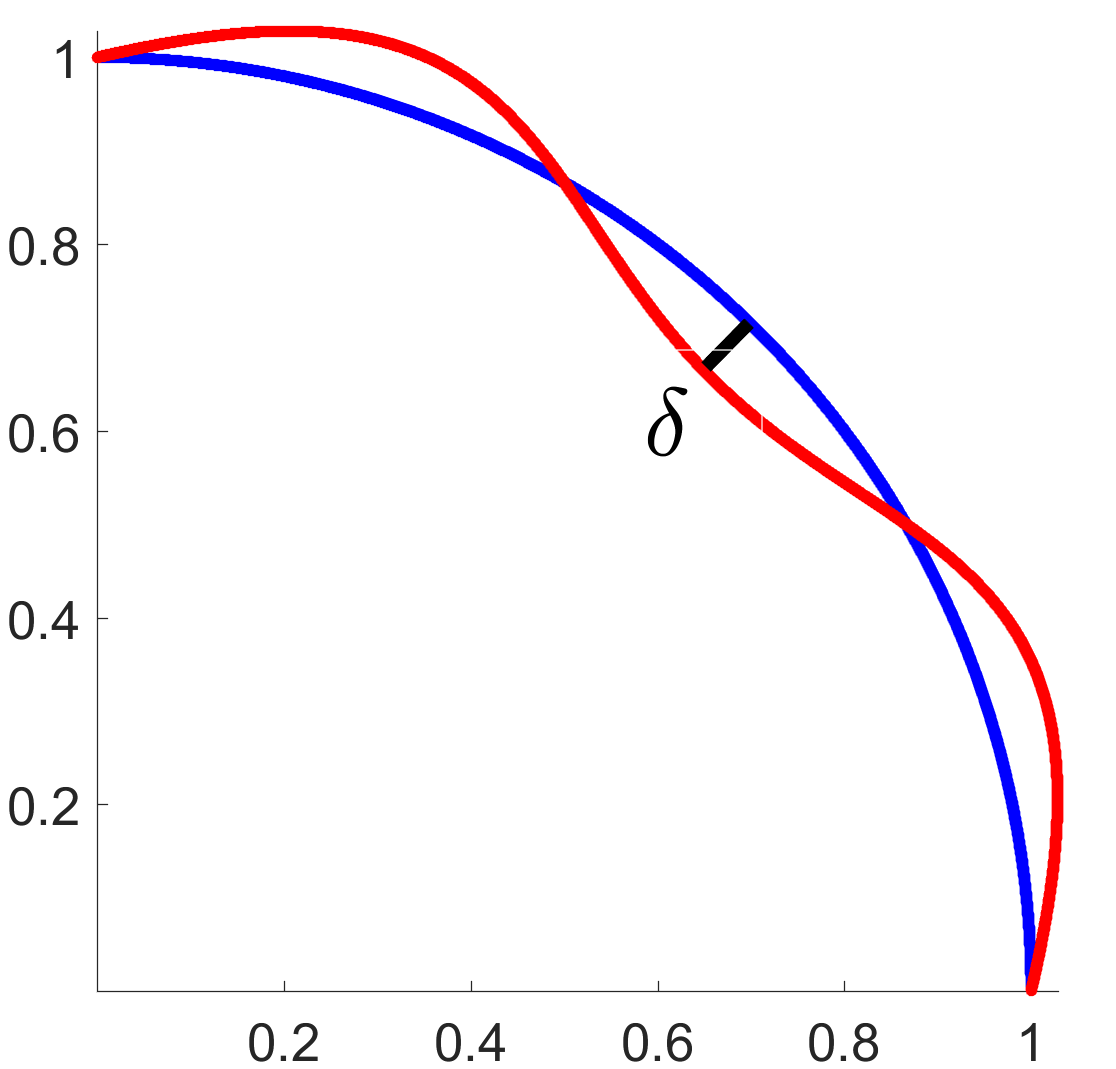} & \hspace{-.5cm}
			\includegraphics[height=.15\textwidth]{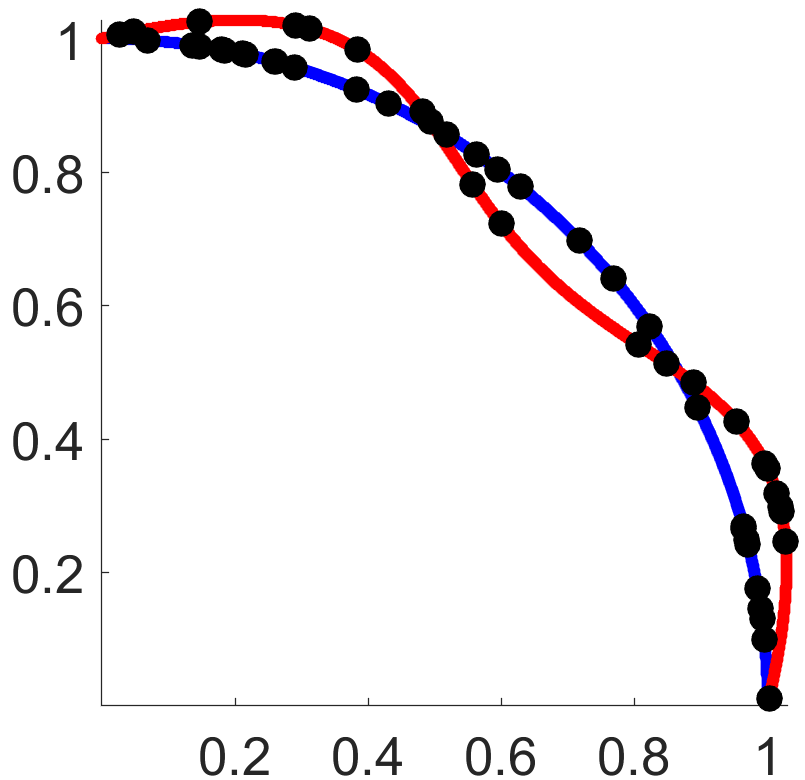} & \hspace{-.5cm}
			\includegraphics[height=.15\textwidth]{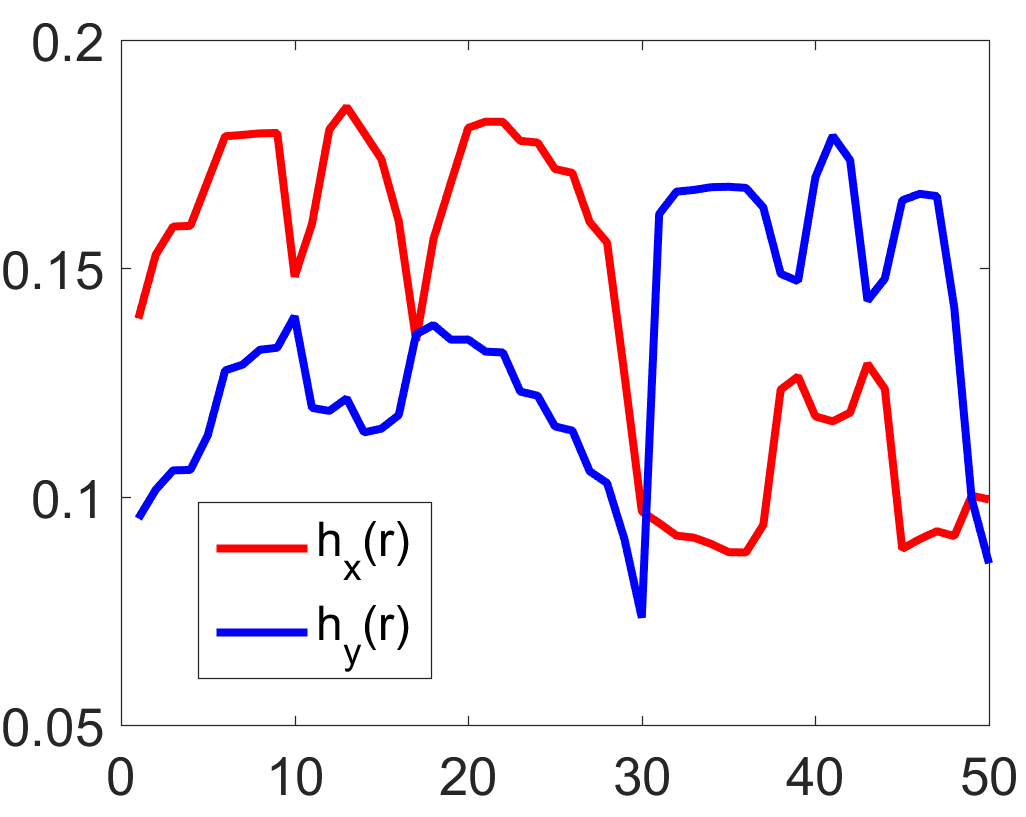} &
					\includegraphics[height=.15\textwidth]{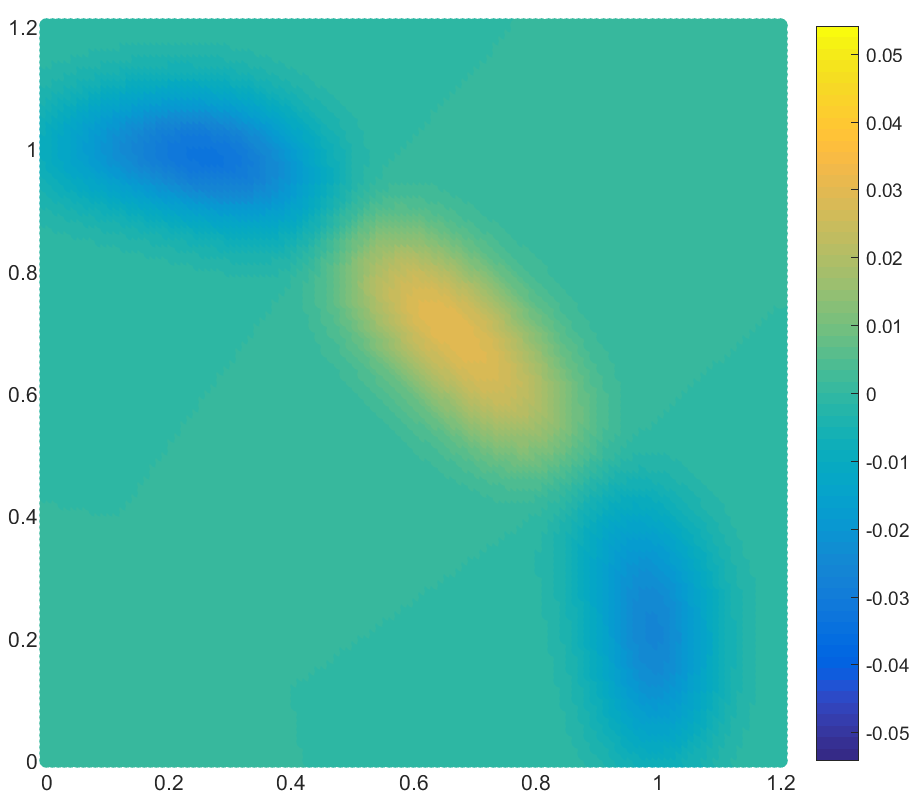} & \hspace{-1cm}
					\includegraphics[height=.15\textwidth]{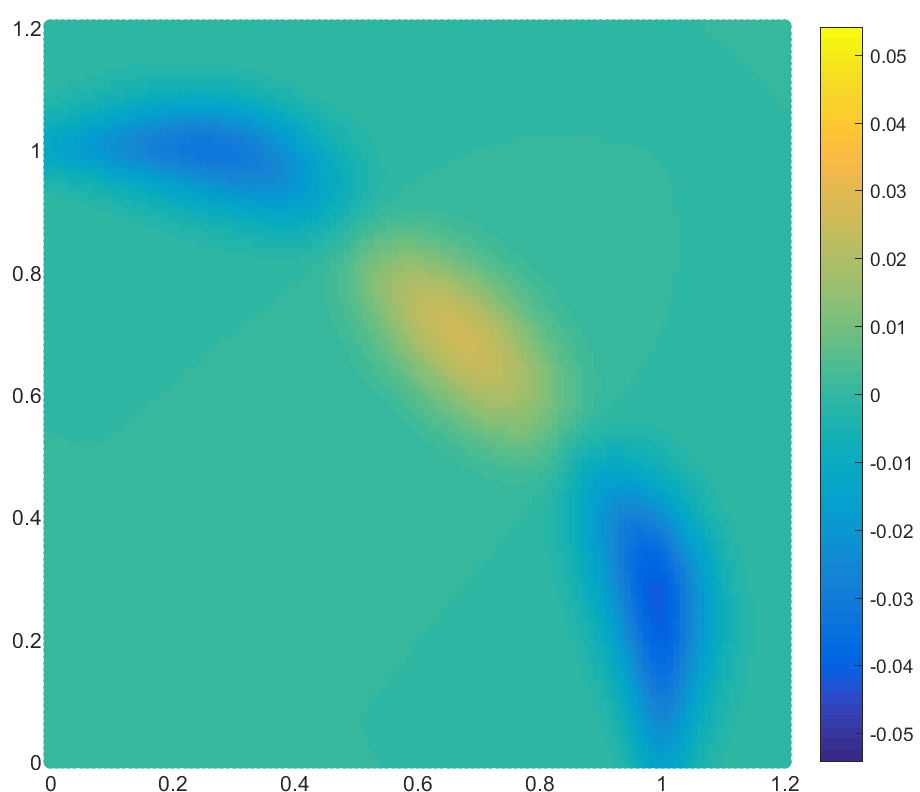} \\
					A1&  \hspace{-1cm} A2 & \hspace{-1cm} A3& 
							C1& \hspace{-1cm} C2\\
			\includegraphics[height=.15\textwidth]{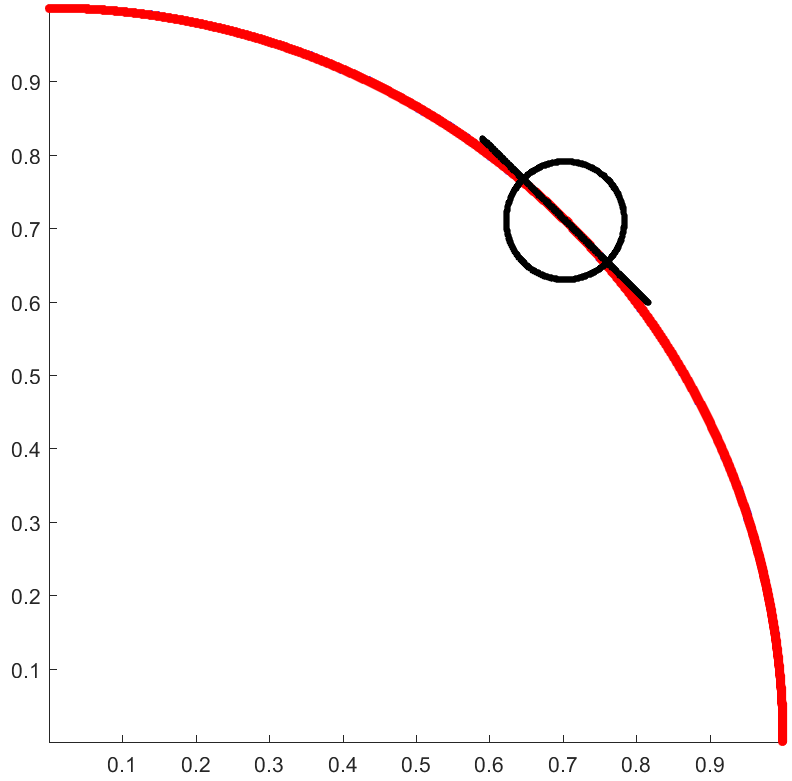} &  \hspace{-.5cm}
			\includegraphics[height=.15\textwidth]{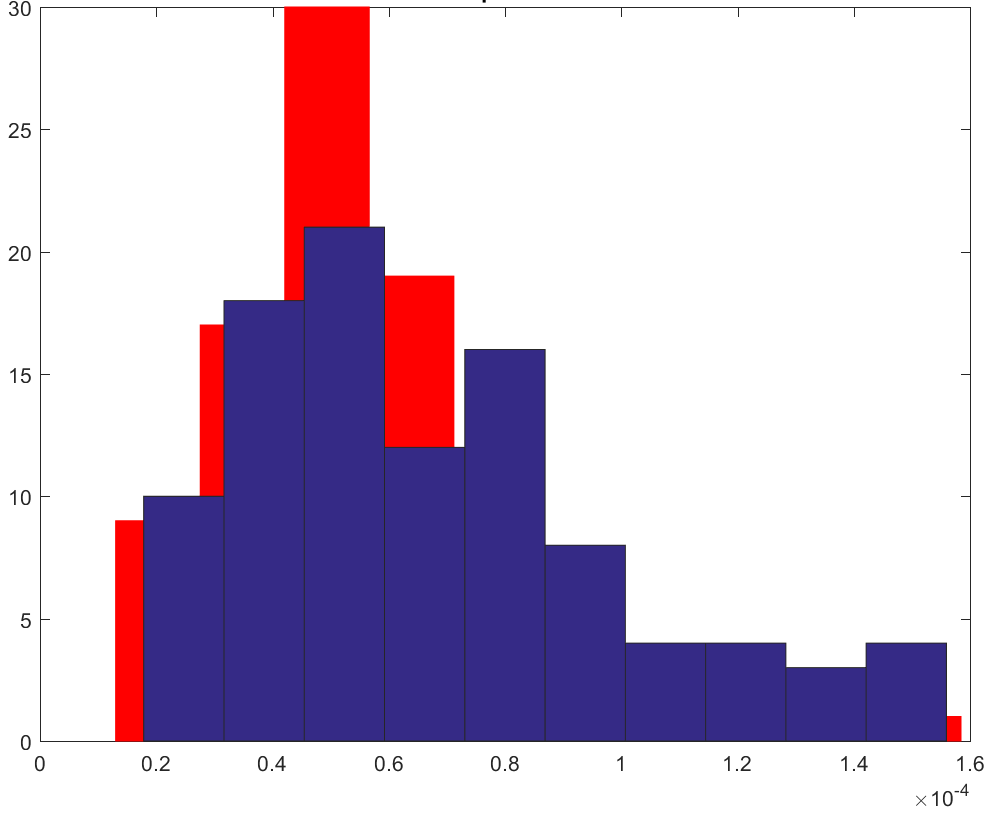}&  \hspace{-.5cm}
			\includegraphics[height=.15\textwidth]{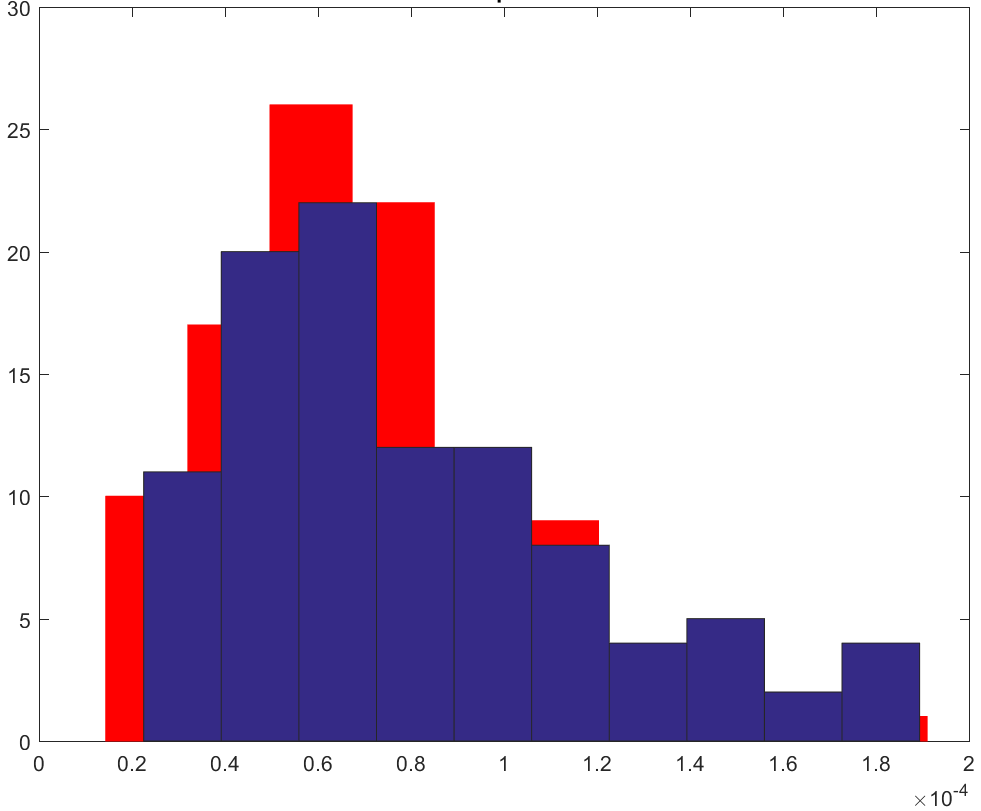} &
					\includegraphics[height=.15\textwidth]{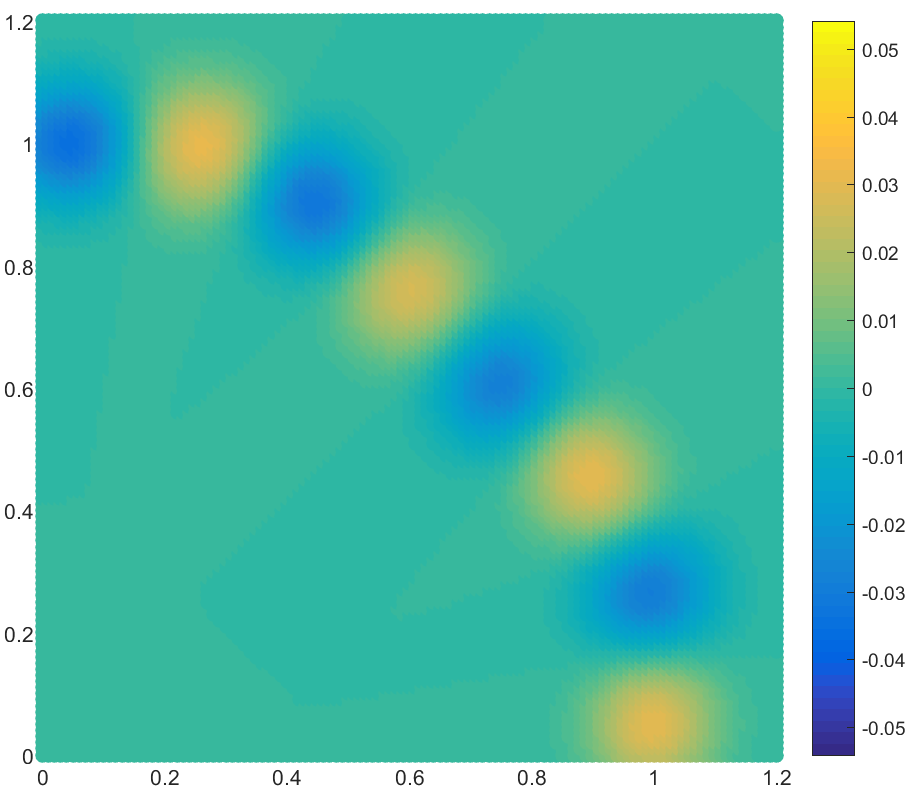} &\hspace{-1cm}
					\includegraphics[height=.15\textwidth]{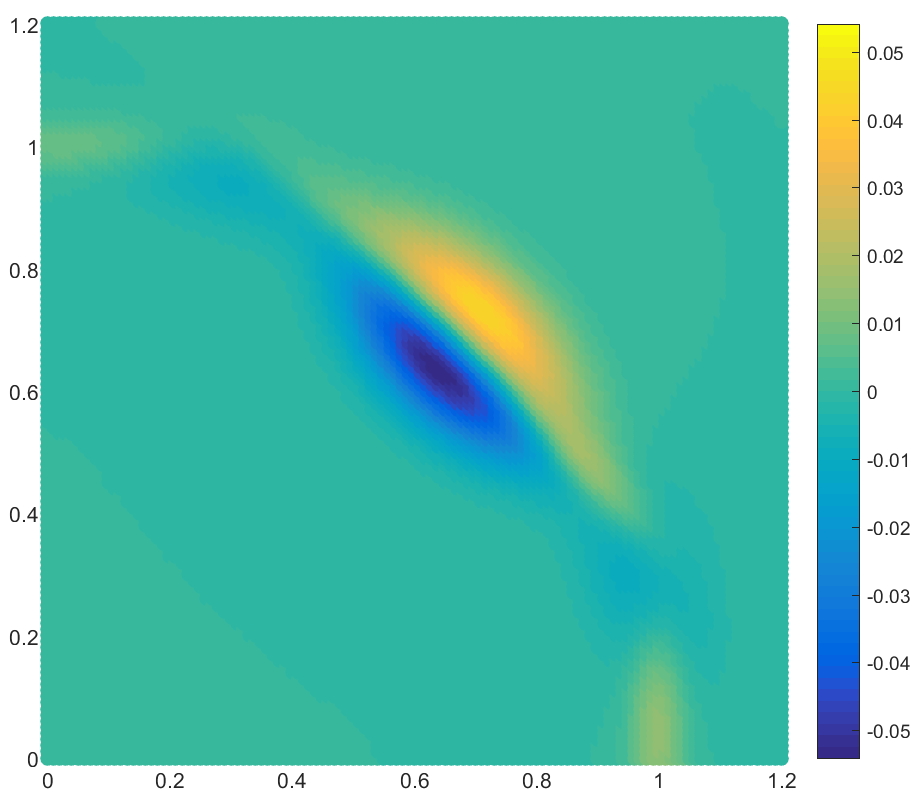} \\
					B1&  \hspace{-1cm} B2 & \hspace{-1cm} B3& 
							C3& \hspace{-1cm} C4\\
			\includegraphics[height=.15\textwidth]{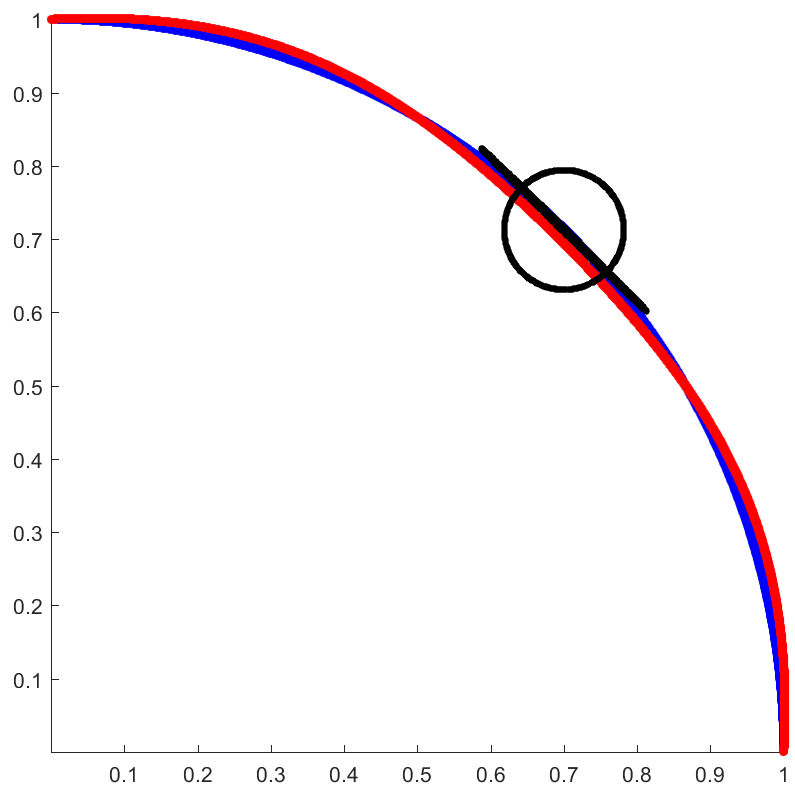} & \hspace{-.5cm}
			\includegraphics[height=.15\textwidth]{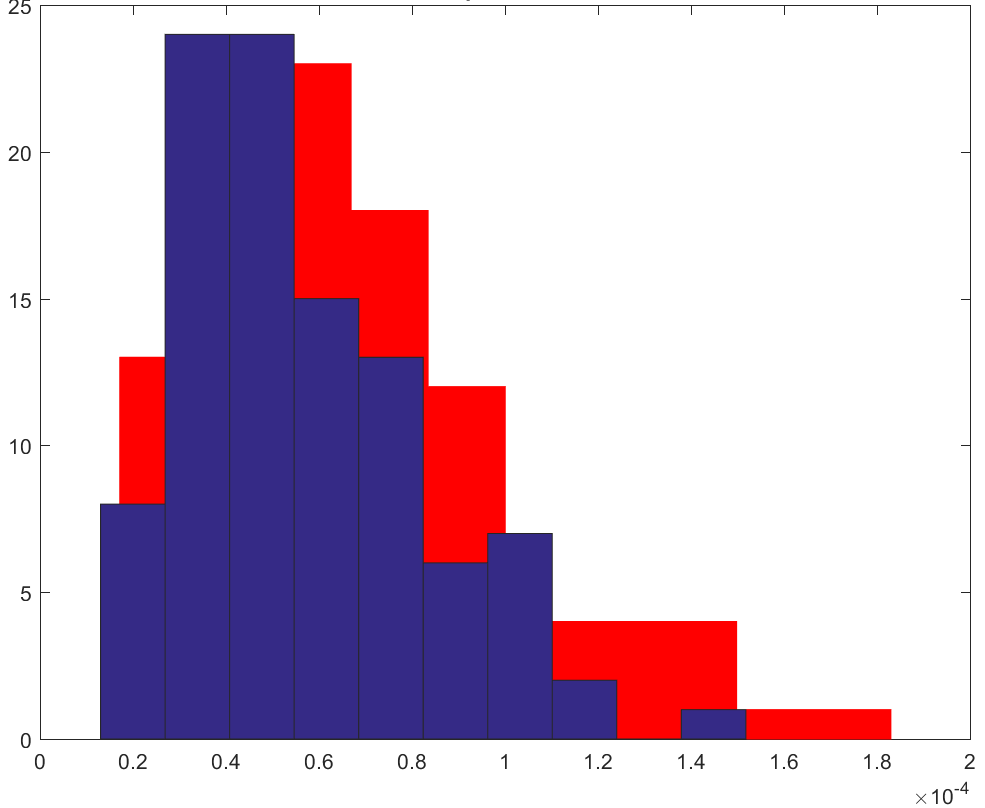}& \hspace{-.5cm}
			\includegraphics[height=.15\textwidth]{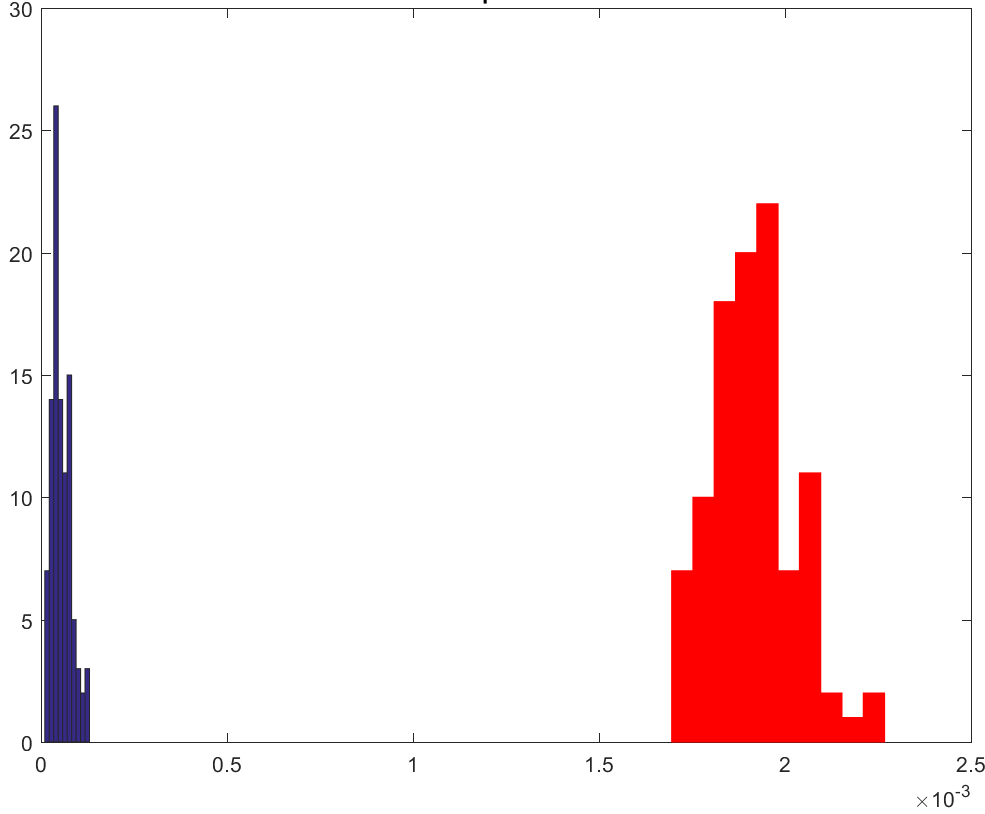} &
					\includegraphics[height=.15\textwidth]{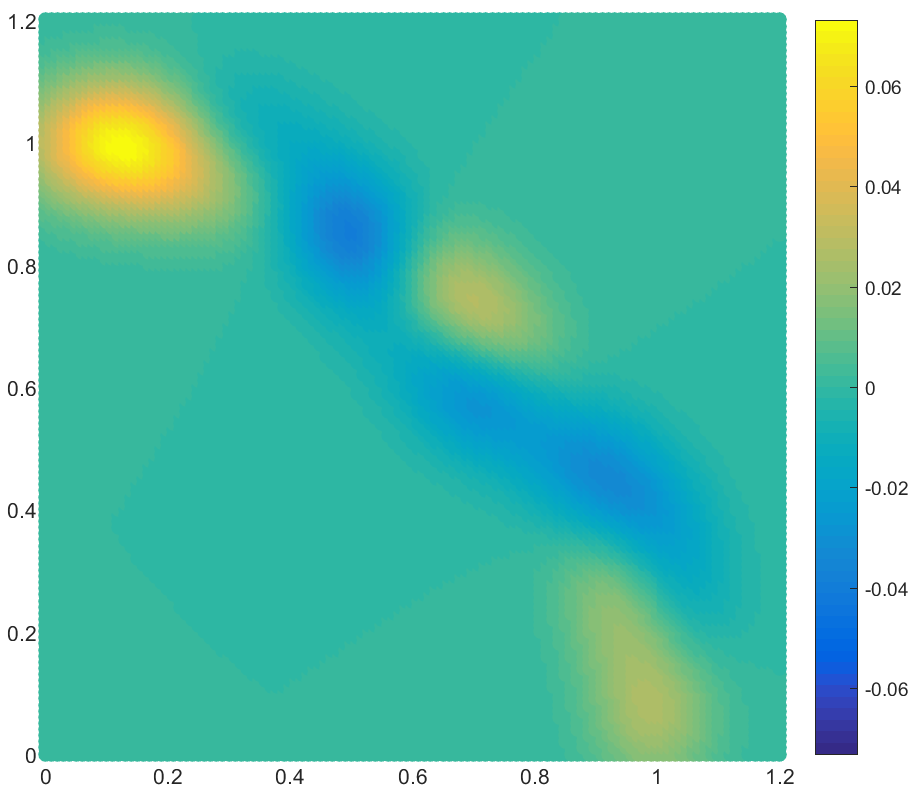}  &\hspace{-1cm}
					\includegraphics[height=.15\textwidth]{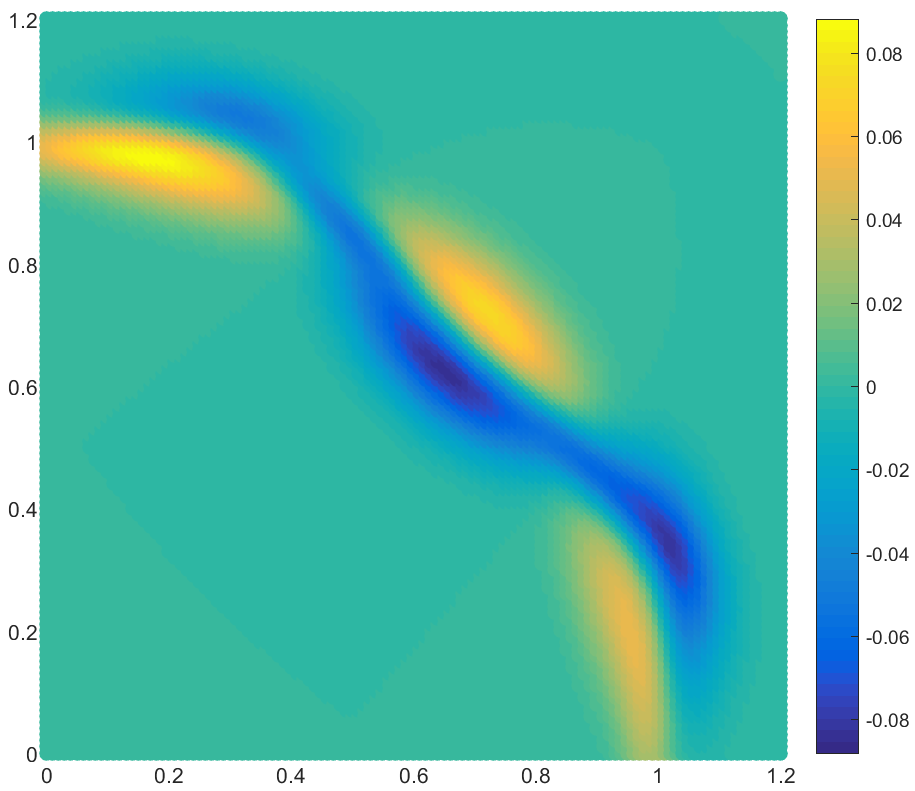}\\
					B4& \hspace{-1cm} B5 & \hspace{-1cm} B6& 
							C5& \hspace{-1cm} C6\\
		\end{tabular}
	\end{center}
	\caption{
		{\bf A1-A3}: 
		(left) $X$ and $Y$ sampled from curves with separation delta and $n_1 = n_2 = 500$, (middle) location of $50$ reference points, (right) $\hat{h}_X(r)$ and $\hat{h}_Y(r)$ for anisotropic kernel.
		{\bf B1-B6}: 
		(left) representative isotropic and anisotropic bins, (middle) histograms $T_0$ (blue) and $T_1$ (red) with isotropic kernel, 
(right) histograms with anisotropic kernel, (top) $\delta=0$, (bottom) $\delta=0.01$.
		{\bf C1-C6}: 
		(left) isotropic kernel, (right) anisotropic kernel, (top) Heatmap of $\psi_3$ as singular functions of kernel, (middle) Heatmap of $\psi_8$, (bottom) witness function.  $\delta=0.05$, and $\psi_k$ agree up to $k=6$.
%
	}\label{fig:twoCircleData}
\end{figure}

Let $p$ and $q$ be two distributions supported on a compact set $\Omega \subset \R^d$.
Suppose that a reference set $R$ is given, and for each point $r\in R$ there is a (non-degenerate) covariance matrix $\Sigma_r$ (e.g. computed by local PCA). We define the asymmetric affinity kernel to be 
\begin{equation}
\label{eq:asymKernel}
a(r,x)=e^{-\|r-x\|^2_{\Sigma_r}}
=\exp\left\{-\frac{1}{2}(x-r)^T\Sigma_r^{-1}(x-r)\right\},
\quad \forall r\in R, x\in \Omega.
\end{equation}
Consider the two independent datasets $X \sim p$ and $Y \sim q$, where $X$ has $n_1$ i.i.d. samples and $Y$ has $n_2$ i.i.d. samples.
The empirical histograms of $X$ and $Y$ at the reference point $r$ are defined as 
\begin{equation}\label{eq:hathXhY}
\hat{h}_X(r) =\frac{1}{n_1}\sum_{i=1}^{n_1} a(r,X_i), \quad
\hat{h}_Y(r) =\frac{1}{n_2}\sum_{j=1}^{n_2} a(r,Y_j), 
\end{equation}
for which the population quantities are
\begin{equation}\label{eq:hphq}
h_p(r) = \int a(r,x)p(x)dx, \quad
h_q(r) =  \int a(r,y)q(y)dy.
\end{equation}
The empirical histograms are nothing else but the Gaussian binning of $X$ and $Y$ at point $r$ with the anisotropic bins corresponding to the covariance matrix $\Sigma_r$. 
We then compute the quantity
\begin{eqnarray}\label{eq:asymMMD}
T = \frac{1}{n_R}\sum_{r \in R} ( \hat{h}_X(r) - \hat{h}_Y(r))^2
\end{eqnarray}
as a measurement of the (squared) distance between the two datasets.

We use the following example to illustrate the difference between 
(1) using anisotropic kernel where $\Sigma_r$ is aligned with the tangent space of the manifold data, 
and 
(2) using the isotropic ones where $\Sigma_r$ is a multiple of the identity matrix. 

The data is like in Figure \ref{fig:twoCircleData}, 
where $p$ and $q$ are supported on two arcs in $\R^2$ separated by a gap of size $\delta$ at various regions.  
We begin by specifying a reference set $R$ and the covariance field $\{\Sigma_r\}_{r\in R}$. 
For simplicity, we do this by uniformly sampling $n_R = 50$ reference points from $X \cup Y$ (see Figure \ref{fig:twoCircleData}). 
At each reference point $r$, we take $k = 50$ neighbors $\{x_{i_j}\}_{j=1}^{k}$ to estimate the local covariance matrix
by $\Sigma_r = \frac{1}{k} \sum_{i=1}^k (x_{j_i} - r)(x_{j_i} - r)^T$. 
The empirical histograms $\hat{h}_X(r)$ and  $\hat{h}_Y(r)$ are computed as in Eqn. (\ref{eq:hathXhY}) at every point $r$ (see Figure \ref{fig:twoCircleData}),
as well as the quantity $T$ as in Eqn. \eqref{eq:asymMMD}. 
We also compute $T$ under a permutation of the data points in $X$ and $Y$ so as to mimic the null hypothesis $p=q$,
and we call the two values $T_1$ and $T_0$ respectively.
The experiment is repeated 100 times, where $n_1=n_2=2000$, 
and the distribution of $T_1$ and $T_0$ are shown as red and blue bars in Figure \ref{fig:twoCircleData}.
The simulation is done across three datasets where $\delta$ takes value as $\{0, 0.01, 0.05\}$,
and we compare  isotropic and anisotropic kernels. 
When $\delta=0$,  the distributions of $T_0$ and $T_1$ overlay each other, as expected.
When $\delta>0$, greater separation between distributions of   $T_0$ and $T_1$ implies greater power for the test.
The advantage of the anisotropic kernel is clearly demonstrated, particularly when $\delta = 0.01$ (the middle row).

The analysis of the testing power of $T$ hinges on the singular value decomposition of $a(r,x)$,
formally written as  $a(r,x) = \sum_k \sigma_k \phi_k(r) \psi_k(x)$ and will be defined in Section \ref{sec:theory}.
The histogram $\hat{h}_X(r)$ 
thus becomes 
$$
\hat{h}_X(r)
= \frac{1}{n_1}\sum_{i=1}^{n_1} a(r,X_i)
= \sum_k \sigma_k \phi_k(r) \left( \frac{1}{n_1} \sum_i \psi_k(x_i) \right),
$$
and similarly for $\hat{h}_Y(r)$,
which means that the ability of $T$ to distinguish $p$ and $q$ is determined by the amount that $X$ and $Y$ differ when projected onto the singular functions $\psi_k$. 
For the example above, 
the first few singular functions are visualized in Figure \ref{fig:twoCircleData}, 
where the $\psi_k$'s of the anisotropic kernel project along the directions where $q$ and $p$ deviate at a lower index of $k$, 
thus contributing more significantly to the quantity $T$.  
Figure \ref{fig:twoCircleData} also shows the ``witness function'' 
(\cite{gretton2012kernel}, c.f. Section \ref{subsec:witness}) 
of kernels, 
which indicates the regions of deviation between $p$ and $q$.

Throughout the paper, we refer to the use of a local Mahalanobis distance with $\Sigma_r\neq c \cdot I$ as an anisotropic kernel and $\Sigma_r = c\cdot I$ as an isotropic kernel.  
Similarly, we refer to a kernel measuring affinity from all points to a reference set (i.e. $\R^d\times \Omega\rightarrow \R$) as an asymmetric kernel. 
The analysis in Section \ref{sec:theory} is for the symmetrized version of the kernel $k(x,y) = \int a(r,x) a(r,y) d\mu_R(r)$,
while in practice one never computes the $n$-by-$n$ kernel 
but only the $n_R$-by-$n$ asymmetric kernel $a(r,x)$ which is  equivalent and way more efficient. 
We discuss more about computation in Section \ref{sec:compcomplexity}.

\subsection{Related Work}\label{sec:review}

The question of two sample testing is a central problem in statistics.  In one dimension, one classical approach to two sample testing is the Kolmogorov-Smirnov distance, which compares the $L^\infty$ distance between the two empirical cumulative distribution functions
\cite{kolmogorov,smirnov1948table}.  While there exist generalizations of these infinitely supported bins in high dimensions \cite{bickel1969distribution,friedman1979multivariate, rosenbaum2005exact,hall2002permutation}, these require a large computational cost for either computing a minimum spanning tree or running a large number of randomized tests.  This warranted binning functions that are defined more locally and in a data-adaptive fashion.  Another high-dimensional extension of Kolmogorov-Smirnov is to randomly project the data into a low-dimensional space and compute the test in each dimension.  

The 1D Kolmogorov-Smirnov statistic can be seen as a special case of the MMD discrepancy, which is generally defined as 
\[
\text{MMD}(p,q; {\cal F}) = \sup_{f \in {\cal F}} \int f(x) (p(x)-q(x))dx,
\]
where ${\cal F}$ is certain family of integrable functions. When ${\cal F}$ equals the set of all indicator functions of intervals $(-\infty, t)$ in $\R$, the MMD discrepancy gives the Kolmogorov-Smirnov distance.
Kernel-based MMD has been studied in \cite{gretton2012kernel}, where the function class ${\cal F}$ consists of all functions s.t. $\|f\|_{\cal H} \le 1$, where $\|\cdot \|_{\cal H}$ indicates the norm of the Hilbert space associated with the reproducing kernel. Specifically, suppose the PSD kernel is $k(x,y)$, the (squared) RKHS MMD can be written as 
\begin{equation}\label{eq:MMD2RKHS}
\text{MMD}^2(p,q) = \int\int k(x,y) (p(x)-q(x)) (p(y)-q(y))dx dy,
\end{equation}
and can be estimated by (here we refer to the biased estimator in \cite{gretton2012kernel} which includes the diagonal terms)
\begin{eqnarray*}
	\text{MMD}^2(X,Y) 
	= \frac{1}{n_1^2}\sum_{x,x'\in X} k(x,x') + \frac{1}{n_2^2}\sum_{y,y'\in Y} k(y,y') - \frac{2}{n_1n_2}\sum_{x\in X,y\in Y} k(x,y).
\end{eqnarray*}
The methodology and theory apply to any dimensional data as long as the kernel can be evaluated.

We consider a kernel $k(x,y)$ of the form $k(x,y)=\int a(r,x)a(r,y)d\mu_R(r)$ and its variants, where $d\mu_R$ is certain measure of the reference points.  This can be seen as a special case of RKHS MMD \cite{gretton2012kernel}, which considers a general PSD kernel $k:X\times X \rightarrow \R^+$.
When $a(r,x) = a(r-x)$ and is isotropic,  $d\mu_R$ is the Lebesgue measure, 
$k(x,y)$ is reduced to gaussian kernel.  However, returning to the asymmetric kernel as we do, allows us to easily redefine the local geometry around reference points and incorporate the local dimensionality reduction in \eqref{eq:asymKernel}.  
While the asymmetric kernel requires the additional technical assumption that the eigenmodes of the kernel do not vanish on the support of $p$ and $q$, 
which is discussed in \cite{gretton2012kernel} in reference to isotropic Parzen windows, 
the construction yields a more powerful test for distributions that are concentrated near locally low-dimensional structures.  
Theoretically, our analysis gives a comparison of the testing power of different kernels,
which turns out to be determined by the spectral decomposition of the kernels.  
This analysis augments that in \cite{gretton2012kernel}:
the asymptotic results in \cite{gretton2012kernel} does not imply testing power in high dimensions, 
as pointed out later by \cite{ramdas2015decreasing}.
Our analysis makes use of the spectral decomposition of the kernel with respect to the data distribution. 
While the empirical spectral expansion of translation-invariant kernels
has been previously used to derive two-sample statistics,  
e.g. in \cite{fernandez2008test} and more recently in \cite{chwialkowski2015fast, zhao2015fastmmd}, 
and the idea dates back to earlier statistical works (see e.g. \cite{epps1986omnibus}), 
our setting is different due to the new construction of the kernel.

We generalize the construction by considering a family of kernels 
via a ``spectral filtering'',
which truncates the spectral decomposition of the anisotropic kernels and modifies the eigenvalues, 
c.f. Section \ref{sec:specfilter}.
The modified kernel may lead to improved testing power for certain departures.
The problem of optimizing over kernels has been considered by \cite{gretton2012optimal}, 
where one constructs a convex combination of a finite number of kernels drawn from a given family of kernels.  The family of kernels considered in \cite{gretton2012optimal} are isotropic kernels, and possibly linear time computed kernels with high variance, which has the effect of choosing better spectral filters for separating two distributions.  However, because the kernels we consider are anisotropic, they lie outside the family of kernels considered in \cite{gretton2012optimal} and, in particular, have fundamentally different eigenfunctions over which we build linear combinations.  Also, building spectral filters directly on the eigenfunctions yields a richer set of filters than those that can be constructed by finite convex combinations.

Our approach is also closely related to the previous study of the distribution distance based on kernel density estimation \cite{anderson1994two}. 
We generalize the results  in \cite{anderson1994two} by considering non-translation-invariant kernels, which greatly increases the separation between the expectation of $T_n$ under the null hypothesis and the expectation of $T_n$ under an alternative hypothesis.  
Moreover, it is well-known that kernel density estimation, which \cite{anderson1994two} is based on, converges poorly in high dimension. 
In the manifold setting, the problem was remedied by normalizing the (isometric) kernel in a modified way 
and the estimation accuracy was shown to only depend on the intrinsic dimension \cite{ozakin2009submanifold}. 
Our proposed approach takes extra advantage of the locally low dimensional structure, and obtains improved distinguishing power compared to the one using isotropic kernels when possible.

At last, the proposed approach can be viewed as related to two sample testing via nearest neighbors \cite{henze1988multivariate}.  
In \cite{henze1988multivariate}, one computes the nearest neighbors of a reference point $r$ to the data $X\cup Y$ and derives a statistical test based on the amount the empirical ratio $\frac{k_X}{k_X + k_Y}$, where $k_X$ is the number of neighbors from $X$ (similarly $k_Y$), deviates from the expected ratio under the null hypothesis, namely $\frac{n_1}{n_1 + n_2}$.  Because the nearest neighbor algorithm is based on Euclidean distance, it is equivalent to a kernel-based MMD with a hard-thresholded isotropic kernel. The approach can be similarly combined with a local Mahalanobis distance as we do, which has not been explored.

%

\section{MMD Test Statistics and Witness Functions}\label{sec:2}

Given two independent datasets $X$ and $Y$, 
where $X$ has $n_{1}$ i.i.d. samples drawn from distribution
$p$, and $Y$ has $n_{2}$ i.i.d samples drawn from $q$,
we aim to 
(i) test the hypothesis $p = q$ against the alternative, 
and 
(ii) when $q\neq p$, 
detect where the two distributions differ. 
We assume that $p$ and $q$ are supported on compact subset of $\mathbb{R}^{d}$,
and both distributions have continuous probability densities,
so we also use $p$ and $q$ to denote the densities and write integration w.r.t. $dp(x)$ as $p(x)dx$
and similarly for $q$.

As suggested in Section \ref{subsec:main-idea},
the reference set $R$ and the covariance field $\{\Sigma_r \}_{r \in R}$ are important for the construction of the anisotropic kernel. 
In this section and next,
we assume that $R$ is given and $\{\Sigma_r \}_r$ is pre-defined. 
In practice, 
$R$ will be computed by a preprocessing procedure,
and $\Sigma_r$ can be estimated by local PCA if not given  {\it a priori}, 
c.f. Section \ref{sec:practical}.

\subsection{Kernel MMD statistics}

Using the  kernel $a(r,x)$ defined in  \eqref{eq:asymKernel}, 
we consider the following empirical statistic: 
\begin{equation}
\hat{T}_{L^{2}} =\int|\hat{h}_{X}(r)-\hat{h}_{Y}(r)|^{2}d\mu_{R}(r),
\label{eq:hatT-l2}
\end{equation}
where $\hat{h}_{X}(r)$ and $\hat{h}_{Y}(r)$ are defined in \eqref{eq:hathXhY}. 
Note that \eqref{eq:hatT-l2} assumes the measure $\mu_{R}$ 
along with the  covariance field $\{\Sigma_r \}_{r}$ needed in \eqref{eq:asymKernel}.  
$\mu_R$ can be any distribution in general,
and in practice, 
it is an empirical distribution over the finite set $R$, i.e. $d\mu_{R}=\frac{1}{n_{R}}\sum_{r\in R}\delta_{r}$,
where $n_{R}$ is the number of points in $R$. 
For now we leave $\mu_{R}$ to be general.
The population statistic corresponding to 
\eqref{eq:hatT-l2} 
is  
\begin{equation}\label{eq:def-T-L2}
T_{L^{2}}=\int|h_{p}(r)-h_{q}(r)|^{2}d\mu_{R}(r).
\end{equation}
$\hat{T}_{L^{2}}$ can be viewed as a special form of RKHS MMD: 
by \eqref{eq:MMD2RKHS}, 
\eqref{eq:hatT-l2} is the (squared) RKHS MMD with the kernel 
\begin{equation}\label{eq:kernel2}
k_{L^{2}}(x,y)=\int a(r,x)a(r,y)d\mu_{R}(r).
\end{equation}
The kernel \eqref{eq:kernel2} is 
clearly
positive semi-definite (PSD), 
however, not necessarily ``universal'', 
meaning
that the population MMD as in \eqref{eq:def-T-L2} 
being zero does not guarantee
that $q=p$.
The test is thus restricted to the departures within the Hilbert space
(Assumption \ref{assump:A2}). 

We introduce a spectral decomposition of the kernel $k_{L^2}$ based upon that of the asymmetric kernel $a(r,x)$,
which sheds light on the analysis:
Let $d\mu(x)$ be a distribution of data point $x$
(which is a mixture of $p$ and $q$ to be specified later).
Since $a(r,x)$ is bounded by 1, 
so is the integral $\int\int a(r,x)^{2}d\mu_{R}(r)d\mu(x)$,
and thus the asymetric kernel is Hilbert-Schmidt and the integral
operator is compact. 
The singular value decomposition of $a(r,x)$
with respect to $d\mu_{R}$ and $d\mu$ can be written as 
\begin{equation}\label{eq:svd-a}
a(r,x)=\sum_{k}\sigma_{k}\phi_{k}(r)\psi_{k}(x),
\end{equation}
where $\sigma_{k}>0$, $\{\phi_{k}\}_{k}$ and $\{\psi_{k}\}$ are
ortho-normal sets w.r.t $d\mu_{R}(r)$ and $d\mu(x)$ respectively.
Then \eqref{eq:kernel2} can be written as 
\begin{equation}\label{eq:kernel-L2}
k_{L^{2}}(x,y)=\sum_{k}\sigma_{k}^{2}\psi_{k}(x)\psi_{k}(y).
\end{equation}
This formula suggests that the ability of the kernel MMD to distinguish
$p$ and $q$ is determined by 
(i) how discriminative the eigenfunctions $\psi_k$ are (viewed as ``feature extractors''),
and 
(ii) how the spectrum $\sigma_k^2$ decay (viewed as ``weights'' to combine the $L^2$ differences extracted per mode).
It also naturally leads to generalizing the definition by modify the weights $\sigma_k^2$, which is next.

\subsection{Generalized kernel and spectral filtering}\label{sec:specfilter}

We consider the situation 
where the distributions $p$ and $q$ lie around certain lower-dimensional manifolds in the ambient space,
and both densities are smooth with respect to the manifold and decay off-manifold,
which is typically encountered in the applications of interest (c.f. Section \ref{sec:applications}). 
Meanwhile, since the reference set is sampled near the data,
$\mu_{R}$ is also centered around the manifold. 
Thus one would expect the population histograms $h_{p}(r)$ and
$h_{q}(r)$ to be smooth on the manifold as well.
This suggests building a ``low-pass-filter'' for the empirical
histograms 
before computing the $L^{2}$ distance between them, namely the MMD statistic.

We thus introduce a general form of kernel as  
\begin{equation}\label{eq:kernel-spec}
k_{\text{spec}}(x,y)=\sum_{k}f_{k}\psi_{k}(x)\psi_{k}(y)
\end{equation}
where $f_{k}$ is a sequence of sufficiently decaying positive numbers,
the requirement to be detailed in Section \ref{sec:theory}.
Our analysis will be based on kernels in form of \eqref{eq:kernel-spec},
which includes $k_{L^2}$ as a special case when $f_k = \sigma_k^2$.
While the eigenfunctions $\psi_k$ are generally not analytically available,
to compute the MMD statistics one only needs
to evaluate $\psi_k$'s on the data points in $X \cup Y$, 
which can be approximated by the 
empirical singular vectors of the $n_{R}$-by-$(n_{1}+n_{2})$ kernel matrix
$\{ a(r,x) \}_{r \in R, \, x \in X \cup Y}$
and computed efficiently for MMD tests (c.f.  Section \ref{sec:practical}). 
Note that the approximation by empirical spectral decomposition may degrade as $k$ increases,
however,
for the purpose of smoothing histograms one typically assigns small values
of $f_{k}$ for large $k$ so as to suppress the ``high-frequency components''.

The construction \eqref{eq:kernel-spec} gives a large family of kernels and is versatile:
First, 
setting $f_k=\sigma_k^{2(m+1)}$
for some positive integer $m$
is equivalent to using the kernel as
$
k(x,y)  = \int \int L(r,r') a(r,x) a(r,y) d\mu_R(r)  d\mu_R(r')
$
where 
$L(r,r') = ( \int a(r,x)a(r',x)d\mu(x) )^m$.
This can be interpreted as redefining the affinity of points $x$ and $y$ by allowing $m$-steps of
``intermediate diffusion'' on the reference set, and thus ``on the data'' \cite{coifman2005geometric}.
When $d \mu_R$ is uniform over the whole ambient space, 
raising $m$ is equivalent to enlarging the bandwidth of the gaussian kernel. 
However, when $\mu_R$ is chosen to adapt to the densities $p$ and $q$,
then the kernel becomes a data-distribution-adapted object. 
As $m$ increases, $f_k$ decays rapidly, 
which ``filters out" high-frequency components in the histograms when computing the MMD statistic,
because 
$
T_n = \int \int L(r,r') (\hat{h}_X(r)-\hat{h}_Y(r))
				(\hat{h}_X(r')-\hat{h}_Y(r')) d\mu_R(r)  d\mu_R(r')$.
Generally, setting $f_k$ to be a decaying sequence has the effect of spectral filtering.
Second,
in the case that prior knowledge about the magnitude of the projection
$\int(p(x)-q(x))\psi_{k}(x)d\mu(x)$ is available, 
one may also choose $f_{k}$ accordingly to select the ``important modes''. 
Furthermore, one may view the kernel MMD with \eqref{eq:kernel-spec} 
as a weighted squared-distance statistics after projecting to the spectral coordinates by $\{ \psi_k\}_k$,
where the coordinates are uncorrelated thanks to the orthogonality
of $\psi_{k}$. We further discuss the possible generalizations in the last section. 
The paper is mainly concerned with the kernel $k_{\text{spec}}$, including $k_{L^{2}}$, with anisotropic $a(r,x)$,
while 
the above extension of MMD may be of interest even when $a(r,x)$ is isotropic.

\subsection{Witness functions}\label{subsec:witness}

Following the convection of RKHS MMD \cite{gretton2012kernel}, 
the ``witness function'' $w(x)$ is defined as
\[
w(x) := 
\arg \max_{w \in{\cal H},\,\|w\|_{{\cal H}}=1}\langle w,\mu_{p}-\mu_{q}\rangle_{{\cal H}}
\]
where $\mu_{p}$ and $\mu_{q}$ are the mean embedding of $p$ and
$q$ in ${\cal H}$ respectively. By Riez representation, $w$ equals
$(\mu_{p}-\mu_{q})$ multiplied by a constant. We will thus consider
$\mu_{p}-\mu_{q}$ as the witness function. 
By definition, $\mu_{p}(x)=\int k(x,y)p(y)dy$,
and similarly for $\mu_{q}$, thus $w(x)=\int k(x,y)(p(y)-q(y))dy$. 
We then have that for the statistic $T_{L^{2}}$, 
\begin{equation}\label{eq:wL2population}
w_{L^{2}}(x)=\int a(r,x)(h_{p}(r)-h_{q}(r))d\mu_{R}(r),
\end{equation}
and generally for $T_{\text{spec}}$, 
\begin{equation}\label{eq:wspecpopulation}
w_{\text{spec}}(x)=\sum_{k}f_{k}\psi_{k}(x)\int\psi_{k}(y)(p(y)-q(y))dy.
\end{equation}
The computation of empirical witness functions will be discussed in Section \ref{subsec:witness_empirical}.

Although the witness function is not an estimator of the difference $(p-q)$ ,
 it gives a measurement of how $q$ and $p$ deviate at local places.
This can be useful for user interpretation of the two sample test, 
  as shown in Section \ref{sec:applications}.
 The witness function augments the MMD statistic, which is a global quantity.

\subsection{Kernel parameters}\label{subsec:kernel-para}

If the reference distribution $\mu_R$ and the covariance field $\Sigma_r$ are given,
there is no tuning parameter to compute $k_{L^2}$-MMD (c.f. Algorithm \ref{method:L2}).
To compute $k_{\text{spec}}$-MMD with general $f_k$, 
which is truncated to be of finite rank $r_f$,
the number $r_f$ and the values $\{f_k \}_{k=1}^{r_f}$ are tunable parameters (c.f. Algorithm \ref{method:spec}).

If the covariance field $\Sigma_r$ is provided up to a global scaling constant by $\Sigma_r^{(0)}$, 
e.g., when estimated from local PCA,
which means that one uses $\Sigma_r = \rho \Sigma_r^{(0)}$ for some $\rho>0$ in \eqref{eq:asymKernel},
then this $\rho$ is a parameter which needs to be determined in practice.

Generally speaking, 
one may view the reference set distribution $\mu_R$ and the covariance field $\{ \Sigma_r \}_r$ 
as ``parameters'' of the MMD kernel.
The optimization of these parameters surely have an impact on the power of the MMD statistics,
and a full analysis goes beyond the scope of the current paper. 
At the same time, there are important application scenarios where pre-defined $\mu_R$ and $\{ \Sigma_r \}_r$ 
are available, e.g., the diffusion MRI imaging data (Section \ref{subsec:diffuion-MRI}).
We thus proceed with the simplified setting by by assuming pre-defined $\mu_R$ and $\{ \Sigma_r \}_r$,
and focusing on the effects of anisotropic kernel and re-weighted spectrum.

%

\section{Analysis of Testing Power}\label{sec:theory}

We consider the population MMD statistic $T$ of the following form
\begin{equation}
T(p,q)=\int\int k(x,y)(p(x)-q(x))(p(y)-q(y))dxdy,\label{eq:Tpopulation}
\end{equation}
where $k = k_{\text{spec}}$ as in \eqref{eq:kernel-spec}, and particularly $k_{L^2}$ as in \eqref{eq:kernel-L2}.
 The empirical version is
\begin{equation}\label{eq:Tempirical}
T_{n}(X,Y)=\int\int k(x,y)(\hat{p}_{X}(x)-\hat{q}_{Y}(x))(\hat{p}_{X}(y)-\hat{q}_{Y}(y))dxdy,
\end{equation}
where $\hat{p}_{X}=\frac{1}{n_{1}}\sum_{i}\delta_{x_{i}}$ and $\hat{q}_{Y}=\frac{1}{n_{2}}\sum_{j}\delta_{y_{j}}$,
and $n=n_{1}+n_{2}$. We consider the limit where both $n_{1}$ and
$n_{2}$ go to infinity and proprotional to each other, in other words,
$n\to\infty$ and $\frac{n_{1}}{n}\to\rho_{1}\in(0,1)$, and $\frac{n_{2}}{n}\to\rho_{2}=1-\rho_{1}$.

We will show that, 
under generic assumptions on the kernel and the departure $q\neq p$, 
the test based on $T_{n}$ is asymptotically consistent, 
which means that the test power $\to 1$ as $n \to \infty$ 
(with controlled false-positive rate).
The asymptotic consistency holds 
when $q$ is allowed to depend on $n$ 
as long as  $\| q-p \|$ decays to 0 slower than $n^{-1/2}$.
We also 
provide a lower bound of the power based based upon Chebyshev which applies to the critical regime when $\tau$ 
is proportional to $n^{-1/2}$.
The analysis also provides a quantitative comparison of the testing power of different kernels.\

\subsection{Assumptions on the kernel}

Because $a(r,x)$ is uniformly bounded and Hilbert-Schmidt, 
$k_{L^{2}}(x,y)$ is positive semi-definite and compact, 
and thus can be expanded as in \eqref{eq:kernel-L2} where
$\{ \psi_k \}_k$ are a set of ortho-normal functions under $\mu=\rho_{1}p+\rho_{2}q$.
By that $0\le a(r,x)\le1$,
we also have that $k=k_{L^{2}}$ satisfies that
\begin{equation}\label{eqn:traceclass1}
0\le k(x,x)\le1,
\quad \forall x,
\end{equation}
and $|k(x,y)|\le1$ for any $(x,y)$.
Meanwhile, $k_{L^{2}}$ is continuous (by the continuity and uniformly boundedness of $a(r,x)$), 
which means that the series in \eqref{eq:kernel-L2}
converges uniformly and absolutely, and the eigenfunctions $\psi_{k}$
are continuous (Mercer's Theorem). 
Finally, \eqref{eqn:traceclass1} implies that the
operator is in the trace class, 
and specifically $\sum_{k} \sigma_{k}^2 \le1$.

These properties imply that when replacing $\sigma_k^2$ by $f_k$ 
as in  \eqref{eq:kernel-spec}, 
one can preserve the continuity and boundedness of the kernel.
We analyze kernels of the form as \eqref{eq:kernel-spec},
and assume the following properties.

\begin{assumption}
	\label{assump:A1}
	The $f_{k}$ in \eqref{eq:kernel-spec} satisfy that 
	$f_{k} \ge 0 $, $\sum_{k}f_{k}\le1$,
	and that the kernel $k(x,y)$ is PSD, 
	continuous, and $0 \le k(x,x) \le1 $ for all $x$.
\end{assumption}

As a result, 
Mercer Theorem applies to guarantee that the spectral expansion 
\eqref{eq:kernel-spec} converges uniformly and absolutely, 
and the operator is in the trace class.
Note that Assumption \ref{assump:A1} holds for a large class of kernels,
including all the important cases considered in this paper.
The previous argument shows that $k_{L^{2}}$ is covered.
Another important case is the finite-rank kernel:
that is,
$f_k = 0$ when $k > r_f$ for some positive integer $r_f$,
and $f_k$ can be any positive numbers when $k \le r_f$
such that $\sum_k f_k \le 1$ and $k(x,x)\le 1$ for all $x$.

For the MMD test to distinguish a deviation $q$ from $p$,
it needs to have $T(p,q) > 0$ for such $q$. 
We consider the family of alternatives $q$ which satisfies the following condition.

\begin{assumption}
	\label{assump:A2}
	When $q \neq p$, 
	there exists $k$ s.t. $ \int \psi_k(x) (p(x)-q(x))dx \neq 0$ and $f_k > 0$.
	In particular, if $f_k$ are strictly positive for all $k$,
	then $q$ satisfies that $h_p(r) - h_q(r) = \int a(r,x) (p(x)-q(x))dx$ 
	does not vanish w.r.t $d \mu_R$, i.e. $\int (h_p(r) - h_q(r) )^2 d\mu_R(r) > 0$.
\end{assumption}

The following proposition, proved in the Appendix, shows that $T(p,q) > 0$ for such deviated $q$.

\begin{proposition}\label{prop:equivalentA2}
	Notations as above, for a fixed $q \neq p$, the following are equivalent
	
	(i) $T(p,q)>0$,
	
	(ii) For some $k$, $\int\psi_{k}(x)(p(x)-q(x))dx\ne0$ and $f_{k}>0$.
	
	If $f_{k}>0$ for all $k$, then (i) is also equivalent to
	
	(iii) $\int (h_p(r) - h_q(r) )^2 d\mu_R(r) > 0$. 
\end{proposition}

Note that $k_{L^{2}}$ satisfies $f_{k}>0$ for all $k$, thus (iii) applies.
The proposition says that $T_{L^2}$ distinguishes an alternative $q$ 
when $(p-q)$ lies in the subspace spanned by $\{\psi_{k} \}_k$,
and for general $k_{\text{spec}}$,
$(p-q)$ needs to lie in the subspace spanned by
$\{\psi_{k} | f_k > 0 \}$. 
These bases are usually not complete,
e.g., 
when the reference set has $n_{R}$ 
points then $\{\psi_k\}$ is of rank at most $n_{R}$. 
However,
when the measure $d\mu_{R}$ is continuous and smooth, 
the singular value decomposition \eqref{eq:svd-a} 
has a sufficiently decaying spectrum,
and the low-frequency $\psi_k$'s can be
efficiently approximated
with a drastic down sampling of $d\mu_{R}$ \cite{bermanis2013multiscale}. 
This means that, 
under certain conditions of $d\mu_{R}$ 
(sufficiently overlapping with $\mu$ and regularity),
after replacing the continuous $d\mu_{R}$ by a discrete sampling in constructing the kernel,
an alternative $q$ violates Assumption \ref{assump:A2} 
only when the departure $(q-p)$ 
lies entirely in the high-frequency modes w.r.t the original $d\mu_{R}$.
In the applications considered in this paper, 
these very high-frequency departures are rarely of interest to detect,
not to mention that estimating $\int\psi_{k}(x)(p(x)-q(x))dx$ for large $k$
lacks robustness with finite samples. 
Thus Assumption \ref{assump:A2} poses a mild constraint 
 for all practical purposes considered in this paper,
 even with the spectral filtering kernel which utilizes a possibly truncated sequence of $f_k$.

We note that, by viewing $\psi_{k}$ as general Fourier modes,
Assumption \ref{assump:A2} servers as the counterpart 
of the classical condition of the kernel ``having non-vanishing Fourier transforms on any interval'',
which is needed for the MMD distance with translation-invariant kernel 
to be a distance (Sec 2.4  \cite{anderson1994two}).

\subsection{The centered kernel}

We introduce the centered kernel $\tilde{k}(x,y)$ under $p$, defined as 
\begin{align}\label{eq:def-tildek}
\tilde{k}(x,y)=k(x,y)-k_{p}(x)-k_{p}(y)+k_{pp},
\end{align}
where $k_{p}(x)=\int k(x,y)p(y)dy$, 
and $k_{pp}=\int k_{p}(x)p(x)dx$.
The spectral decomposition of $\tilde{k}$ is the key quantity used in later analysis.

The following lemma shows that 
the MMD statistic, both the population and the
empirical version, 
remains the same if $k$ is replaced by $\tilde{k}$. 
The proof is by definition and details are omitted. 

\begin{lemma}\label{lemma:tildk-as-k}
	Notations as above, 
	\begin{align}
	T(p,q) 
	& =\int\tilde{k}(x,y)(p(x)-q(x))(p(y)-q(y))dxdy,
	 \label{eq:def-T-tildek}\\
	T_{n}(X,Y)
	& =\int\int\tilde{k}(x,y)(\hat{p}_{X}(x)-\hat{q}_{Y}(x))(\hat{p}_{X}(y)-\hat{q}_{Y}(y))dxdy.
	\label{eq:def-Tn-tildek}
	\end{align}
	In particular, under Assumption \ref{assump:A2} that $T(p,q) > 0$, then so is \eqref{eq:def-T-tildek}.
\end{lemma}

The kernel $\tilde{k}$ also inherits 
the following properties from $k$, proved in Appendix \ref{app:A}:

\begin{proposition}\label{prop:tildek}
	The kernel $\tilde{k}$, as an integral operator on the space with the measure $p(x)dx$, is PSD. 
	Under Assumption \ref{assump:A1},
	
	(1) $0\le\tilde{k}(x,x)\le 4$ for any $x$, and $|\tilde{k}(x,y)|\le 4$
	for any $x,y$,
	
	(2) $\tilde{k}$ is continuous,
	
	(3) $\tilde{k}$ is Hilbert-Schmidt as an operator on $L^{2}(\Omega,p(x)dx)$
	and is in the trace class. It has the spectral expansion
	\begin{equation}
	\tilde{k}(x,y)=\sum_{k}\tilde{\lambda}_{k}\tilde{\psi}_{k}(x)\tilde{\psi}_{k}(y),
	\quad 
	\tilde{\lambda}_k >0,
	\label{eq:tildek-spec}
	\end{equation}
	where $\int \tilde{\psi}_k(x) p(x)dx =0$, $\int\tilde{\psi}_{k}(x)\tilde{\psi}_{l}(x)p(x)dx=\delta_{kl}$.
	$\tilde{\psi}_{k}$ are continuous on $\Omega$, 
	$\{\tilde{\lambda}{}_{k}\}_{k}$ are both summable and square summable,
	and the series in \eqref{eq:tildek-spec} converges uniformly
	and absolutely.
	
	(4) The eigenfunctions $\tilde{\psi}_{k}$ are square integrable,
	and thus integrable, w.r.t. $q(y)dy$.
	Furthermore,
	$\sum_{k}\tilde{\lambda}_{k} \int\tilde{\psi}_{k}(y)^{2}q(y)dy \le 4$.
\end{proposition}

To proceed, we define 
\begin{equation}
\tilde{v}_{k}:=\int\tilde{\psi}_{k}(x)(p(x)-q(x))dx,
\label{eq:tildevk}
\end{equation}
then by \eqref{eq:tildek-spec}, \eqref{eq:def-T-tildek}, \eqref{eq:def-Tn-tildek},
shortening the notation $T(p,q)$ as $T$, $T_{n}(X,Y)$ as $T_{n}$, 
we have that
\begin{align}
T
& =\sum_{k}\tilde{\lambda}_{k}\tilde{v}_{k}^{2}, 
\label{eq:T-lambdak-vk} \\
T_{n}
& =\sum_{k}\tilde{\lambda}_{k}\left(\frac{1}{n_{1}}\sum_{i}\tilde{\psi}_{k}(x_{i})-\frac{1}{n_{2}}\sum_{j}\tilde{\psi}_{k}(y_{j})\right)^{2}.
\label{eq:Tn1}
\end{align}

\subsection{Limiting distribution of $T_{n}$ and asymptotic consistency}

Consider the alternative $q (\tau) =p + \tau g$ for some fixed $g$, 
$0 \le \tau \le 1$.
$q(\tau) $ remains a probability density for any $\tau$. 
We define the constants
\begin{equation}\label{eq:ck}
c_{k}:=\int\tilde{\psi}_{k}(y)g(y)dy,
\quad 
\tilde{v}_{k}=-\tau c_{k}
\end{equation}
where $c_k$ are finite 
by the integrability
of $\tilde{\psi}_k$ under $q_1(x)dx$ ((4) of Proposition \ref{prop:tildek}),
and $\sum_k \tilde{\lambda}_k c_k^2 <\infty$.

The theorem below identifies the  limiting distribution of $T_{n}$
under various order of $\|q-p\|$, which may decrease to 0 as $n\to \infty$.
The techniques very much follow Chapter 6 of \cite{serfling1981approximation} 
(see also  Theorem 12, 13 in  \cite{gretton2012kernel}),
where the key step is to replace the two independent sets of summations 
in \eqref{eq:Tn1}
by normal random variables 
via multi-variate CLT (Lemma \ref{lemma:replace}),
using a truncation argument based on the decaying of $\tilde{\lambda}_k$.
The proof is left to Appendix \ref{app:A}.

\begin{theorem}\label{thm:limit1}
	Let $\tau=\tau_{n}$ may depend on $n$, $ 0 \le \tau \le 1$, 
	and notations $\tilde{\lambda}_k$, $\tilde{\psi}_k$, $c_k$  and others like above.
	Under Assumption \ref{assump:A1}, 
	as $n=n_{1}+n_{2}\to\infty$
	with $\frac{n_{1}}{n}\to\rho_{1}\in(0,1)$, $\rho_1 + \rho_2 = 1$, 
	
	(1) If $\frac{\tau}{n^{-1/2}} \to a$, $0 \le a < \infty$ (including the case when $\tau=0$),
	then 
	\[
	nT_{n}\overset{d}{\to}\sum_{k}\tilde{\lambda}_{k}(  -ac_{k}+ \xi_k )^2,
	\quad 
	\xi_k \sim {\cal N} \left(0,\frac{1}{\rho_1} + \frac{1}{\rho_2} \right) \text{ i.i.d}.
	\]
	Due to the summability
	of $\sum_{k}\tilde{\lambda}_{k}$ the random variable on the right
	rand side is well-defined.
	
	(2) If $ \tau = n^{-1/2+\delta}$, where $0<\delta<\frac{1}{2}$,
	then 
	\[
	\frac{nT_{n} - n^{2\delta}  \sum_{k}\tilde{\lambda}_{k}c_{k}^{2}}{n^{\delta}}
	\overset{d}{\to} 
	 {\cal N}(0,\sigma_{(2)}^{2})
	\]
	where 
	$\sigma_{(2)}^{2}
	=
	4 \sum_{k}\tilde{\lambda}_{k}^{2} c_{k}^{2}(\frac{1}{\rho_{1}}+\frac{1}{\rho_{2}}) < \infty$.
	
	(3) If $\tau = 1$,  then 
	\[
	\sqrt{n}(T_{n} - \sum_{k}\tilde{\lambda}_{k}c_{k}^{2})
	\overset{d}{\to} 
	{\cal N}(0,\sigma_{(3)}^{2})
	\]
	where 
	$\sigma_{(3)}^{2} = 4 \left(\frac{1}{\rho_{1}}\sum_{k}\tilde{\lambda}_{k}^{2}c_{k}^{2}
	                               +\frac{1}{\rho_{2}} \sum_{k,l} \tilde{\lambda}_k  \tilde{\lambda}_l c_k c_l S_{kl}\right) < \infty$,
	with $S_{kl}=\int\tilde{\psi}_{k}(y)\tilde{\psi}_{l}(y)q(y)dy- c_{k}c_{l}$.
\end{theorem}

\begin{remark}
As a direct result of Theorem \ref{thm:limit1} (1),
	under $\mathcal{H}_{0}$, 
	\[
	nT_{n} \overset{d}{\to}
	\sum_{k}\tilde{\lambda}_{k} \xi_k^2,
	\quad 
	\xi_k \sim {\cal N} \left(0,\frac{1}{\rho_1} + \frac{1}{\rho_2} \right) \text{ i.i.d}.
	\]
	When $n_{1}=n_{2}$, 
	the limiting density is 
	$4\sum_{k}\tilde{\lambda}_{k}\chi_{1,k}^{2}$ where $\chi_{1,k}^{2}$ are i.i.d. $\chi_{1}^{2}$ random variables.
Theorem \ref{thm:limit1} (3) shows the asymptotic normality of  $T_n$ under ${\cal H}_{1}$.
One can verify that when $n_{1}=n_{2}$, 
$\sigma_{(3)}^{2}=8 \text{Var}_{x,y}(\tilde{k}_{1}(x,y))$,
where $\tilde{k}_{1}(x,y)=-\int(\tilde{k}(x,y')-\tilde{k}(y,y'))q(y')dy'=\sum_k \tilde{\lambda}_k \tilde{v}_k(\tilde{\psi}_k(x)-\tilde{\psi}_k(y))$.
These
limiting densities 
under
$n_{1}=n_{2}$ 
recover the classical result of V-statistics 
(Theorem 6.4.1.B and Theorem 6.4.3 \cite{serfling1981approximation}). 
\end{remark}

The numerical experiments in Section \ref{subsec:compare} show that the theoretical limits identified in  Theorem \ref{thm:limit1} (1) 
approximate the empirical distributions of $T_n$ quite well when $n$ equals a few hundreds (Fig. \ref{fig:compare1}).
In below,
we address the asymptotic consistency of the MMD test, 
and  the finite-sample bound of the test power will be discussed in the next subsection. 

A test based on the MMD statistic $T_n$  rejects $\mathcal{H}_{0}$ 
whenever $T_{n}$ exceeds certain threshold $t$.
For a target ``level'' $0 < \alpha < 1$ (typically $\alpha = 0.05$),
a test (with $n$ samples) achieves level $\alpha$ if 
$\Pr [T_{n}> t |\mathcal{H}_{0}] \le \alpha$,
and the ``power'' against an alternative $q$ is defined as 
$\Pr[T_{n} > t |\mathcal{H}_{1}\text{ with \ensuremath{q}}]$.
We define the test power at level $\alpha$ as
\begin{equation}\label{eq:pi-n-alpha}
\pi_{n, \alpha}(q)= \sup_t 
\{
\Pr[T_{n} > t |\mathcal{H}_{1}\text{ with \ensuremath{q}}],
\, s.t. ~ 
\Pr [T_{n}> t |\mathcal{H}_{0}] \le \alpha
\}.
\end{equation}
The threshold $t$ needs to be sufficiently large to guarantee that the false-positive rate is at most $\alpha$, 
and in practice it is a parameter to determine (c.f. Section \ref{subsec:permutation_test}).
In below, we omit the dependence on $\alpha$ in the notation, and write $\pi_{n, \alpha}(q)$ as $\pi_{n}(q)$. 
The test is called asymptotically  consistent if $\pi_{n}(q)\to1$ as $n\to\infty$.
For the one-parameter family of $q(\tau) = p + \tau g$,
the following theorem shows that, if $q$ satisfies Assumption \ref{assump:A2},
the MMD test has a nontrivial power when  $\frac{\tau}{n^{-1/2}}$ converges to a positive constant,
and is asymptotically consistent if $\frac{\tau}{n^{-1/2}} \to \infty$.

\begin{theorem}\label{thm:consist1}
	Under Assumption \ref{assump:A1}, 
	and suppose that the departure $g$ makes $q = p + \tau g$ satisfy Assumption \ref{assump:A2} for $ \tau > 0$,
	where $ 0 \le \tau \le 1$ may depend on $n$.
	As $n \to \infty$ with $\frac{n_1}{n }\to \rho_1 \in (0,1)$,
	
	(1) If $\frac{\tau}{n^{-1/2}} \to a$, $0 < a < +\infty$, 
	then  $\pi_{n}( q ) \to f(a) > \alpha$, 
	where $f$ is a monotonic function of $a$. 
	
	(2) If $\frac{\tau}{n^{-1/2}} \to +\infty$, 
	then $\pi_{n}( q ) \to 1$.
\end{theorem}

This is qualitatively the same result as for RKHS MMD 
(Theorem 13 in \cite{gretton2012kernel}).
The claim (1) directly follows from Theorem \ref{thm:limit1} (1),
and the proof for (2) uses similar techniques.  The proof is left to Appendix \ref{app:A}.

\subsection{Non-asymptotic bound of the testing power}\label{sec:nonasymptoticpower}

In this section, we derive a non-asymptotic (lower) bound of the testing power $\pi_n(q)$ for finite $n$,
which shows that the speed of the convergence $\pi_n(q) \to 1$ is at least as fast as $O(n^{-1})$ as $n$ increases
whenever $n$ is greater than certain threshold value.

\begin{theorem}
\label{thm:consist2}
Notations $\tau$, $\tilde{\lambda}_k$, $c_k$ as above.
Define $\rho_{1,n} := \frac{n_1}{n}$, $\rho_{2,n} := \frac{n_2}{n}$, $ 0 < \rho_{1,n} <1$ and $\rho_{1,n}  + \rho_{2,n} =1$,
and let $ 0< \alpha < 1$ be the target level. 
Under Assumption \ref{assump:A1} and \ref{assump:A2},
define $T_1 := \sum_k \tilde{\lambda}_k c_k^2 > 0$.
If $ n > \frac{16}{0.1}( \frac{1}{ \rho_{1,n}^3 } + \frac{4}{ \rho_{2,n}^3 }   ) $
and
\begin{equation}\label{eq:cond2-consist2}
(\tau^2 n)  T_1 
 > 
C_4 + \sqrt{ \frac{C_3 + 0.1}{\alpha} }, 
\end{equation}
then 
\begin{equation}
\label{eq:bound-consist2}
1 - \pi_n(q) 
\le 
\frac{ (\tau^2 n) C_1 + \tau C_2 + C_3 + 0.1    }
{ \left(  (\tau^2 n)  T_1 -  (C_4 + \sqrt{ \frac{C_3 + 0.1}{\alpha} }  )        \right)^2 }
\end{equation}
where
\begin{align*}
C_1 
	&:= 4 \left(   \frac{1}{\rho_{1,n}} \sum_k \tilde{\lambda}_k^2 c_k^2 + \frac{16}{\rho_{2,n}}    \right), 
\quad	
C_2 
	:= 128 \left(  \frac{1}{\rho_{1,n}^2} +  \frac{1}{\rho_{2,n}^2} \right), \\
C_3 
	&:=  \frac{32}{(\rho_{1,n} \rho_{2,n})^2},  
\quad
C_4 
	:=  \frac{1}{\rho_{1,n} \rho_{2,n}} \sum_k \tilde{\lambda}_k. 
\end{align*}
In particular,
assuming that $\rho_{1,n}$ is uniformly bounded to be between $(r, 1-r)$ for some $0 <r <\frac{1}{2}$,
then the constants $C_1$, $C_2$, $C_3$, $C_4$ are all uniformly bounded with respect to $n$
by constants which only depend on $\tilde{\lambda}_k$ and $c_k$.
\end{theorem}

The above bound applies when $\tau$ is proportional to $n^{-1/2}$, which is the critical regime. 
The lower bound of $\pi_n(q) $ in \eqref{eq:bound-consist2} 
 increases with $\tau^2 n$ and approaches 1 when  $\tau^2 n \to \infty$,
showing the same behavior as Theorem \ref{thm:consist1}.
It can be used as another way to prove Theorem \ref{thm:consist1} (2).
In particular, 
as $n$ increases (or $\tau^2 n$ increases, since $\tau$ stays bounded) 
and assuming uniformly bounded $\rho_{1,n}$,
the r.h.s of \eqref{eq:bound-consist2}
is dominated by 
\[
\frac{ C_1}{T_1^2} \frac{1}{\tau^2 n},
\]
which leads to an $O(n^{-1})$ bound of $1-\pi_n (q)$ for fixed $\tau$.
The theorem only uses Chebyshev,
so the bound may be improved,
e.g.,
by investigating the concentration of $n T_n$ which is a quadratic function of independent sums, c.f. \eqref{eq:Tn1}. 

We notice two other possible ways of deriving non-asymptotic bound of the power:
(1) 
Berry-Essen rate of the convergence to asymptotic normality 
has been proved for U-statistics in the case of $p\neq q$, 
and can be extended to the V-statistics (for the case of $n_1 = n_2$),
c.f. Chapter 6 of \cite{serfling1981approximation}.
This can lead to a control of $1-\pi_n (q) < C(k, \tau) n^{-1/2}$, 
where the constant 
depends on the kernel and the departure magnitude $\tau$.
However, when $\tau$ is proportional to $n^{-1/2}$,
the Berry-Essen rate loses sharpness and cannot give a nontrivial bound. 
(2) 
Large deviation bounds of the empirical MMD $T_n$ have been obtained by
McDiarmid inequality and the boundedness of the Rademacher complexity of the unit ball in the RKHS,
where $T_n$ is proved to converge to the population value $T$  
 at the rate $O(n^{-1/2})$ with exponential tail
(Theorem  7 of \cite{gretton2012kernel}).
This gives 
a stronger concentration than Chebyshev when $\tau^2 > c n^{-1/2}$ for some constant $c$ 
and leads to an exponential decay of $1-\pi_n (q)$ with increasing $n$,
but it does not apply to give a control when $\tau$ is proportional to $n^{-1/2}$ i.e. the critical regime. 

\subsection{Comparison of kernels}\label{subsec:compare}

\begin{figure}[tbp]
	\begin{centering}
		\includegraphics[width=0.32\textwidth]{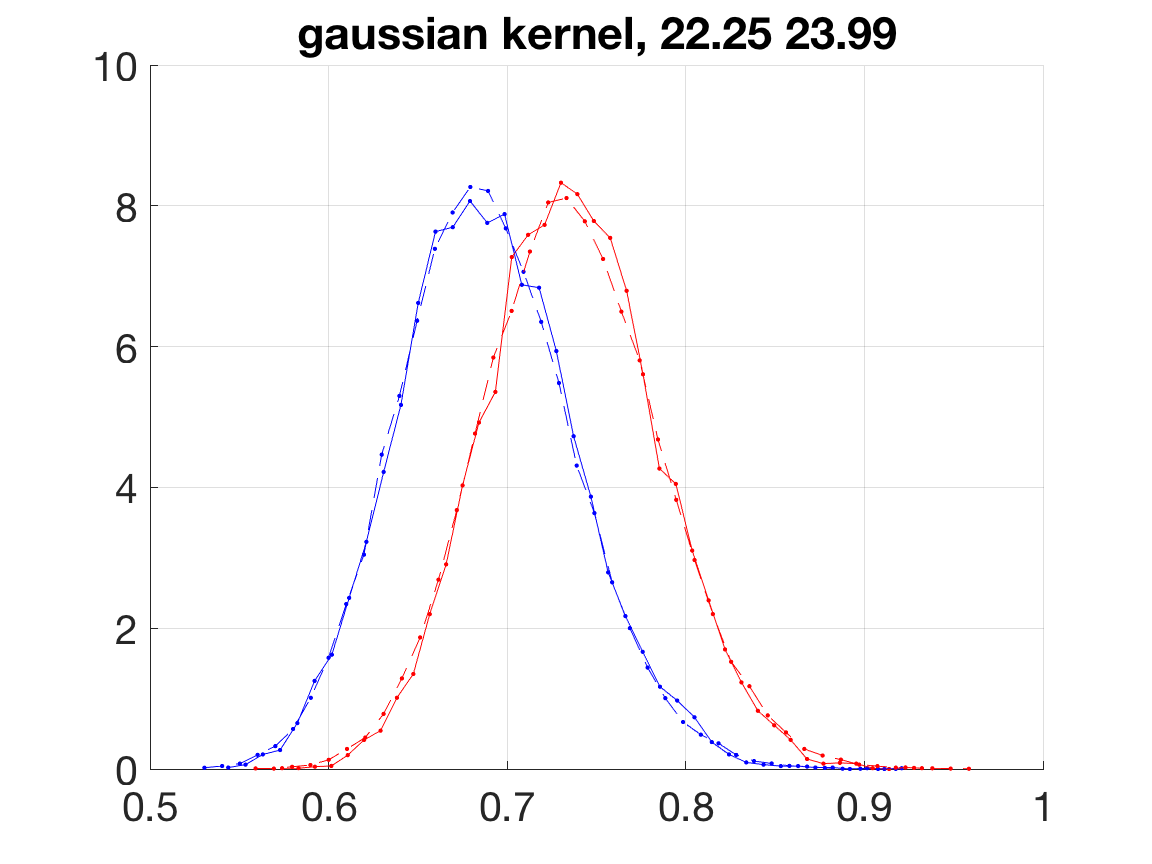}
		\includegraphics[width=0.32\textwidth]{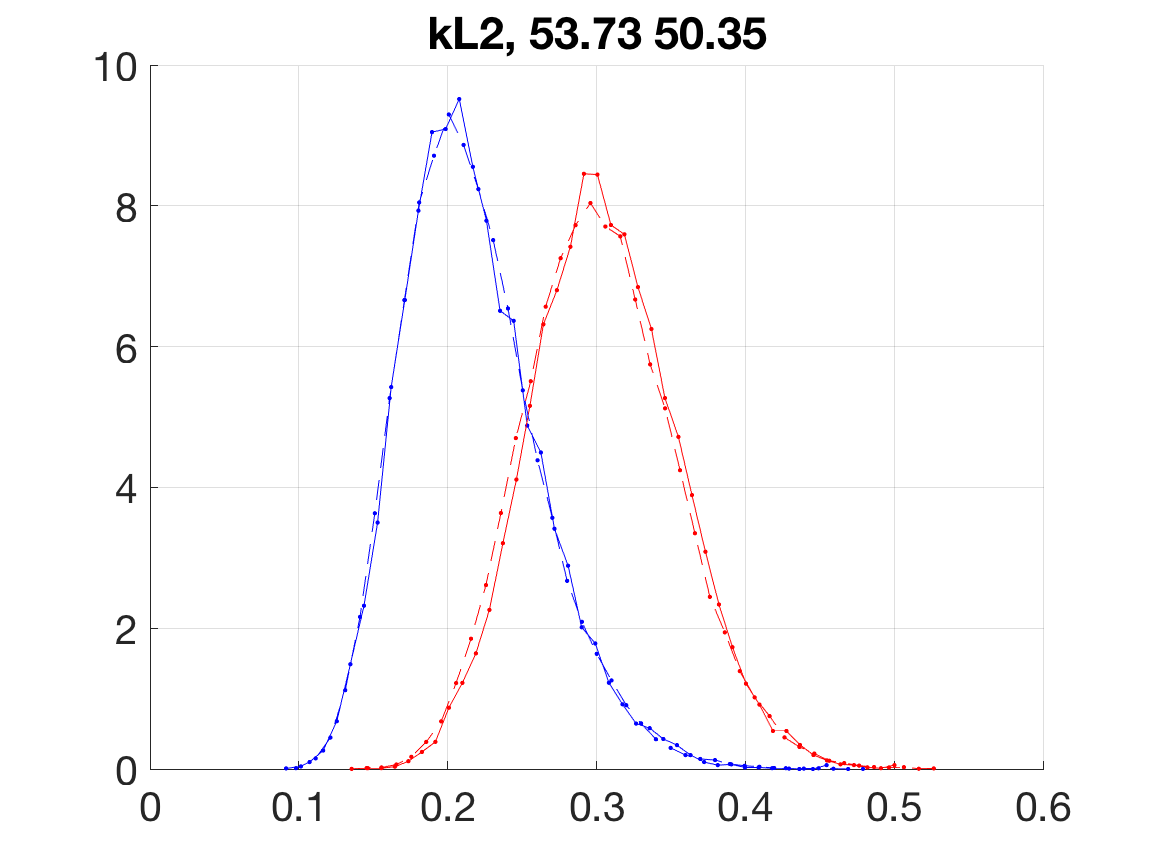}
		\includegraphics[width=0.32\textwidth]{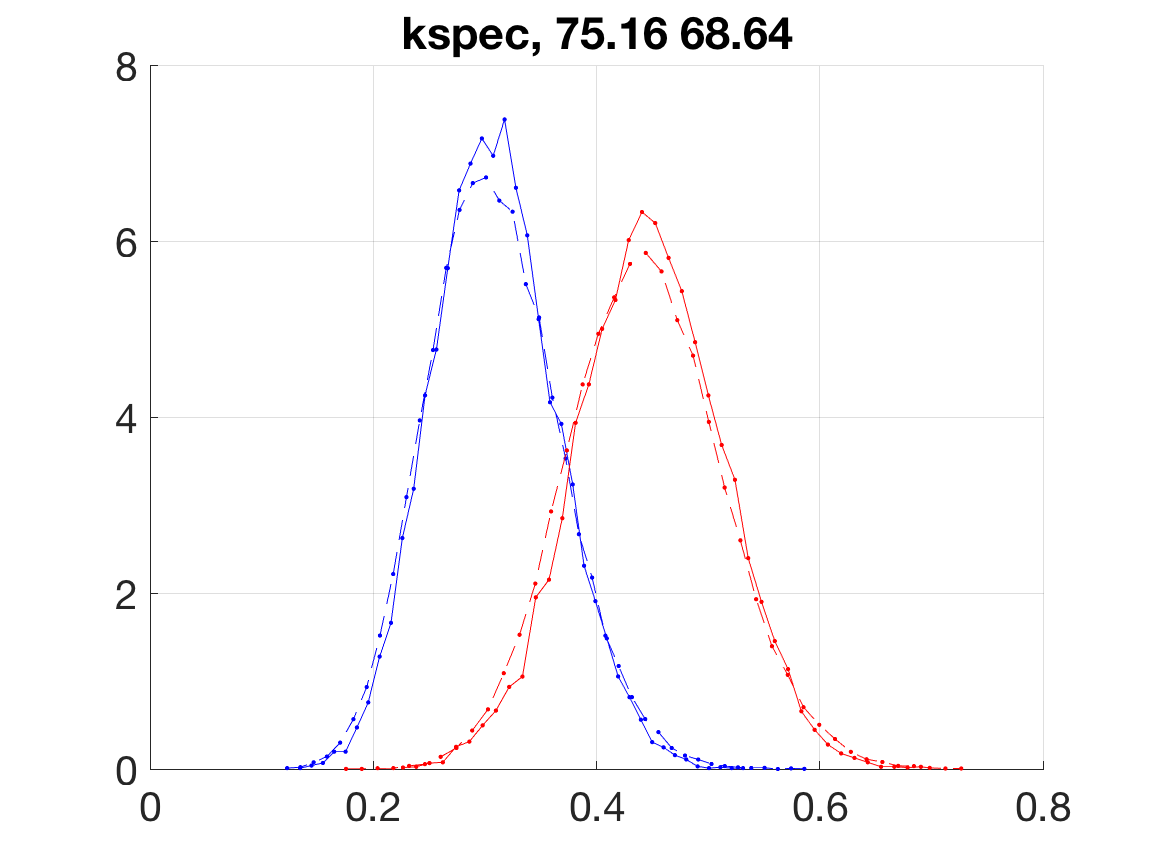} \\
		\caption{\label{fig:compare1}
			Empirical distribution of the statistic $T_n$ under ${\cal H}_0$ (blue) and  ${\cal H}_1$ (red) 
			shown in comparison with the large $n$ limiting distribution (broken lines) as in Theorem \ref{thm:limit1},
			for gaussian kernel (left), $k_{L^2}$ (middle) and $k_{\text{spec}}$ (right).
			The numbers are the fraction of $T_n | {\cal H}_1$ lying to the right of the 0.95-quantile of $T_n | {\cal H}_0$, 
			for empirical and theoretical curves respectively.
			Data sets $X$ and $Y$ drawn from $p$ and $q$ in the example in Section \ref{subsec:compare}, 
			a scatter plot is shown in Figure \ref{fig:curve1}, $n_1 = n_2 = 400$.
		}
	\end{centering}
\end{figure}

As suggested by Theorem \ref{thm:consist2},
the power of the MMD test depends on the mean and variance of the statistic under 
${\cal H}_0$ and ${\cal H}_1$ respectively, which are
\begin{eqnarray*}
	\theta_{0} & = & E[T_{n}|{\cal H}_{0}],\quad\theta_{1}=E[T_{n}|{\cal H}_{1}],\\
	\sigma_{0}^{2} & = & \text{Var}(T_{n}|{\cal H}_{0}),\quad\sigma_{1}^{2}=\text{Var}(T_{n}|{\cal H}_{1}).
\end{eqnarray*}
By Theorem \ref{thm:limit1} (1), at the critical regime where $\tau$ is proportional to $n^{-1/2}$,
they can be approximated by the following (recall that $\tilde{v}_{k}=-\tau c_{k}$) 
\begin{eqnarray*}
	\bar{\theta}_{0} &= & \frac{2}{n}\sum_{k}\tilde{\lambda}_{k},
		\quad\bar{\theta}_{1}=\sum_{k}\tilde{\lambda}_{k}\tilde{v}_{k}^{2}+\frac{2}{n}\sum_{k}\tilde{\lambda}_{k},\\
	\bar{\sigma}_{0}^{2} &= &\text{Var}(\sum_{k}\tilde{\lambda}_{k}\frac{1}{n}(h_{k}-g_{k})^{2}),
	\quad\bar{\sigma}_{1}^{2}=\text{Var}(\sum_{k}\tilde{\lambda}_{k}(\tilde{v}_{k}+\frac{1}{\sqrt{n}}(h_{k}-g_{k}))^{2}),
\end{eqnarray*}
where $h_{k},g_{k}\sim N(0,1)$ i.i.d. 
Notice that the gap $\bar{\theta}_{1}-\bar{\theta}_{0}=\sum_{k}\tilde{\lambda}_{k}\tilde{v}_{k}^{2}$, 
which is the population $T$.

\begin{figure}
	\begin{centering}
		\includegraphics[width=0.4\textwidth]{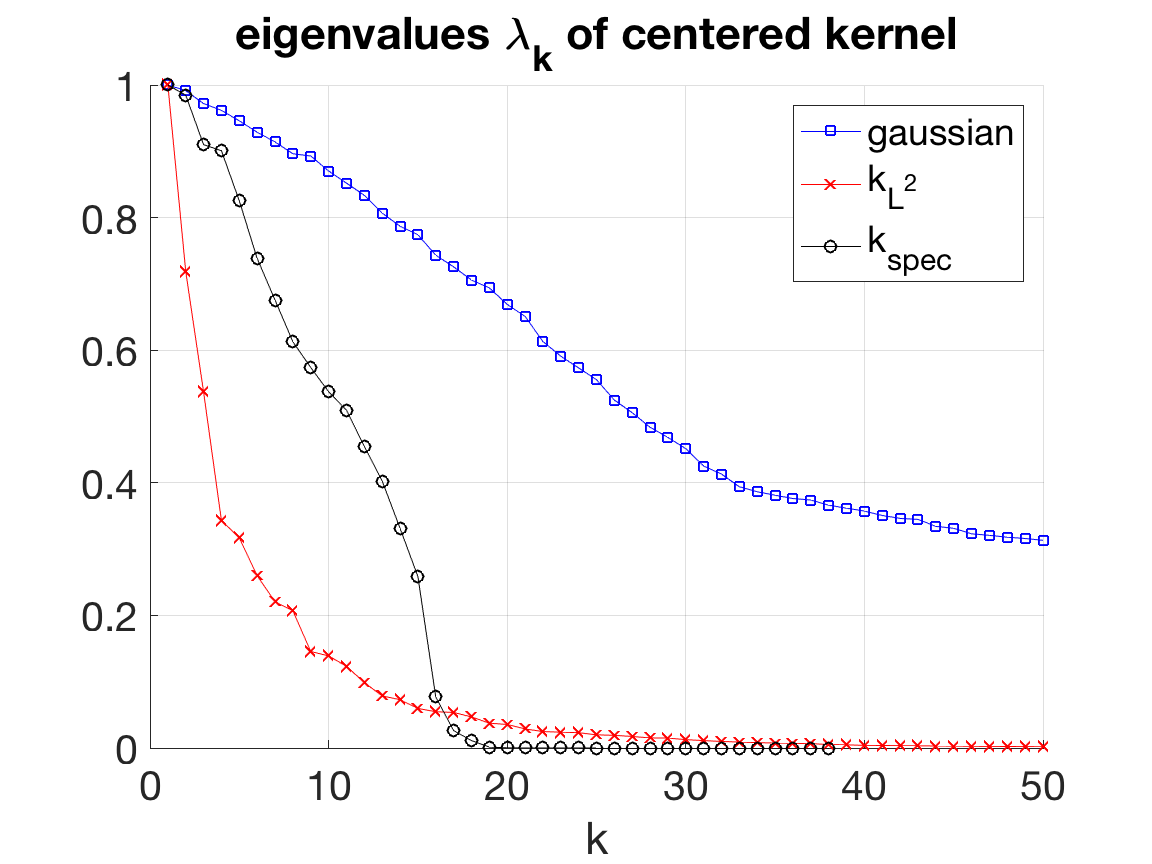} ~~~~~~~~
		\includegraphics[width=0.4\textwidth]{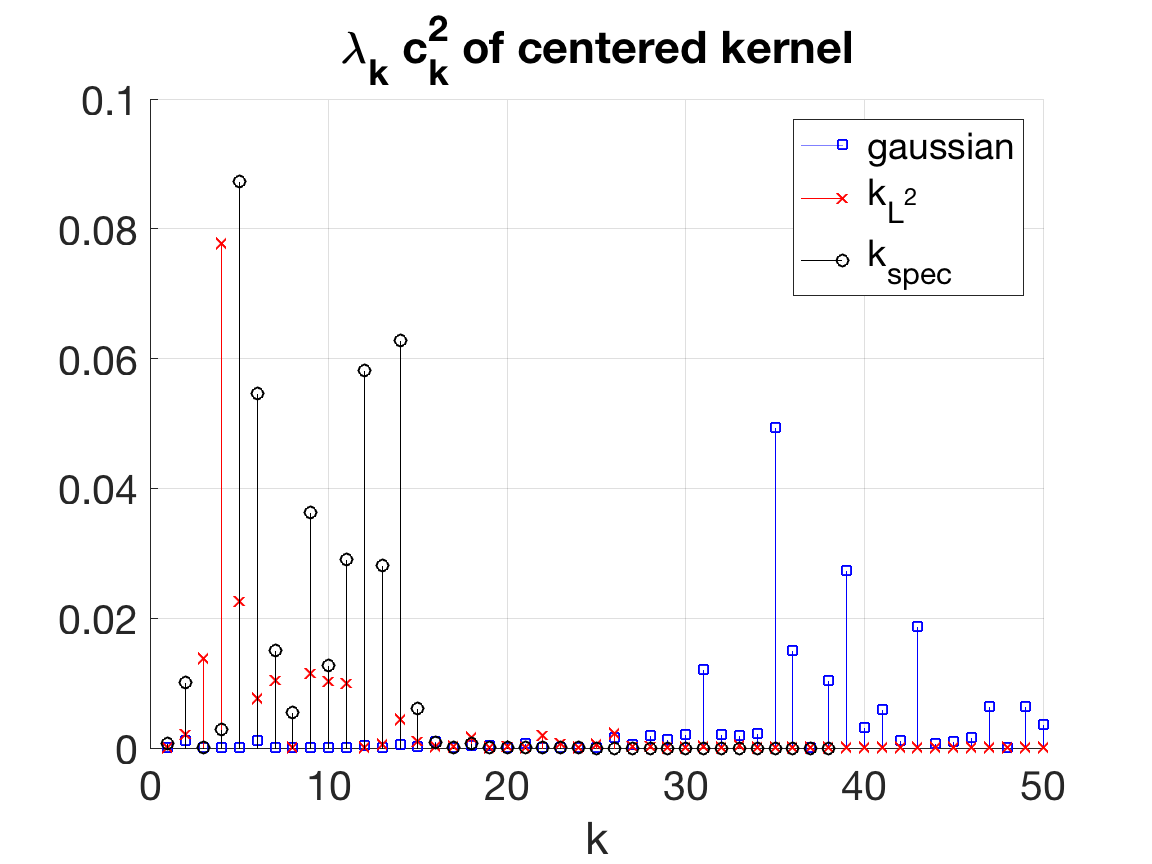}
		\caption{\label{fig:compare3}
			(Left) Eigenvalues $\tilde{\lambda}_k$  of the three kernels:  
			gaussian, $k_{L^2}$ and $k_{\text{spec}}$.
			(Right) Product $\tilde{\lambda}_k c_k^2$ where $c_k$ is the bias for mode $k$ as explained in Section \ref{subsec:compare}.
		}
	\end{centering}
\end{figure}

In this subsection, we compare isotropic and anisotropic kernels by numerically computing the above quantities 
under a specific choice of data distributon.
We will consider three kernels: 
(1) the gaussian kernel $k_{g}(x,y)=e^{-\|x-y\|^{2}/2\epsilon_{x}^{2}},$
(2) the one induced by the anisotropic kernel $k_{L^{2}}(x,y)=\int a(r,x)a(r,y)dr,$
where $dr$ is set to be the Lebesgue measure, and $\Sigma_{r}$ is
constructed so that the tangent direction is the first principle direction
with variance $0.2^{2}$, and the the normal direction is the second
principle direction with variance $\epsilon_{x}^{2}$, and 
(3) the spectral filtered kernel as in Eqn. (\ref{eq:kernel-spec}) where
$f_k$ are designed to be 1 for $1 \le k \le 10$ and smoothly decays to zero at $k = 20$.
All kernels are multiplied by a constant to make the largest eigenvalue
$\tilde{\lambda}_{1}=1$ so as to be comparable.
We also adopt the ratio (similar to the object of kernel optimization considered in \cite{gretton2012optimal})
\[
r=\frac{\theta_{1}-\theta_{0}}{\sigma_{1}+\sigma_{0}}
\]
to illustrate the testing power, where the larger the $r$ the more powerful the test is likely to be.

For the distribution $p$ and $q$,
we use the 2-dimensional example in the first section of the paper. 
Specifically, $p$ is the uniform distribution on the curve $\{(\cos\frac{\pi}{2}t,\sin\frac{\pi}{2}t),\,0\le t\le1\}$ convolved
with $N(0,\epsilon_{x}^{2}I_{2})$, $\epsilon_{x}=0.02$, 
and $q$
is the distribution of the shifted curve 
$\{( (1-\delta) \cos\frac{\pi}{2}t, (1-\delta) \sin\frac{\pi}{2}t),\,0\le t\le1\}$
convolved with $N(0,\epsilon_{x}^{2}I_{2})$, where $\delta=0.02$.
One realization of 200 points in $X$ and $Y$ is shown in Figure \ref{fig:curve1}.
\footnote{
	Strictly speaking $p$ and $q$ are longer compacted supported due
	to convolving with the 2-dimensional gaussian distribution, however,
	the normal density exponentially decays and $\epsilon_{x}$ is small,
	so that there is not any practical difference. The analysis can extend
	by a truncation argument.} 
Let $g=q-p$, 
we consider the one-parameter family of alternative $q=p+\tau g$, 
and will simulate in the critical regime where the test power is less than 1. 
We assume that $n_{1}=n_{2}$ and denote by $n$ in this subsection.

In all experiments, the quantities $\theta_{0}$, $\theta_{1}$, $\sigma_{0}$
and $\sigma_{1}$ are computed by Monte-Carlo simulation over 10000 runs. 
Their asymptotic version (with bar) are computed by 50000 draws of the limiting distribution, 
truncating the summation over $k$ to the first 500 terms; 
the values of $\tilde{\lambda}_{k}$ and $c_{k}=\int\tilde{\psi}_{k}(y)(q(y)-p(y))dy$
are approximated by the empirical eigenvalues and mean of eigenvectors over 10000 sample points drawn from $p$  and $q$
($\tilde{\psi}(y)$ is computed by Nystrom extension). 
The values of
$\tilde{\lambda}_{k}$ and $c_{k}$ are shown in Figure \ref{fig:compare3}, and those
of $\theta_{0}$ etc. in Table \ref{tab:tab2}.

\begin{table}
	\begin{center}
		\begin{tabular}[t]{ c | c c c c | c}
			\hline
			& $\theta_0$ 	& $\theta_1$ 	& $\sigma_0$ 	& $\sigma_1$ 	& $r$ \\
			\hline
			$n=200, \tau=0.5$ & & & & & \\
			
			Gaussian 	         	& 0.4771	& 0.5444	& 0.0677	& 0.0704	& 0.4874	 \\
			& 0.4754	& 0.5439	& 0.0676	& 0.0736	& 0.4848	\\
			$k_{L^2}$  	 	& 0.0489	& 0.0958	& 0.0214	& 0.0306	& 0.9000	 \\
			& 0.0488	& 0.0939	& 0.0214	& 0.0312	& 0.8573	  \\
			$k_{\text{spec}}$ 	& 0.0985	& 0.2046	& 0.0348	& 0.0587	& 1.1351    \\
			& 0.0983	& 0.2013	& 0.0374	& 0.0620	& 1.0354 \\
			\hline
			$n=400, \tau= 0.5/\sqrt{2}$ & & & & & \\
			Gaussian 			&  0.2381	& 0.2720	& 0.0334	& 0.0359	& 0.4885	 \\
			&  0.2379	& 0.2722	& 0.0339	& 0.0368	& 0.4850	 \\
			$k_{L^2}$ 		&  0.0243	& 0.0477	& 0.0107	& 0.0153	& 0.8972 \\ 
			&  0.0244	& 0.0471	& 0.0106	& 0.0157	& 0.8616	\\
			$k_{\text{spec}}$	&  0.0490	& 0.1036	& 0.0177	& 0.0305	& 1.1343	 \\
			& 0.0490	& 0.1003	& 0.0188	& 0.0310	& 1.0290	\\					
			\hline
		\end{tabular}
	\end{center}	
	
\begin{flushleft}
\caption{
	The values of three kernels: gaussian, $k_{L^2}$ and $k_{\text{spec}}$,
	in the example in Section \ref{subsec:compare}.
	In first row are average over Monte-Carlo simulations, 
	and in second row are the approximated values computed according to the limiting distribution identified by Theorem \ref{thm:limit1}.
}\label{tab:tab2}
\end{flushleft}
\end{table}

%

\section{Practical Considerations}\label{sec:practical}

Algorithm \ref{algo1} gives the pseudo code for two-sample test based on the proposed MMD statistics,
where two external subroutines, $\textproc{akMMD}$ and $\textproc{akWitness}$, 
will be given in Algorithm \ref{method:L2} and Algorithm \ref{method:spec} for 
$T_{L^2}$  and $T_{\text{spec}}$ statistics respectively. 
In Algorithm \ref{method:spec},
an extra input parameter $f$,  
which is a positive vector of length $r_f $ ($r_f < \min\{ n_1+n_2, n_R \}$),
is needed.
$f$ is the target spectrum of the kernel $k_{\text{spec}}$ in Eqn. (\ref{eq:kernel-spec}).
Both algorithms compute the threshold $t_\alpha$, 
which is the maximum threshold to guarantee level $\alpha$ (controlled false discovery),
by bootstrapping.
It is also assumed that the reference set $R$ with $\{ \Sigma_r\}_{r \in R}$ are predefined.

In the rest of the section, 
we explain the bootstrapping approach,
the empirical estimator of the witness function,
and the construction of reference set 
in detail. We close the section by commenting on the computation complexity.

\begin{algorithm}[t]
	\caption{Two-sample Test with Anisotropic-kernel MMD (\textproc{akMMD})}\label{algo1}
	
	{\bf Input:} Datasets $X$ and $Y$, function handle $a(r,x)$,
	$n_{\text{boot}}$, 
	data points $x_{\text{witness}}$
	
	{\bf Output:} Acceptance/Rejection of ${\cal H}_0$, 
	the witness function evaluated on $x_{\text{witness}}$
	
	{\bf External:} Subroutines $\textproc{akMMD}$, $\textproc{akWitness}$               
	
	\begin{algorithmic}[1]
		
		\Function{TwoSampleTest}{$X$, $Y$, $a$, $n_{\text{boot}}$, $x_{\text{witness}}$}
		
		\State ${n_1} \leftarrow \text{size}(X)$, ${n_2} \leftarrow \text{size}(Y)$, 
		${n_R} \leftarrow \text{size}(R)$
		\State Compute matrix $A_X \leftarrow \{ a(r,x) \}_{r\in R, x\in X}$ and $A_Y \leftarrow \{ a(r,y) \}_{r\in R, y\in Y}$
		\Comment{$a(r,x)$ as in (\ref{eq:asymKernel})}
		
		\State $T, \, T_{\text{null}} \leftarrow \text{ \textproc{akMMD}($A_X$, $A_Y$, $n_{\text{boot}}$)} $
		\Comment{Subroutine  to compute MMD with permutation}
		
		\State $\alpha \leftarrow 0.05$, $t_\alpha  \leftarrow \text{the $(1-\alpha)$-quantile of $T_{\text{null}}$}$
		\Comment{$\alpha$ is the level of Type I error}
		
		\State $Reject  \leftarrow (T > t_\alpha)$
		\Comment{$Reject$ is a Boolean variable}
		
		\State $w  \leftarrow \text{ \textproc{akWitness}($A_X$, $A_Y$, $a$, $x_{\text{witness}}$)} $
		\Comment{Subroutine to compute witness function}
		
		\State \textbf{return} $Reject$, $w$	
		
		\EndFunction
		
	\end{algorithmic}
	
\end{algorithm}

\begin{algorithm}[t]
	\caption{Methods for $k_{L^2}$-MMD}\label{method:L2}
	\begin{algorithmic}[1]
		\Function{MMD-L2}{ $A_X$, $A_Y$, $n_{\text{boot}}$ }
		\State Concatenate $A  \leftarrow [A_X \, | \, A_Y]$
		\Comment{$A$ is $n_R$-by-$(n_1+n_2)$}	
		
		\State For $i=1,\cdots, n_R$,
		$h_X[i] \leftarrow \frac{1}{n_1} \sum_{j=1}^{n_1} A[i,j]$,
		$h_Y[i] \leftarrow \frac{1}{n_2} \sum_{j=n_1+1}^{n_1+n_2} A[i,j]$,
		\State $T \leftarrow \frac{1}{n_R} \sum_{i=1}^{n_R} (h_X[i]-h_Y[i] )^2 $	
		\Comment{Empirical MMD}
		
		\For{$ k = 1$ to $n_{\text{boot}}$}
		\State $A_{\text{permute}} \leftarrow \text{$A$ with random permuted columns}$
		\State For $i=1,\cdots, n_R$,
		$h_X[i] \leftarrow \frac{1}{n_1} \sum_{j=1}^{n_1} A_{\text{permute}} [i,j]$,
		$h_Y[i] \leftarrow \frac{1}{n_2} \sum_{j=n_1+1}^{n_1+n_2} A_{\text{permute}} [i,j]$,
		\State $T_{\text{null}}[k] \leftarrow \frac{1}{n_R} \sum_{i=1}^{n_R} (h_X[i]-h_Y[i] )^2 $
		\Comment{One sample of MMD under null hypothesis}
		\EndFor
		
		\State \textbf{return} $T$, $T_{\text{null}}$	
		\EndFunction
		
		\Statex
		
		\Function{Witness-L2}{$A_X$, $A_Y$, $a$, $Z$}
		\State For $i=1,\cdots, n_R$,
		$h_X[i] \leftarrow \frac{1}{n_1} \sum_{j=1}^{n_1} A_X[i,j]$,
		$h_Y[i] \leftarrow \frac{1}{n_2} \sum_{j=1}^{n_2} A_Y[i,j]$,
		\State Compute matrix $A_Z \leftarrow \{ a(r,x) \}_{r \in R, x \in Z}$ 
		
		\State ${n_Z} \leftarrow \text{size}(Z)$, for $j=1,\cdots, n_Z$, 
		$w[j] \leftarrow \frac{1}{n_R}  \sum_{i=1}^{n_R}  A_Z[i,j] (h_X[i]-h_Y[i] ) $	
		\State \textbf{return} $w$	
		\EndFunction
		
	\end{algorithmic}
\end{algorithm}

\begin{algorithm}[t]
	\caption{Methods for $k_{\text{spec}}$-MMD}\label{method:spec}
	\begin{algorithmic}[1]
		
		\Function{MMD-Spec}{ $A_X$, $A_Y$, $n_{\text{boot}}$, $f$}
		\Comment{$f \in \R_+^{r_f}$ is the target eigenvalues}	
		\State Concatenate $A  \leftarrow [A_X \, | \, A_Y]$
		\Comment{$A$ is $n_R$-by-$(n_1+n_2)$}	
		\State $U,S,V \leftarrow \textproc{SVD}(A,r_f)$
		\Comment{$U S V^T$ is the best rank-$r_f$ approximation of $A$}
		
		\State $v_X \leftarrow \frac{1}{n_1}(V_{1:n_1,:})^T \textbf{1}_{n_1}$, 
		$v_Y \leftarrow \frac{1}{n_2}(V_{n_1+1:n_1+n_2,:})^T \textbf{1}_{n_2}$
		\Comment{$\textbf{1}_{m}$ is all-ones vector of length $m$}
		\State $T \leftarrow (v_X-v_Y)^T \text{Diag}\{f\} (v_X-v_Y)$
		\Comment{Empirical MMD}
		
		\For{$ k = 1$ to $n_{\text{boot}}$}
		\State $V_{\text{permute}} \leftarrow \text{$V$ with random permuted rows}$
		\State $v_X \leftarrow \frac{1}{n_1}(V_{1:n_1,:})^T \textbf{1}_{n_1}$, 
		$v_Y \leftarrow \frac{1}{n_2}(V_{n_1+1:n_1+n_2,:})^T \textbf{1}_{n_2}$			
		\State $T_{\text{null}}[k] \leftarrow  (v_X-v_Y)^T \text{Diag}\{f\} (v_X-v_Y) $
		\Comment{One sample of MMD under null hypothesis}
		\EndFor	
		
		\State \textbf{return} $T$, $T_{\text{null}}$	
		\EndFunction
		
		\Statex
		
		\Function{Witness-Spec}{$A_X$, $A_Y$, $a$, $Z$, $f$}
		\State Concatenate $A  \leftarrow [A_X \, | \, A_Y]$, $U,S,V \leftarrow \textproc{SVD}(A,r_f)$
		\State $v_X \leftarrow \frac{1}{n_1}(V_{1:n_1,:})^T \textbf{1}_{n_1}$, 
		$v_Y \leftarrow \frac{1}{n_2}(V_{n_1+1:n_1+n_2,:})^T \textbf{1}_{n_2}$
		\State Compute matrix $A_Z \leftarrow \{ a(r,x) \}_{r \in R, x \in Z}$ 
		\State $V_Z  \leftarrow  S^{-1}U^T A_Z$
		\Comment{$V_Z$ is $r_f$-by-$n_Z$, $n_Z$ the size of $Z$}	
		\State $w \leftarrow V_Z^T \text{Diag}\{f\} (v_X-v_Y)$	
		
		\State \textbf{return} $w$	
		\EndFunction	
	\end{algorithmic}
\end{algorithm}

\subsection{Permutation test and choice of the threshold $t_\alpha$}\label{subsec:permutation_test}

The MMD statistics $T$ introduced in Section \ref{sec:theory} is non-parametric, 
which means that the threshold  $t_\alpha$ does not have a closed-form expression for finite $n$.
In practice, one may use a classical method known as the 
\emph{permutation test} \cite{higgins2003introduction},
which is a bootstrapping strategy to empirically estimate $t_\alpha$, and previously used in RKHS MMD \cite{gretton2012kernel}.
The idea is to model the null hypothesis by pooling the two datasets, resampling from that pool, and computing one instance of the MMD under that null hypothesis. Repetitive resampling thus corresponds to permuting the joint dataset multiple times,
and it generates a sequence values of of $T$, the $(1-\alpha)$-quantile of which is used as $t_\alpha$
(see both Algorithm \ref{method:L2} and Algorithm \ref{method:spec}).
Strictly speaking, the power of the test is slightly degraded due to the empirical estimate of the null hypothesis.  
However, as we see in Section \ref{sec:applications}, the test still yields very strong results in a number of examples.

Meanwhile, the limiting distribution of the statistics $T_n$ appears to well approximate the empirical one under ${\cal H}_0$, as shown in Section \ref{subsec:compare}.
This suggests the possibility of determining $t_\alpha$ based on the empirical spectral decomposition of the kernel. In the current work we focus on the bootstrapping approach (permutation test) for simplicity.

\subsection{Empirical witness functions}\label{subsec:witness_empirical}

The population witness functions have been introduced in Section \ref{subsec:witness}, and, by definition,
can be evaluated at any data point $x \in \Omega$.
For the kernel $k_{L^2}$, the empirical version of Eqn. (\ref{eq:wL2population}) is 
\begin{equation}\label{eq:w-L2}
\hat{w}_{L^{2}}(x)=\frac{1}{n_{R}}\sum_{r\in R}a(r,x)(\hat{h}_{X}(r)-\hat{h}_{Y}(r)),
\end{equation}
which can be computed straightforwardly from the empirical histograms $\hat{h}_X$, $\hat{h}_Y$ and the pre-defined kernel $a(r,x)$, 
see Algorithm \ref{method:L2}.

For the kernel  $k_{\text{spec}}$, Eqn. (\ref{eq:wspecpopulation}) is approximated by 
\begin{equation}\label{eq:w-spec}
\hat{w}_{\text{spec}}(x)
=\sum_{k}f_{k}\hat{\psi}_{k}(x)\left(
\frac{1}{n_{1}} \sum_{i=1}^{n_{1}} \hat{\psi}_{k}(x_{i})-\frac{1}{n_{2}}\sum_{j=1}^{n_{2}}\hat{\psi}_{k}(y_{j}) \right),
\end{equation}
where $\hat{\psi}_{k}(x_i)$, $\hat{\psi}_{k}(x_j)$ are computed by SVD of the assymetric kernel matrix, 
and the out-of-sample extension to $\hat{\psi}_{k}(x)$ by Nystrom method,
see Algorithm \ref{method:spec}.

\subsection{Sampling of the reference set}\label{sec:ref_heuristic}

In bio-informatics applications e.g. flow cytometry, 
datasets are usually large in volume so that one can construct the reference set $R$ and covariance field $\{\Sigma_r\}_r$ from a pool of data points, which combines multiple samples, before the differential analysis in which MMD is involved. 
In some application scenario (like diffusion MRI), $R$ and $\{\Sigma_r\}_r$ are provided by external settings.
Thus we treat the procedure of constructing $R$ and $\{\Sigma_r\}_r$ separately from the two-sample analysis, 
which is also assumed in the theory.

In experiments in this work, we sample the reference points randomly in Lebesgue measure using the following heuristic:
given a pool of data points e.g. drawn from $p$ and $q$ or subsampled from larger datasets, 
the procedure loops over batches until $n_R$ points have been generated.
In each loop,
a batch of points are loaded from the pool, 
and candidate points are sampled from the batch according to $p_i^{-1}$ which is the KDE of each point in the batch.
The candidate points are giggled, and points  which have too few neighbors in the pool dataset are excluded. 
The code can be found in the software github repository \url{https://github.com/AClon42/two-sample-anisotropic}.

The covariance fielded is computed by local PCA, when needed. This is related to the issue of ``$\sigma$-selection" in gaussian MMD \cite{gretton2012kernel}, 
where a ``median heuristic" was proposed to determine the $\sigma$ in the gaussian kernel $k_\sigma(x)=e^{-\|x\|^2/2\sigma^2}$.
In our setting, the extension of ``$\sigma$" is the local covariance matrix $\Sigma_{r}$, 
which allows different $\sigma$ at different points, apart from different $\sigma$ along different local directions. 
In estimating $\Sigma_r$, a controlling parameter is  then the size of the local neighborhood, i.e. the $k$-nearest neighbors from which local PCA is computed. 
In the manifold setting, there are strategies to choose $k$ so as to most efficiently estimate local covariance matrix, see e.g. \cite{little2016multiscale}, and it is best done by using different $k$ at different point. 
For simplicity,  in all the experiments we set $k$ to be a fraction of the total number of samples, which may be sub-optimal.
We also introduce a parameter $\tilde{\sigma} > 0$ and set $\Sigma_r = \tilde{\sigma}^2 \Sigma^0_r$, where  $\Sigma^0_r$ is computed via local PCA, so as to make the ``size" of $\Sigma_r$ tunable. 
The parameter $\tilde{\sigma}$ is sampled on a binary grid ($2^k$ for $k=-2,\cdots,2$).

\subsection{Computational complexity}\label{sec:compcomplexity}

We will only discuss the cost of computing the MMD statistics, namely that of $\textproc{MMD-L2}$ (Algorithm \ref{method:L2}) and 
$\textproc{MMD-Spec}$ (Algorithm \ref{method:spec}). 
The extra cost for computing the witness function is negligible, as can be seen from the code (some computation, e.g. that of $h_X$ and $h_Y$ in $\textproc{Witness-L2}$, and the SVD of $A$ in $\textproc{Witness-spec}$ are repetitive for illustrative purpose.)

We firstly discuss $\textproc{MMD-L2}$: 
The cost for computing one empirical MMD statistics $T$ is of $O(n \cdot n_R)$, where $n=n_1+n_2$ is the size of the two samples. The main cost is the one-time construction of the asymmetric kernel matrix $A=[ A_X | A_Y ]$, which also dominates the memory requirements. While the choice of the reference set, and hence the number of reference points $n_R$, is related to the problem being addressed, any amount of structure in the samples leads to a choice of $n_R$ that is $o(n)$. 
This yields major computational and memory benefits as the test complexity is much smaller than the $O(n^2)$, which is the order for computing the U-statistics via a symmetric gaussian kernel without extra techniques. In all applications we've considered and synthetic examples we've worked with, $n_R$ is in the tens or hundreds, and can remain relatively constant as $n$ grows while still yielding a test with large power.  This means the test with an anisotropic kernel becomes linear in the number of points being tested.

In $\textproc{MMD-spec}$, the extra computation is for the rank-$r_f$ SVD of the $n_R$-by-$n$ matrix $A$. 
The computation can be done in $O(n_R n k)$ time via the classical pivoted QR decomposition,
and  may be accelerated by modern randomized algorithms, e.g. the approach in \cite{woolfe2008fast} which takes 
$O(n_R n \log (k) + n k^2)$ time to achieve an accuracy proportional to  the magnitude of the $(k+1$)-th singular value.
The computational cost may be furtherly reduced by making use of the sparse/low-rank structure of the kernel matrix $A$ when possible.
When the kernel $a(r,x)$ almost vanishes outside a local neighborhood of $r$ which has at most $s$ points,
the rows of $A$ are $s$-sparse, 
and then fast nearest neighbor search methods (by Kd-tree or randomized algorithm) can be applied under the local Mahalanobis distance 
and the storage is reduced to $O(n_R s)$. 
When $A$ has a small numerical rank, e.g. only $k$ singular values are significantly nonzero, 
$A$ can be stored in its rank-$k$ factorized form $(U,S,V)$ which can be computed in linear time of $n$.
This will save computation in bootstrapping as both $T_{L^2}$ and $T_{\text{spec}}$ can be computed from $S$ and (row-permuted) $V$ only.

%

\section{Applications}\label{sec:applications}

\subsection{Synthetic examples}\label{subsec:synthetic}

\subsubsection{Example 1: Curve in 2 dimension}

\begin{figure}[t]
\footnotesize
\begin{center}
\begin{tabular}{ccccc}
	\includegraphics[height=.15\textwidth]{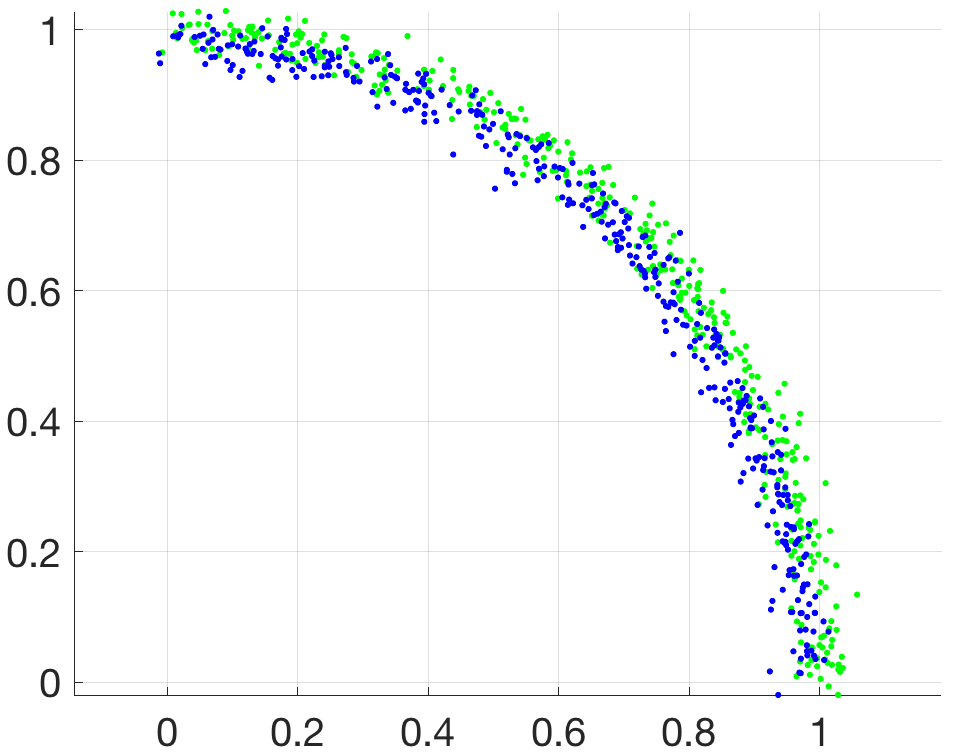}& 
	\includegraphics[height=.15\textwidth,width=.18\textwidth]{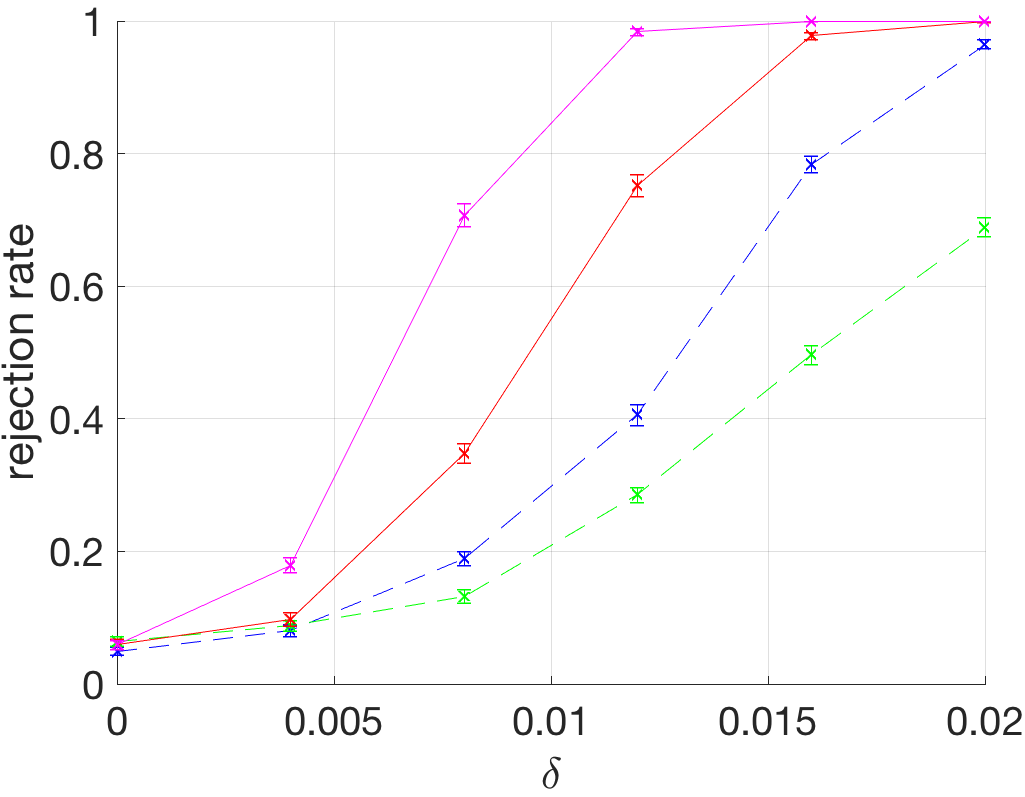}&
	\includegraphics[height=.15\textwidth,width=.18\textwidth]{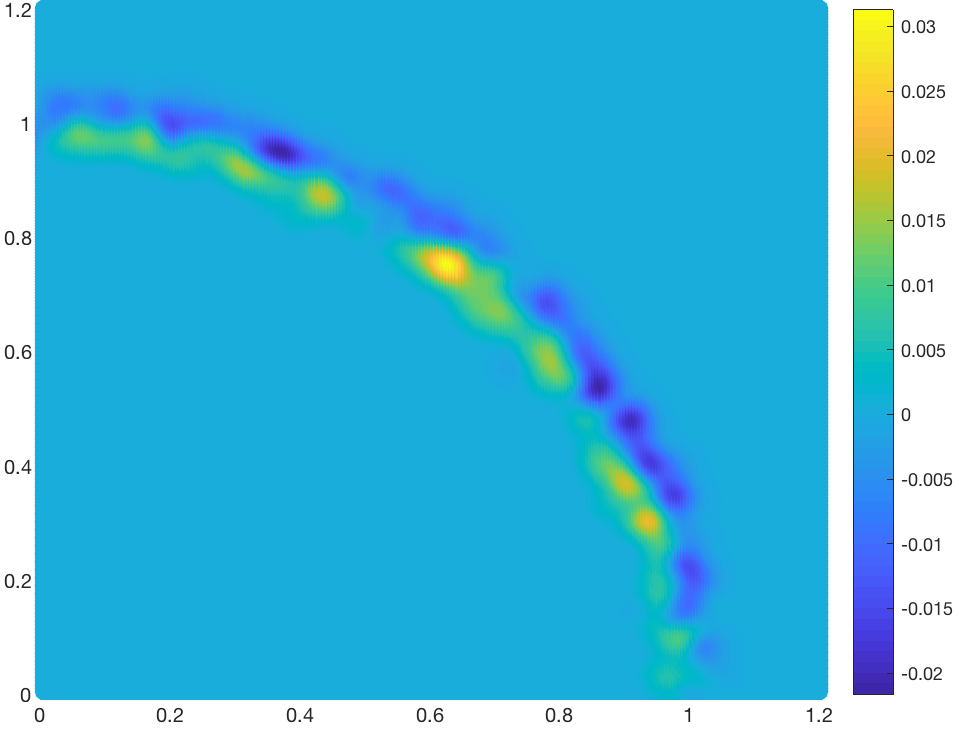}&\hspace{-1cm}
	\includegraphics[height=.15\textwidth,width=.18\textwidth]{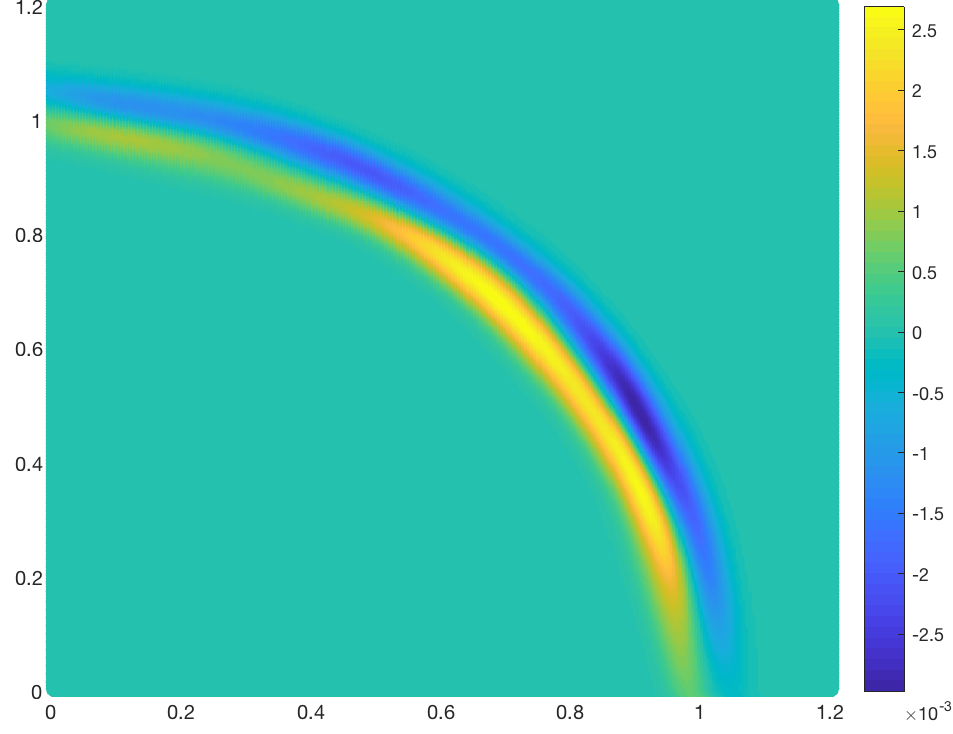}&\hspace{-1cm}
	\includegraphics[height=.15\textwidth,width=.18\textwidth]{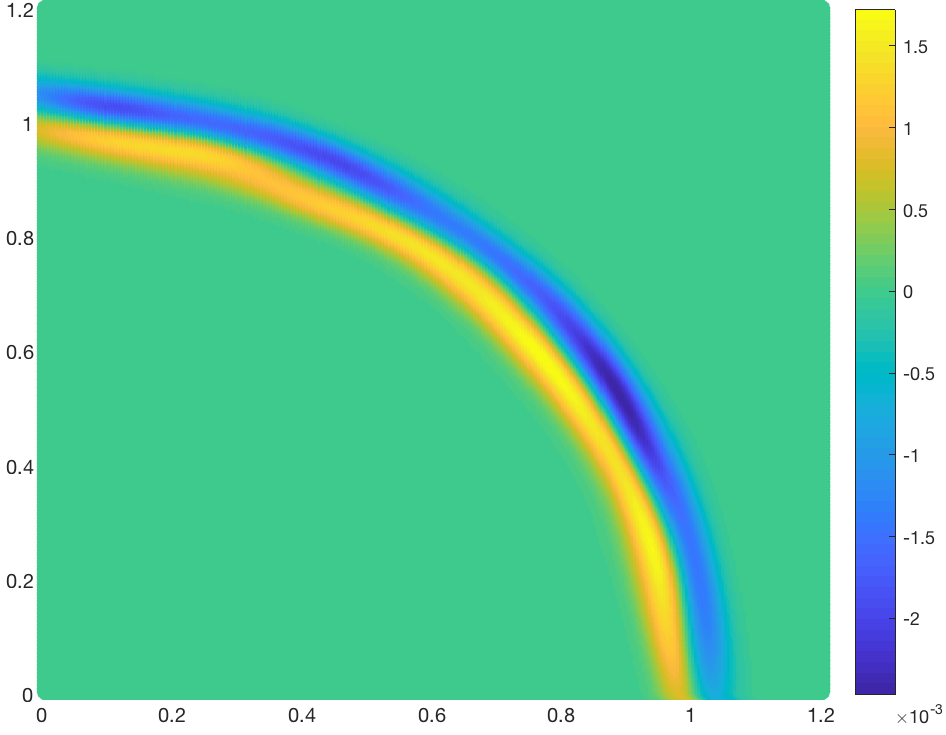}\\
	A1 &\hspace{-1cm}  A2 & A3& \hspace{-1cm}  A4  &\hspace{-1cm} A5
\end{tabular}
\end{center}
	\caption{
		\label{fig:curve1}
		Example 1 in Section \ref{subsec:synthetic},  two curves in  $\R^2$.
		A1: Two samples $X$ and $Y$ in $\R^2$ in green and blue respectively.
		A2: testing power of MMD using 
			(blue) gaussian kernel,
			(red) anisotropic kernel $k_{L^2}$,
			(pink) spectral-filtered anisotropic-kernel $k_{\text{spec}}$,
			and (green) summed KS distance from 20 random projections to 1D.
		Witness functions for two sample test of 
		(A3) gaussian kernel, 
		(A4) $k_{L^2}$ and 
		(A5) $k_{\text{spec}}$ respectively.
	}
\end{figure}

The density $p$ is the uniform distribution on a quarter circle with radius 1 convolved with ${\cal N}(0,\epsilon_x^2)$,
$q$ on a quarter circle with radius $1-\delta$ convolved with ${\cal N}(0,\epsilon_x^2)$,
which is the  same example in Section \ref{subsec:compare} (top left in Figure \ref{fig:compare1}), $\epsilon_x =0.02$.
The departure is parametrized by $\delta > 0$, taking values from 0 to 0.02.
The rejection rate is computed from the average over $1000$ runs according to Algorithm \ref{algo1}, 
as shown in the left of Figure \ref{fig:curve1}, where the errorbars indicate the standard deviation.
For the gaussian kernel, the curve corresponds to the $\sigma$ which has the best performance over a grid, and it is obtained by an intermediate value on the grid.
The witness functions of the three kernels when $\delta=0.2$ are computed  according to Section \ref{subsec:witness_empirical},
and shown in Figure \ref{fig:curve1}.
 
\subsubsection{Example 2: Gaussian mixture in 3 dimension}

\begin{figure}[t]
	\includegraphics[width=.33\textwidth]{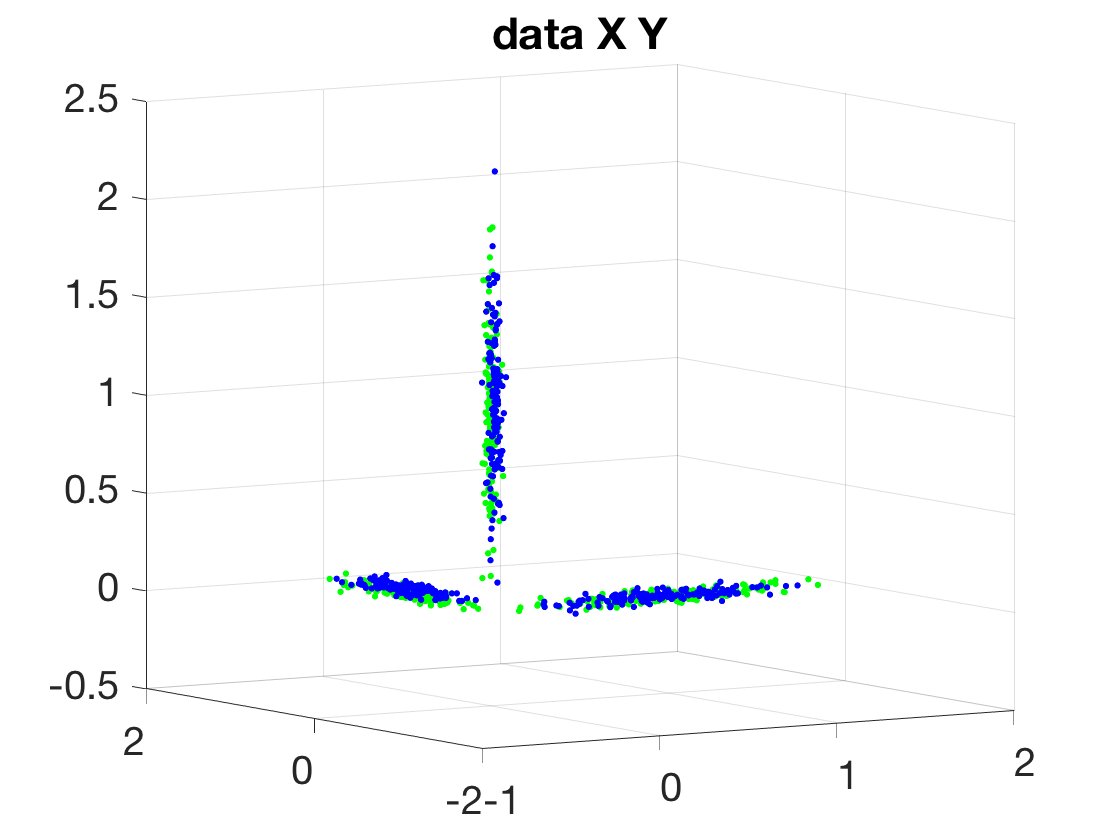}~
	\includegraphics[width=.33\textwidth]{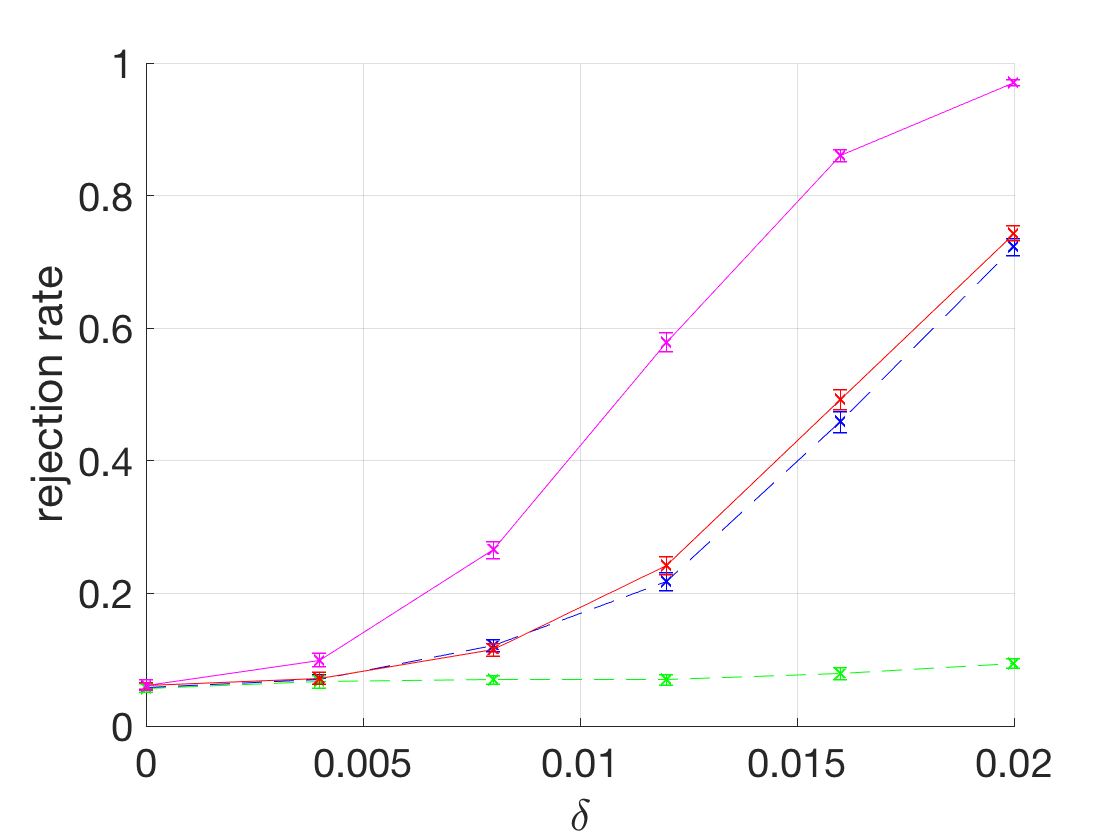}~
	\includegraphics[width=.33\textwidth]{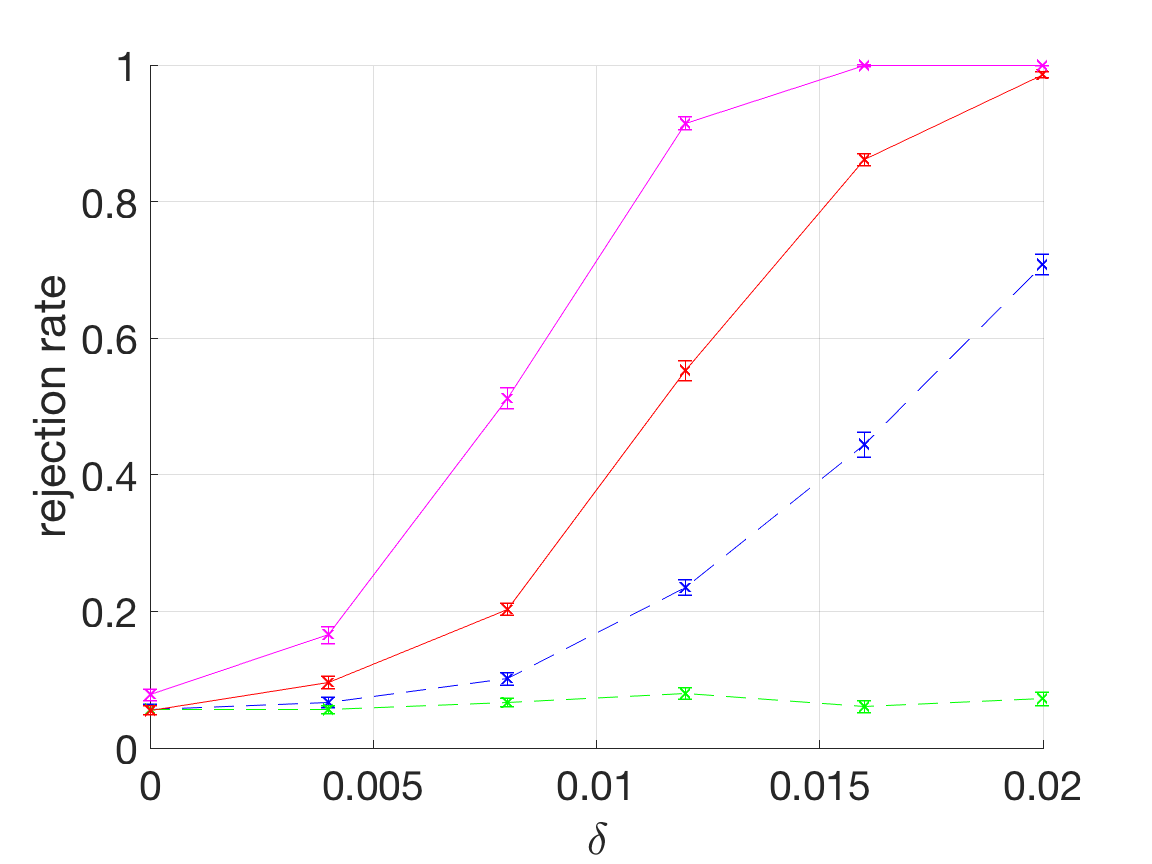}
	\caption{
		\label{fig:mixture1}
		Example 2 in Section \ref{subsec:synthetic}: two gaussian mixtures in  $\R^3$.
		(Left) two samples,
		(Middle) testing power using $\{\Sigma_r\}_r$ determined by local covariance estimation,
		(Right) testing power using $\{\Sigma_r\}_r$ by the true covariance matrix of the belonging cluster.
	}
\end{figure}

Two gaussian mixtures in $\R^3$, $n_1=n_2=200$, $n_R=100$.  The density $p$ consists of 3 equal-size components centering at $(1,0,0)^T$, $(0,1,0)^T$ and $(0,0,1)^T$ respectively, with diagonal covariance matrix $\text{diag}\{ (\frac{1}{3})^2, \epsilon_x^2,  \epsilon_x^2\}$, $\text{diag}\{ \epsilon_x^2, (\frac{1}{3})^2,  \epsilon_x^2\}$ and $\text{diag}\{ \epsilon_x^2,  \epsilon_x^2, (\frac{1}{3})^2\}$, $\epsilon_x=0.02$; and density $q$ differs from $p$ by centering the three means at $(1,\delta,0)^T$,  $(0, 1,\delta)^T$ and $(\delta,0,1)^T$, while letting the covariance matrices remain the same. The constant $\delta$ takes value from 0 to 0.02.
Results shown in Figure \ref{fig:mixture1}, similar to Figure \ref{fig:curve1}.

\subsection{Flow cytometry data analysis}
Flow Cytometry is a laser based technology used for cell counting.  It is routinely used in diagnosing blood diseases by measuring various physical and chemical characteristics of the cells in a sample.  This leads to each sample being represented by a multi-dimensional point cloud of tens of thousands of cells that were measured.  

Two natural questions arise in comparing various flow cytometry samples.  The first is a supervised learning question: can a statistical test detect the systemic differences between a set of healthy people and a set of unhealthy people?  The second question is an unsupervised learning question: can the distance measure be used to determine the pairwise distance between any two samples, and will that distance matrix yield meaningful clusters?

We address both of these questions for two different flow cytometry datasets.  We compare isotropic gaussian MMD to anisotropic MMD and demonstrate, in all cases, our anisotropic MMD test has more power and yields more accurate unsupervised clusters.  We only show the pairwise distance embedding for the reference set asymmetric kernel $a(r,x)$ due to the untenable $O(n^2)$ cost of computing each of the ${k}\choose{2}$ pairs.

\subsubsection{AML dataset}

\begin{figure}[t]
\footnotesize
	\begin{center}
		\begin{tabular}{ccc}
			\includegraphics[height=.2\textwidth]{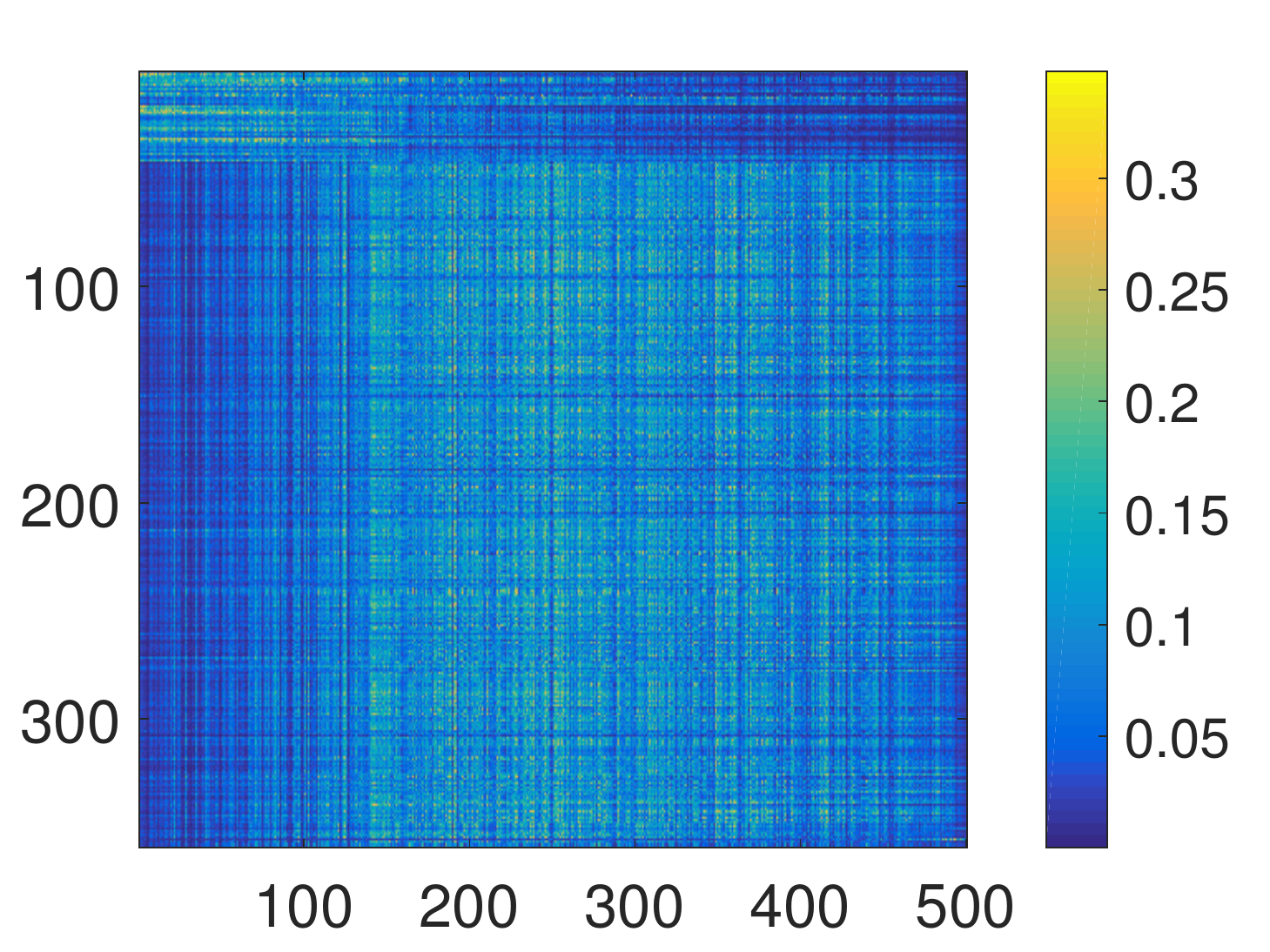} & \hspace{-1cm}
			\includegraphics[height=.2\textwidth]{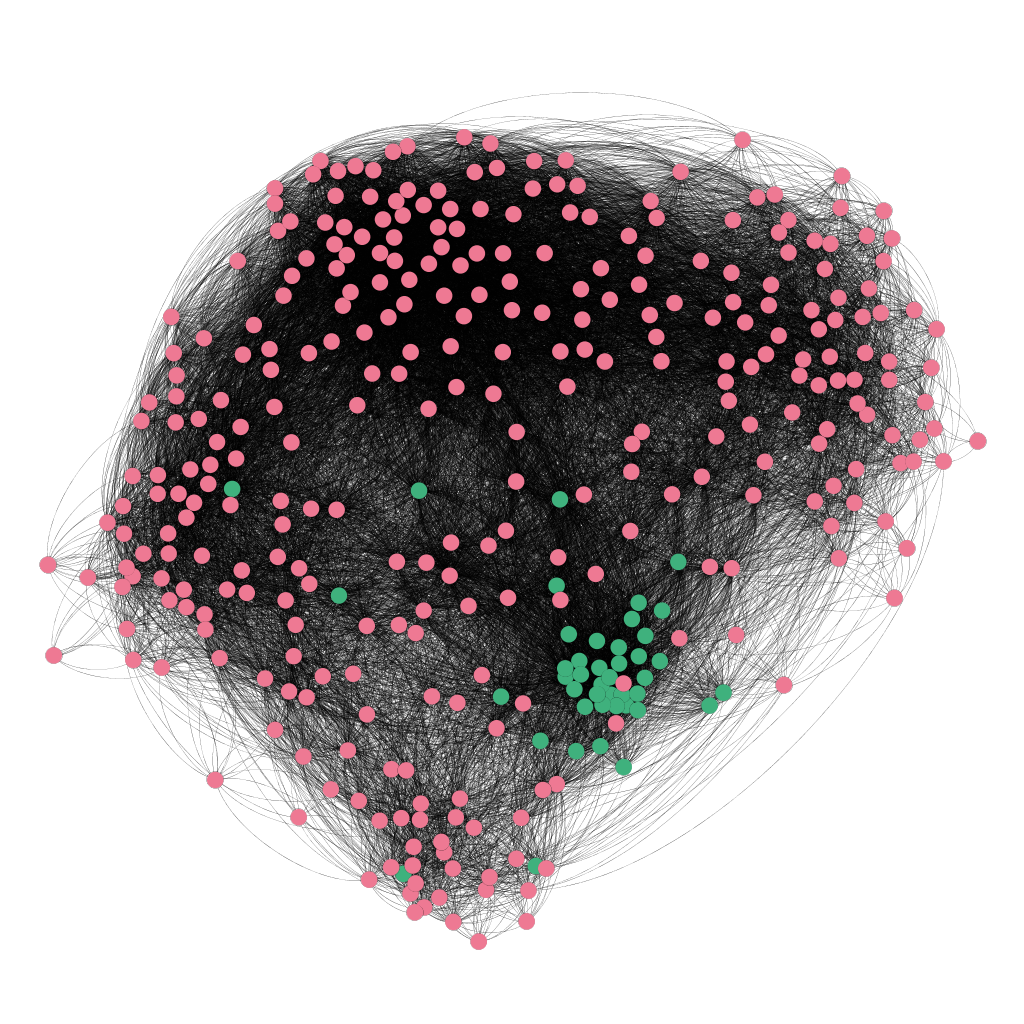} &
						\includegraphics[height=.2\textwidth]{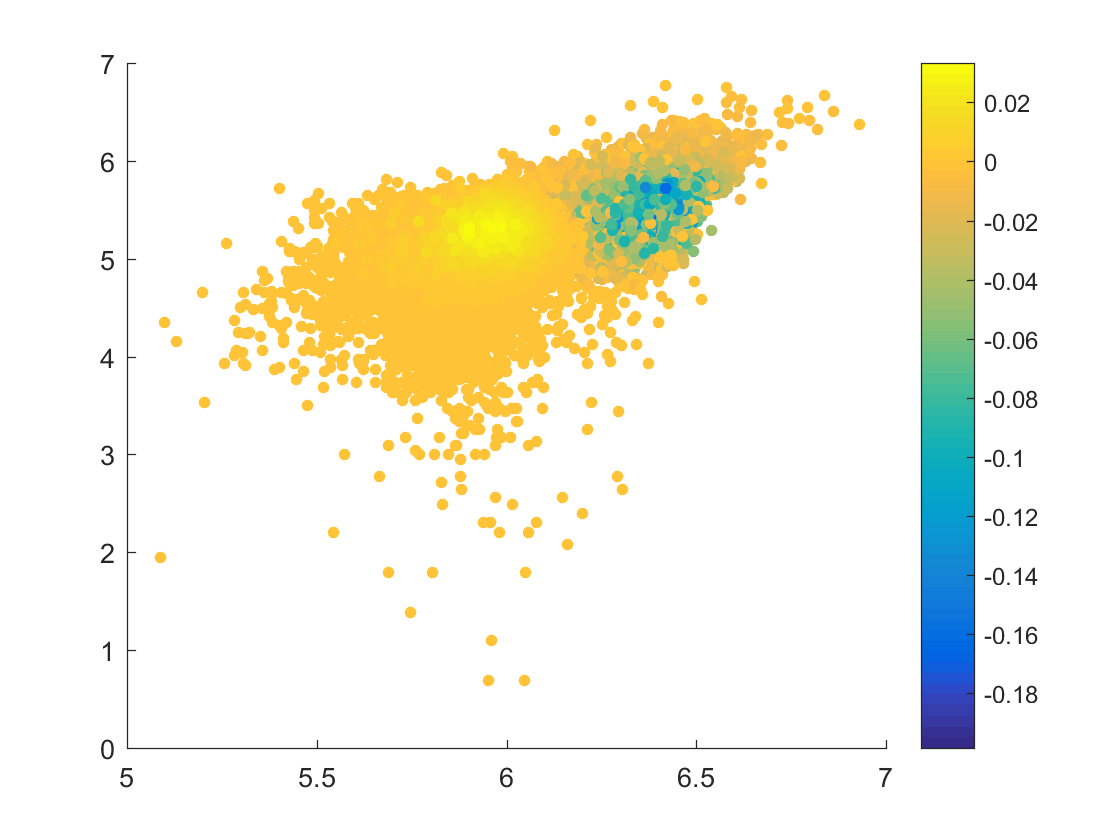} \\
			A1 & \hspace{-1cm} A2 & B1 \\
			\includegraphics[height=.2\textwidth]{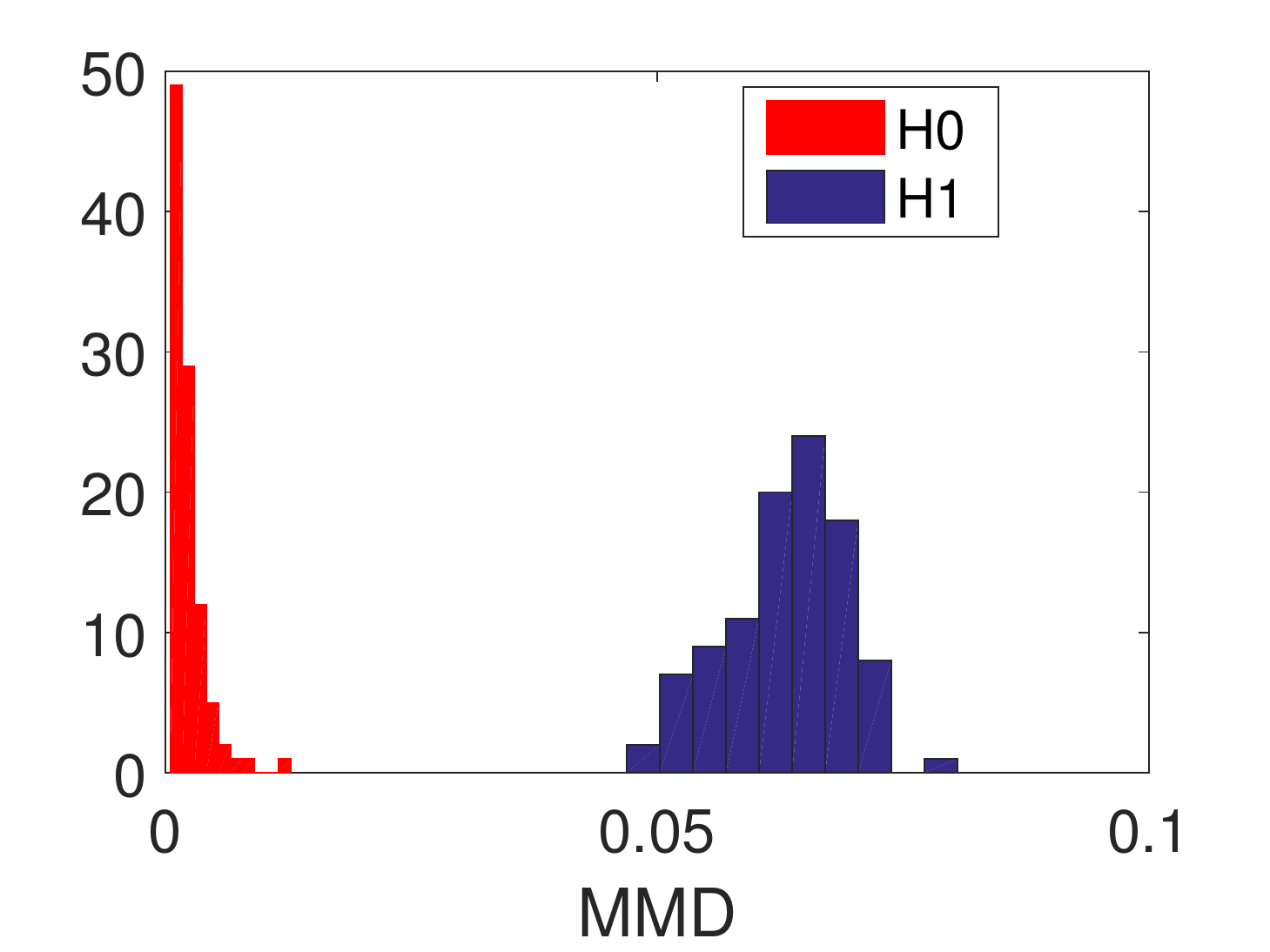} & \hspace{-1cm}
			\includegraphics[height=.2\textwidth]{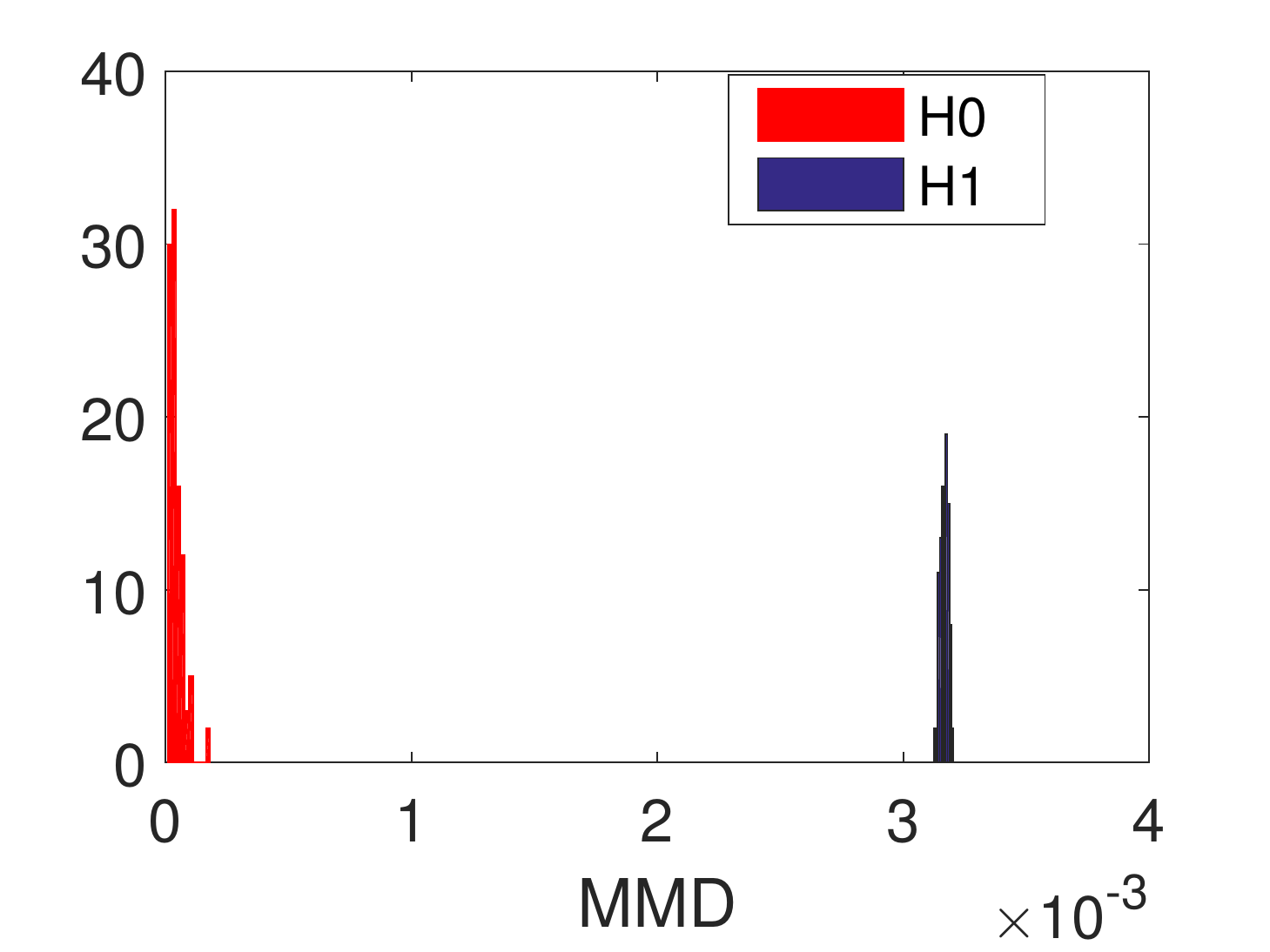} & 
					\includegraphics[height=.2\textwidth]{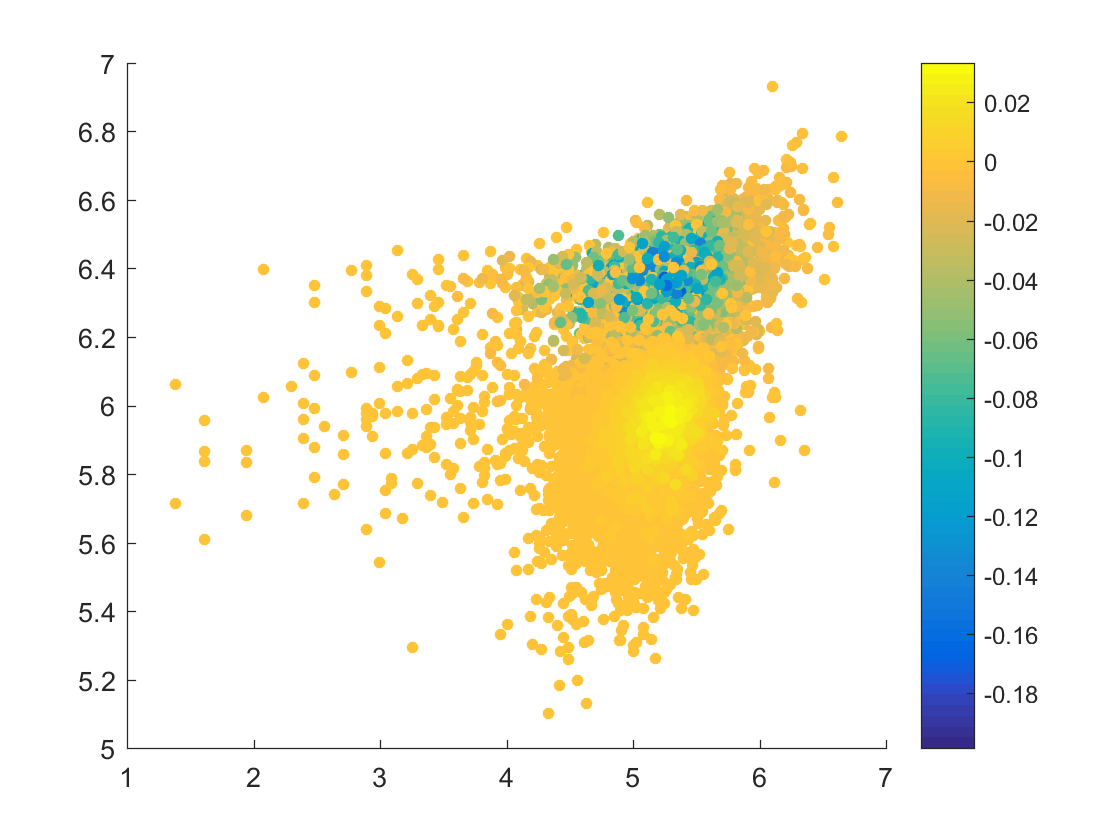} \\
			A3 & \hspace{-1cm} A4 & B2
		\end{tabular}
	\end{center}
	\caption{\label{fig:aml}
		A1: unsupervised histograms 
		for AML patients, A2: network clustering of pairwise distances between patients (green for AML, red for healthy),
		A3: permutation test with isotropic kernel, A4: permutation test with anisotropic kernel.
		B1-B2: different 2D slices of point cloud colored by witness function. Blue are cells more likely to be indicative of AML.
		}
\end{figure}

Acute myeloid leukemia (AML) is a cancer of the blood that is characterized by a rapid growth of abnormal white blood cells.  While being a relatively rare disease, it is incredibly deadly with a five-year survival rate of $27\%$ \cite{deschler2006acute}.  AML is diagnosed through a flow cytometry analysis of the bone marrow cells, and while there are features that distinguish AML in certain cases, other cases can be more difficult to detect.  
State of art supervised methods include SPADE, Citrus, etc. See \cite{bruggner2014automated} and the references therein.

The data is from the public competition Flow Cytometry: Critical Assessment of Population Identification Methods II (FlowCAP-II).  This dataset consists of 316 healthy patients and 43 patients with Acute myeloid leukemia (AML), where each person has approximately $30,000$ cells measured across 7 physical and chemical features.  

We use an anisotropic kernel with only $500$ reference points.  We create an unsupervised clustering by computing the pairwise distances $d^2[i,j] = T_{L^2}$ between person $i$ and person $j$.  In Figure \ref{fig:aml}, we display the network of these people constructed by weighting each edge as the exponential of the negative distance $d^2[i,j]$ properly scaled.  When supervising the process and running a two sample test between the pool of healthy cells and the pool of unhealthy cells, we see that the anisotropic kernel yields significantly better separation and lower variance than the isotropic gaussian kernel.

We also examine the witness function in Figure \ref{fig:aml} 
that is generated by the anisotropic kernel, as introduced in Section \ref{subsec:witness}.  This yields a tool for visualizing the separation between the two samples in the original data space.  This gives a way to communicate the decision boundary to the medical community, which uses visualization of the 2D slices as the diagnostic tool for determining whether the patient as AML.

\subsubsection{MDS dataset}
Myelodysplastic syndromes (MDS) are a group of cancers of the blood.  There is more variability in how MDS presents itself in the cells and flow cytometry measurements.  The data we work with came from anonymized patients that were treated at Yale New Haven Hospital.  After choosing to examine surface markers CD16, CD13, and CD11B, along with several physical characteristics of the cells, we're left with 72 patients that were initially diagnosed with some form of MDS and 87 patients that were not.  These patients are represented by about $25,000$ cells in 8 dimensions.

While MDS is more difficult to detect than AML, the unsupervised pairwise graph, created the same way as in Figure \ref{fig:aml}, still yield a fairly strong unsupervised clustering, as we see in Figure \ref{fig:mds}.  When supervising the process and running a two sample test between the pool of healthy cells and the pool of unhealthy cells, we see that the anisotropic kernel yields strongly significant separation between the two classes, unlike the isotropic gaussian kernel.

\begin{figure}[t]
\footnotesize
	\begin{center}
		\begin{tabular}{cccc}
			\includegraphics[height=.2\textwidth]{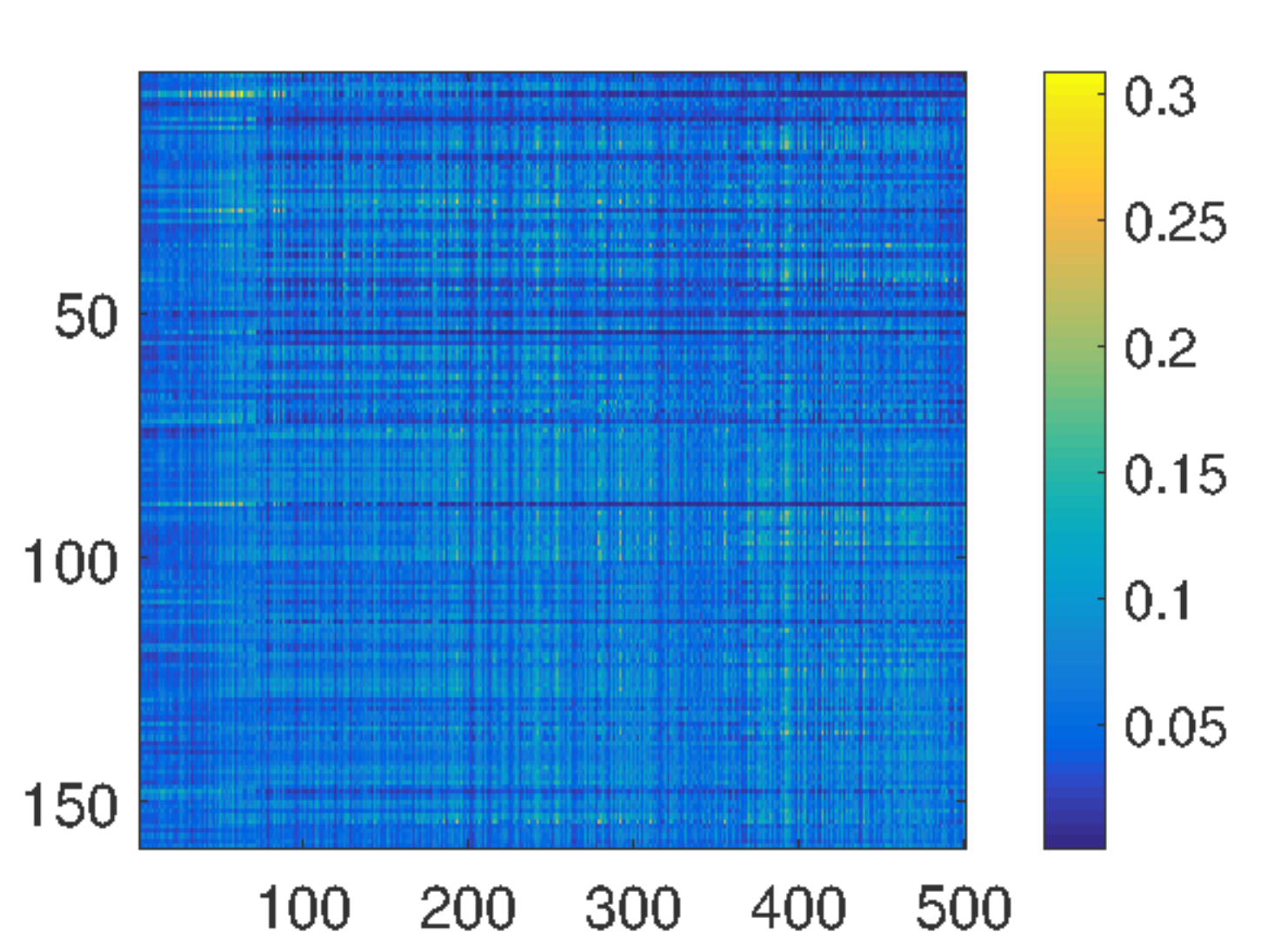} & \hspace{-1cm}
			\includegraphics[height=.2\textwidth]{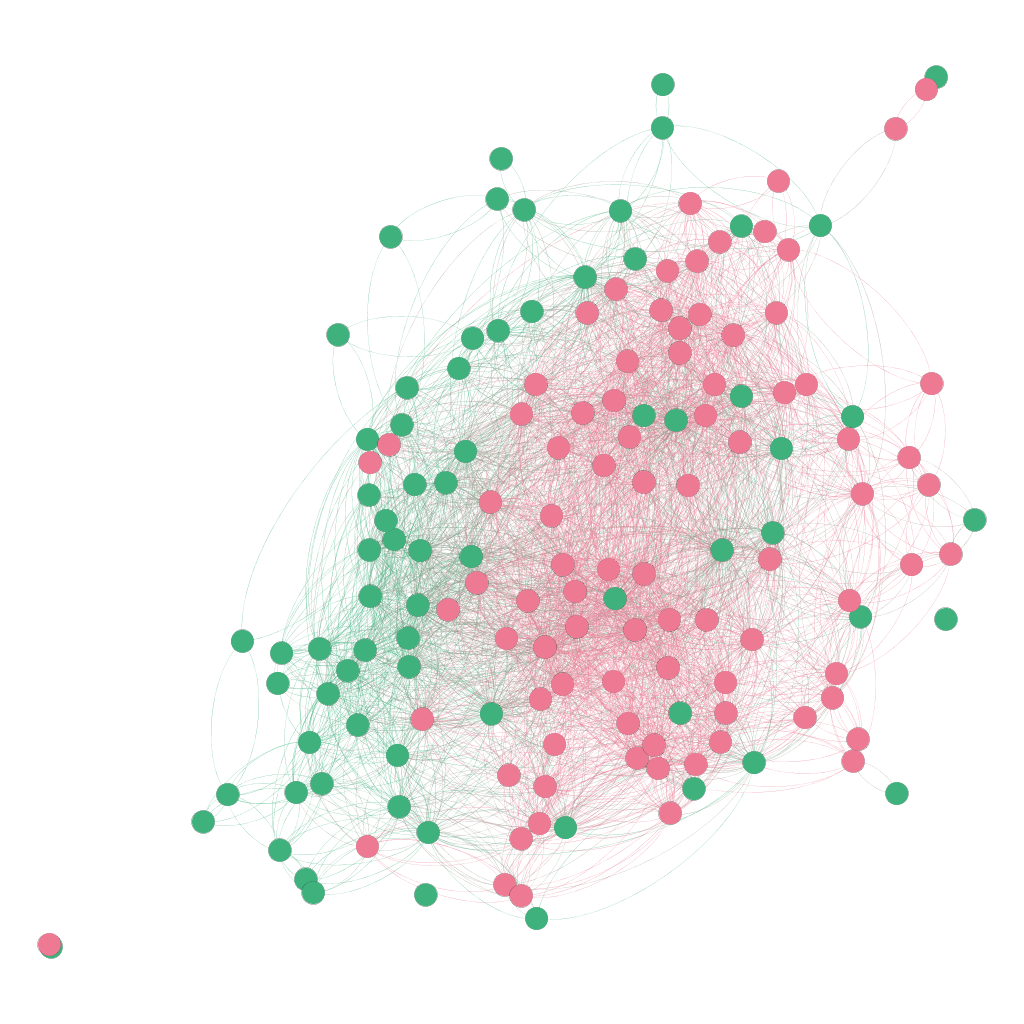} & \hspace{-1cm}
			\includegraphics[height=.2\textwidth]{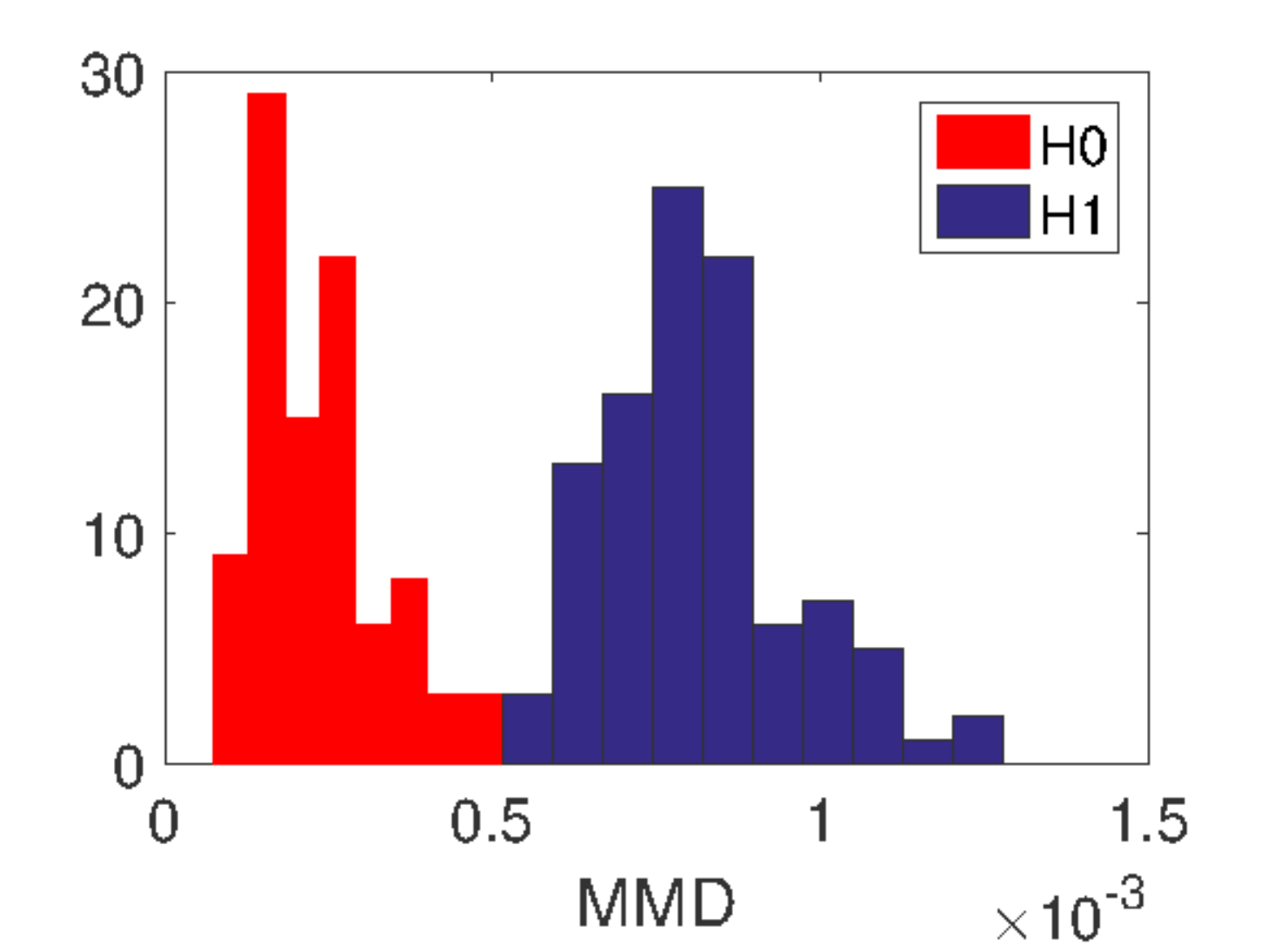} &  \hspace{-1cm}
			\includegraphics[height=.2\textwidth]{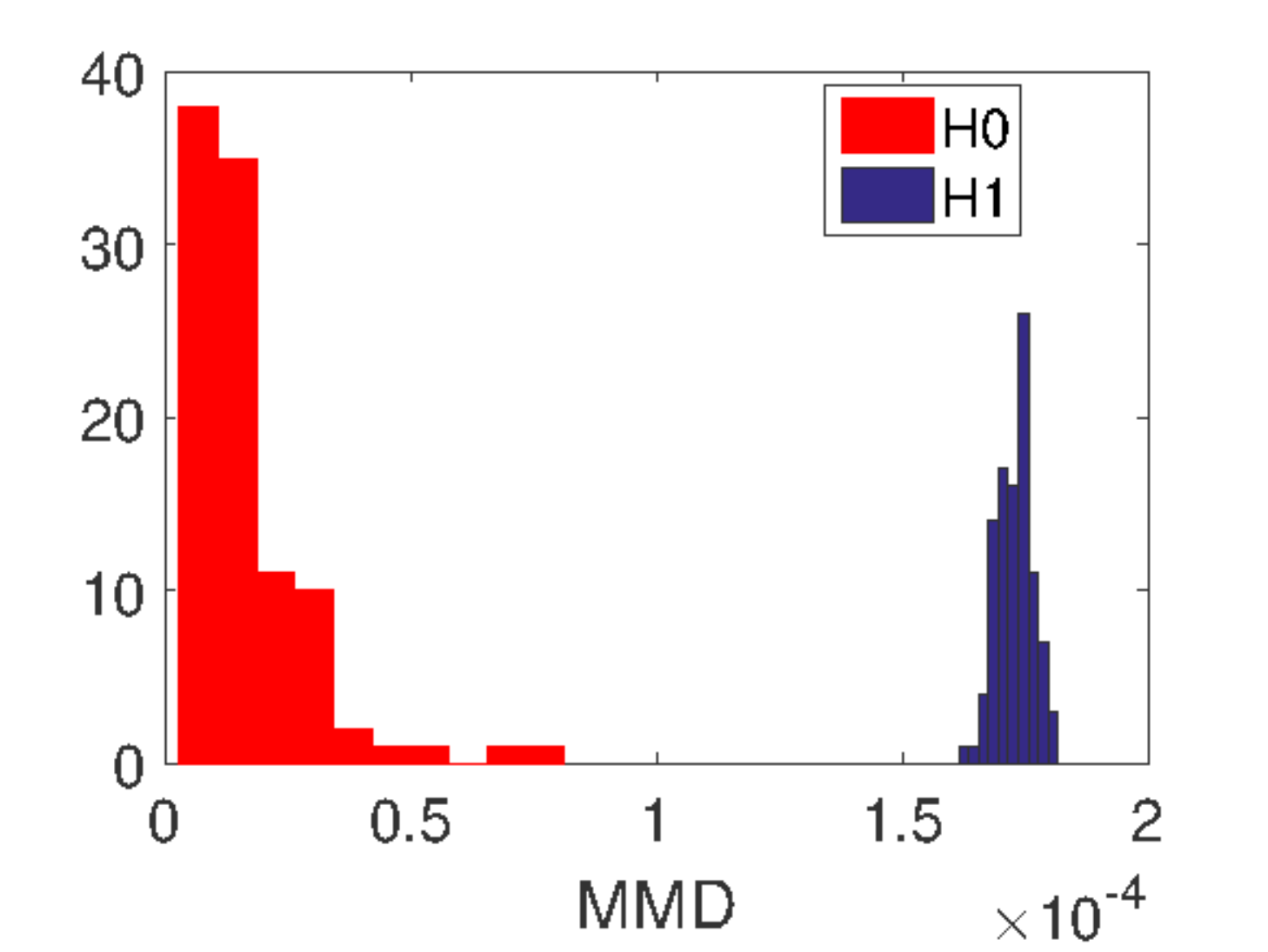} \\
A1 & \hspace{-1cm} A2 &  \hspace{-1cm} A3 &  \hspace{-1cm} A4 
		\end{tabular}
	\end{center}
	\caption{
		A1: unsupervised histograms  
		for MDS patients, A2: network clustering of pairwise distances between patients (green for MDS, red for healthy),
		A3: permutation test with isotropic kernel, A4: permutation test wtih anisotropic kernel.
	}\label{fig:mds}
\end{figure}

\subsection{Diffusion MRI imaging analysis}\label{subsec:diffuion-MRI}

\begin{figure}[t]
	\footnotesize
	\begin{center}
		\begin{tabular}{ccccc}
			\includegraphics[height=.15\textwidth]{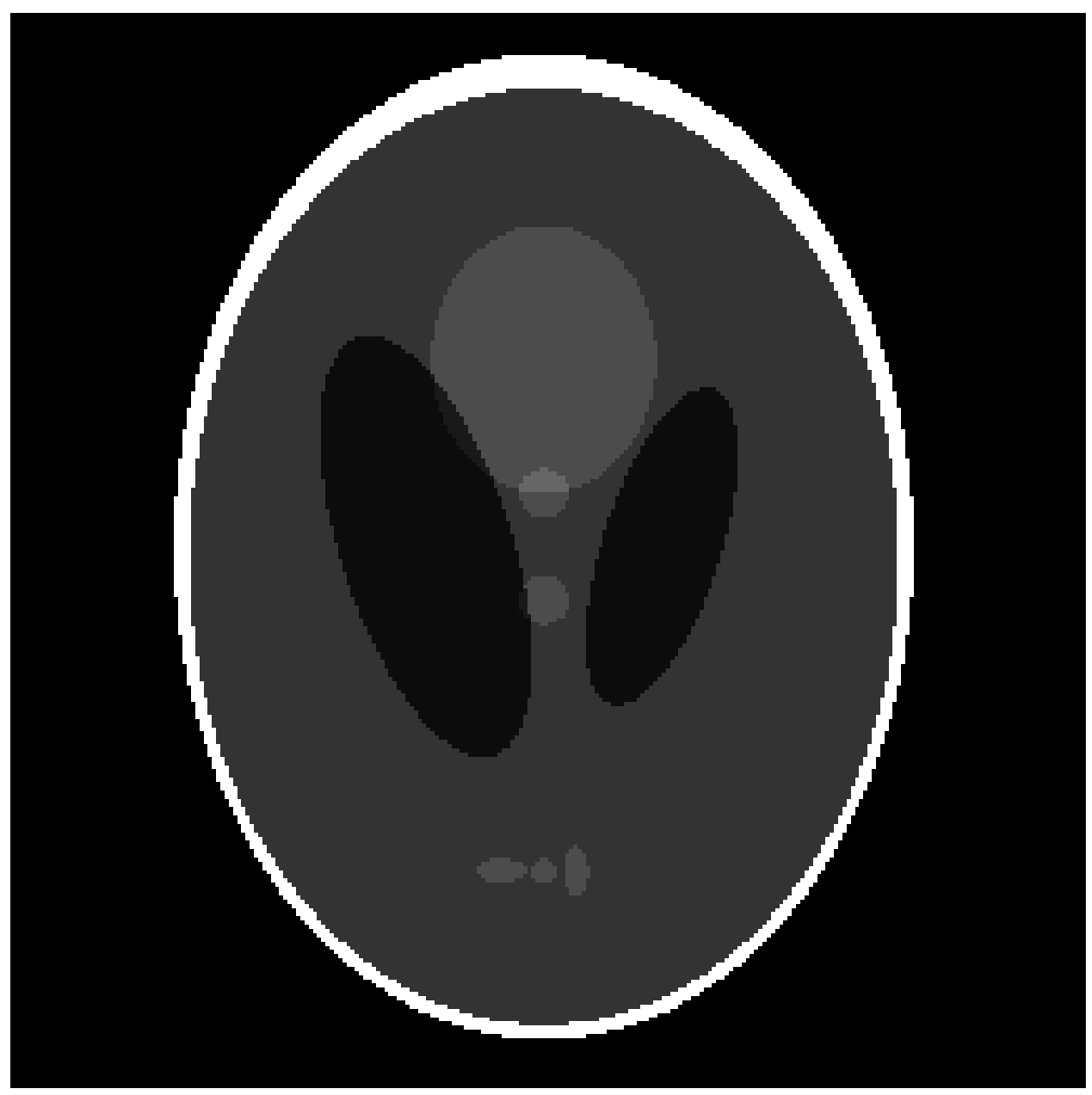} & \hspace{-.5cm}
			\includegraphics[height=.15\textwidth]{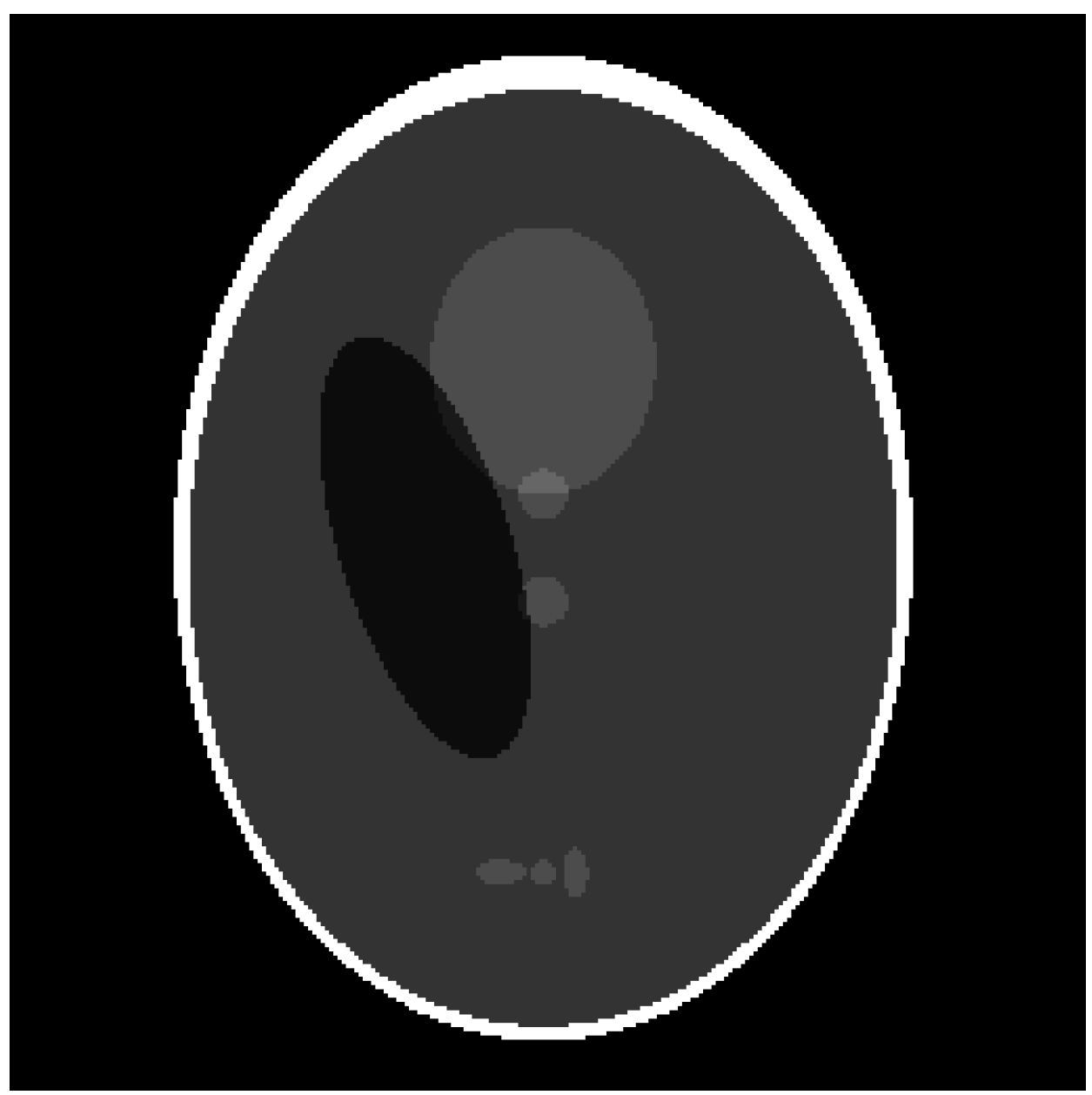} &
						\includegraphics[height=.15\textwidth]{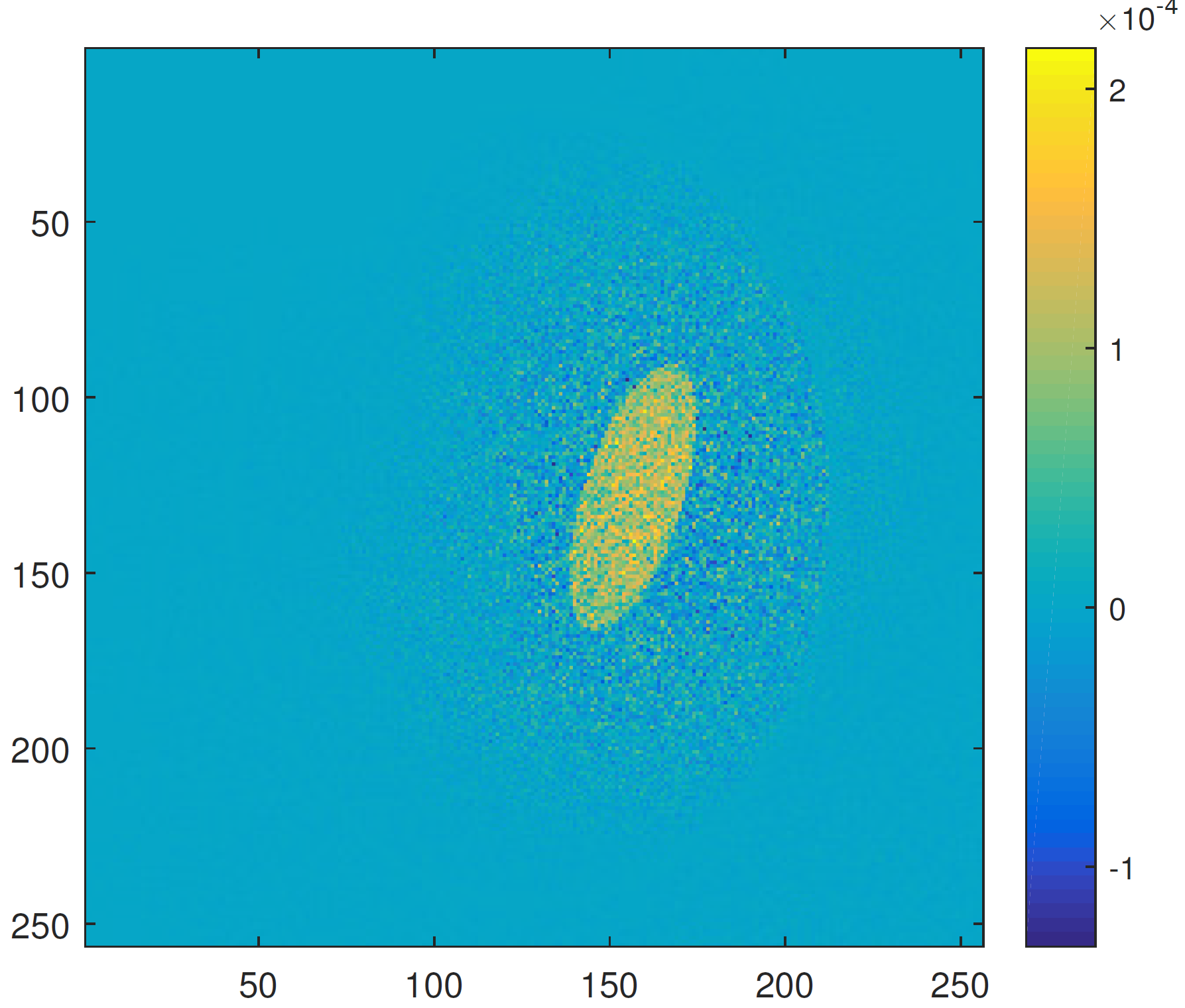} & \hspace{-.5cm}
						\includegraphics[height=.15\textwidth]{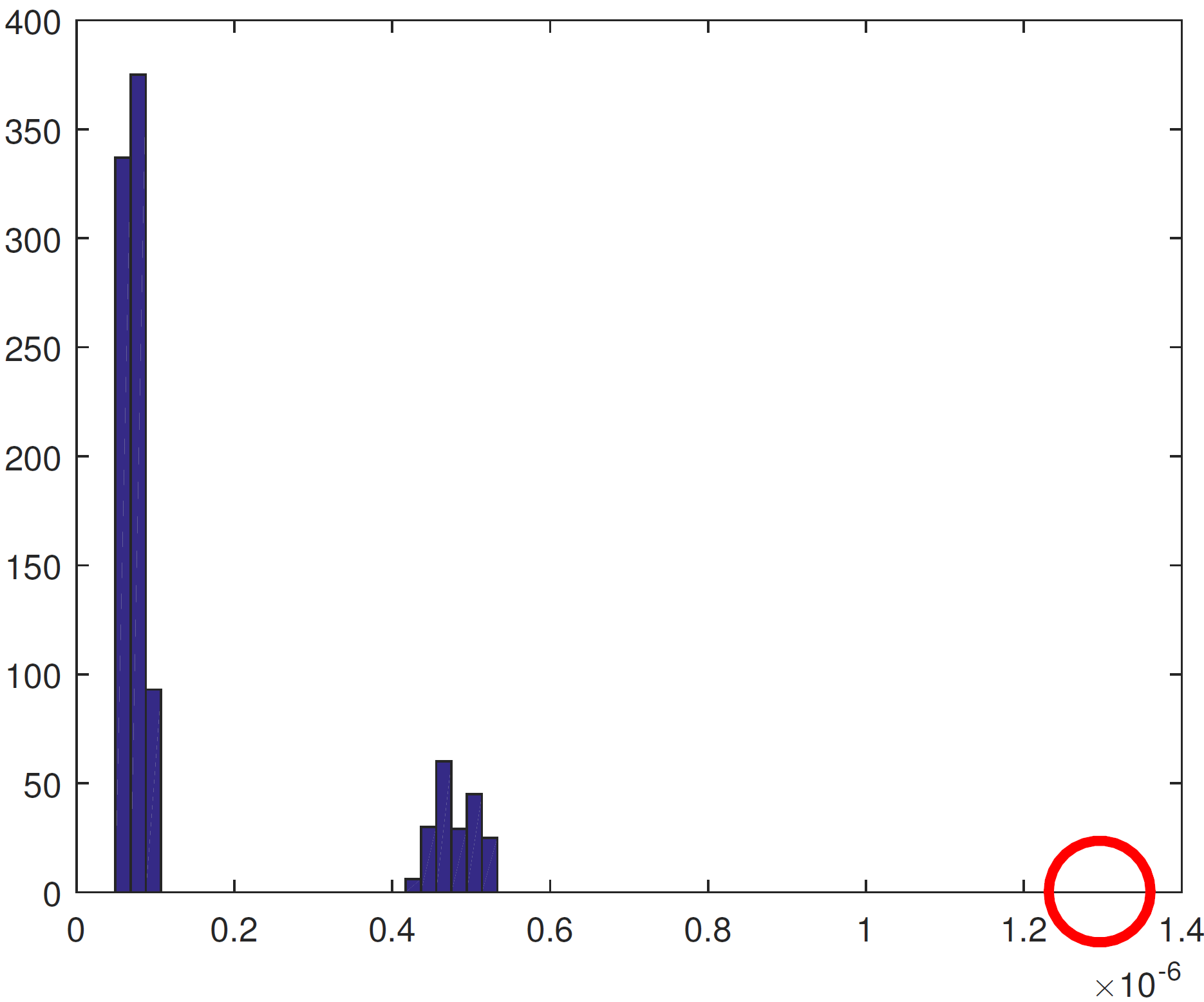} & 
						\includegraphics[height=.13\textwidth]{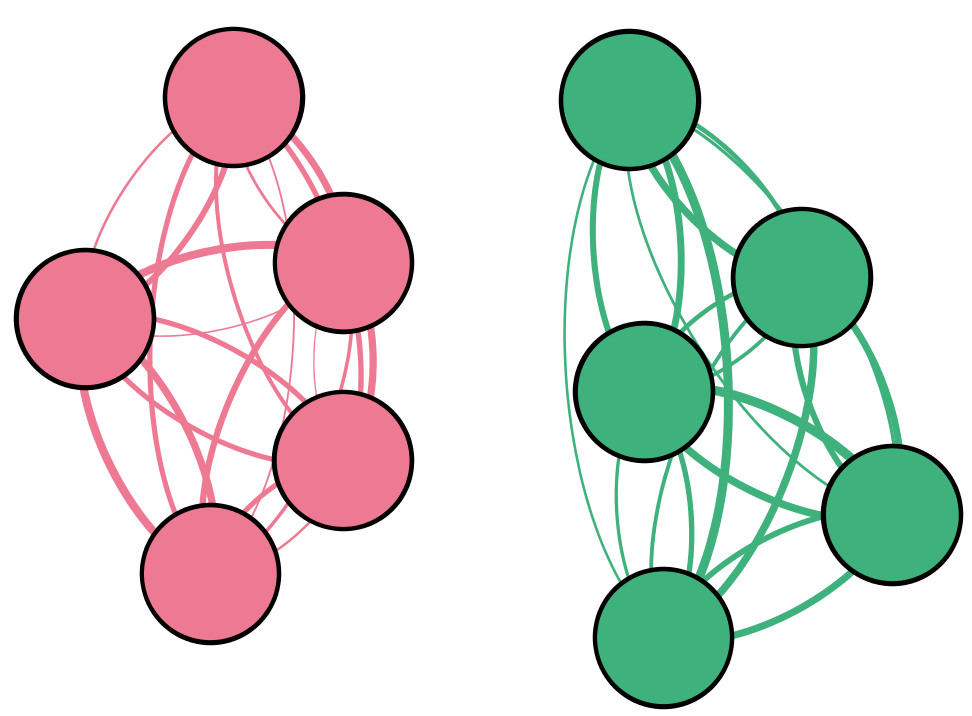} \\
			A1 &\hspace{-1cm} A2 & 
					B1 &\hspace{-1cm} B2 &\hspace{-1cm} B3 \\
			\includegraphics[height=.15\textwidth]{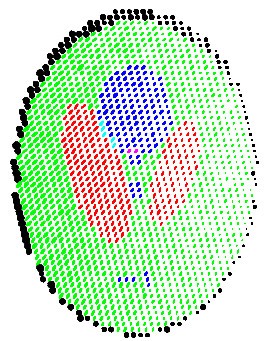} &\hspace{-.5cm}
			\includegraphics[height=.15\textwidth]{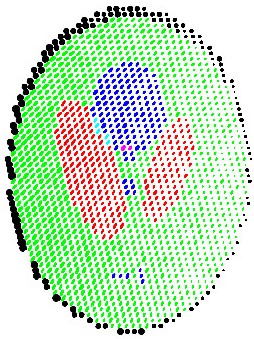} & 
						\includegraphics[height=.15\textwidth]{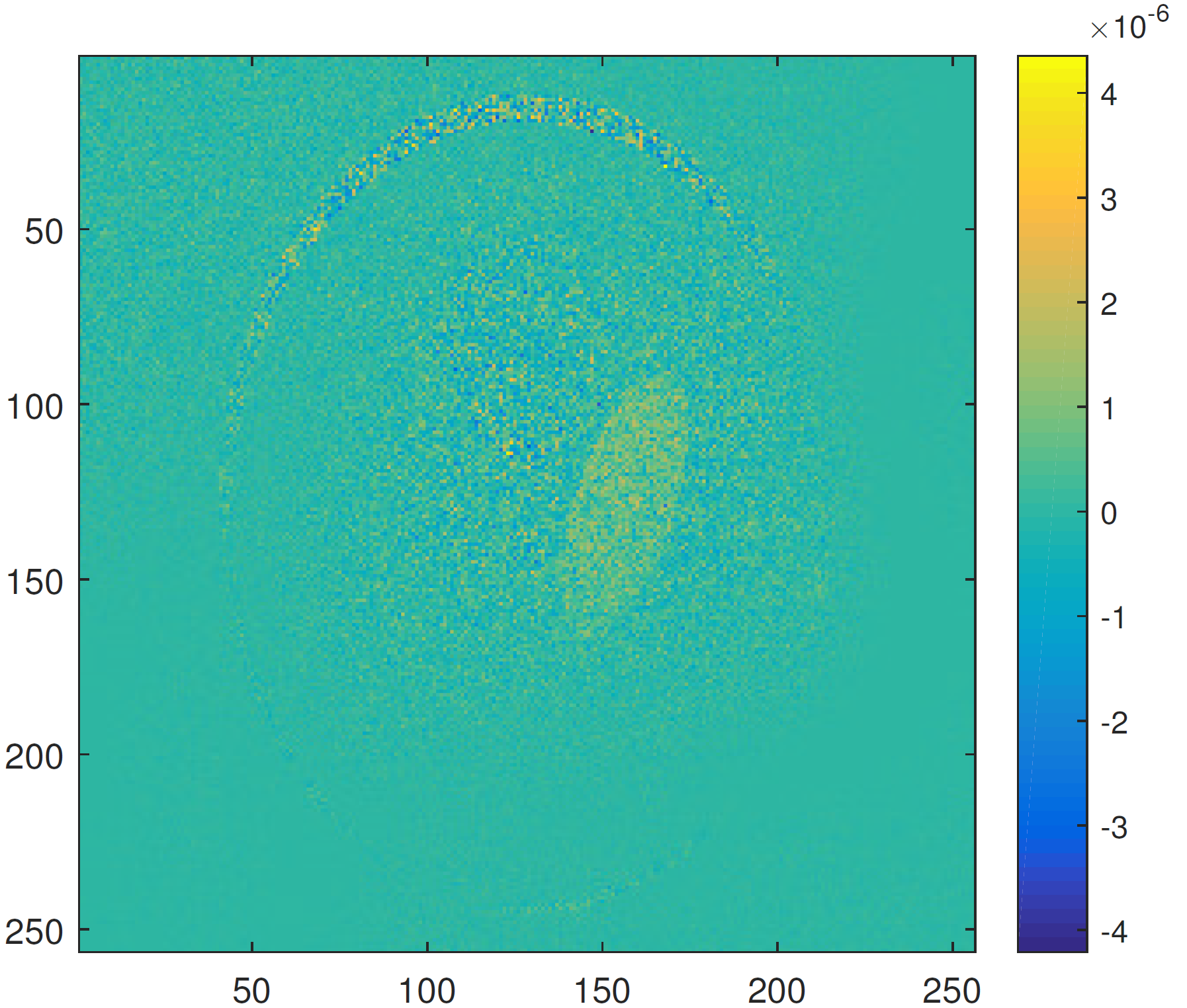} & \hspace{-.5cm}
						\includegraphics[height=.15\textwidth]{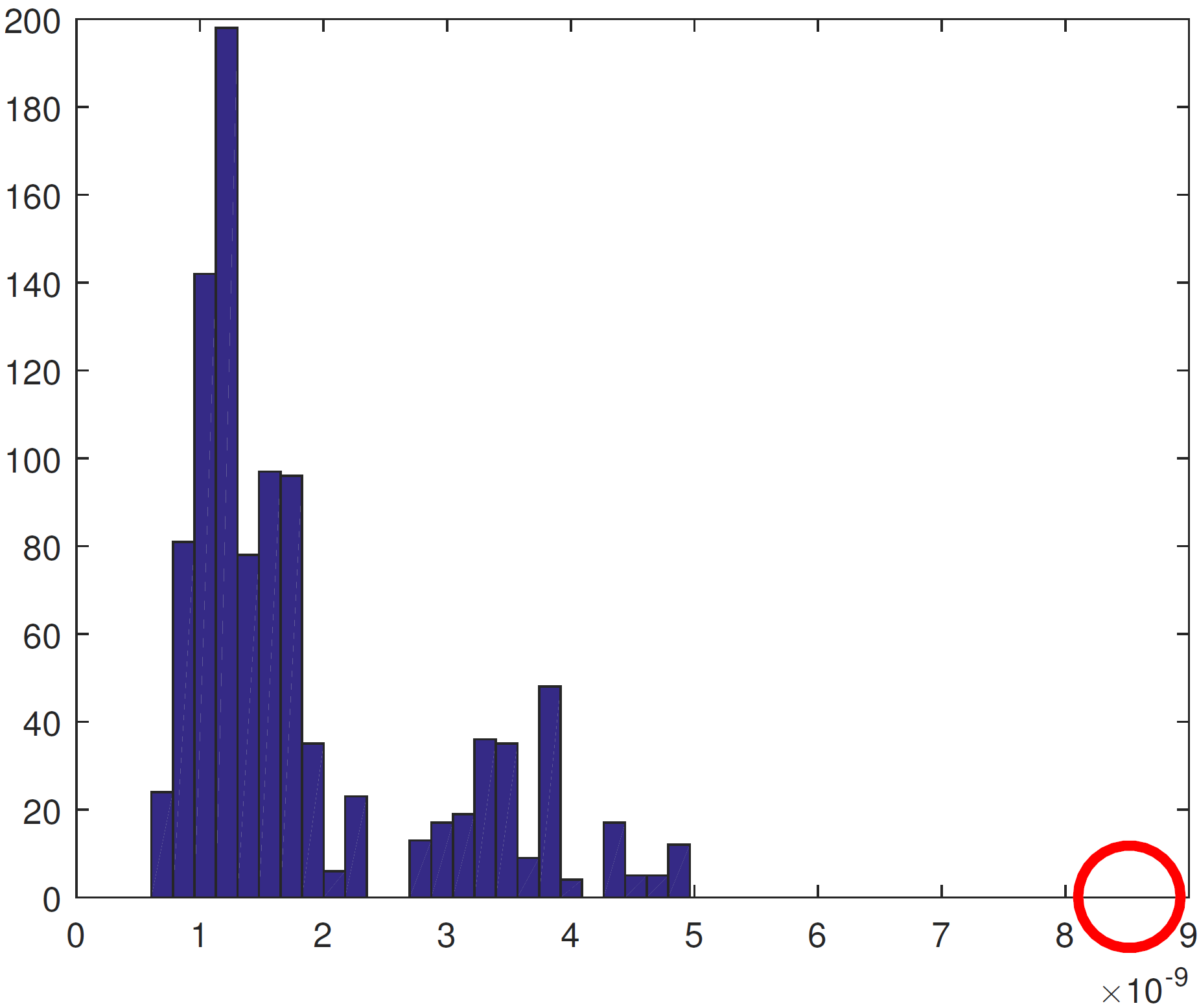} &
						\includegraphics[height=.13\textwidth]{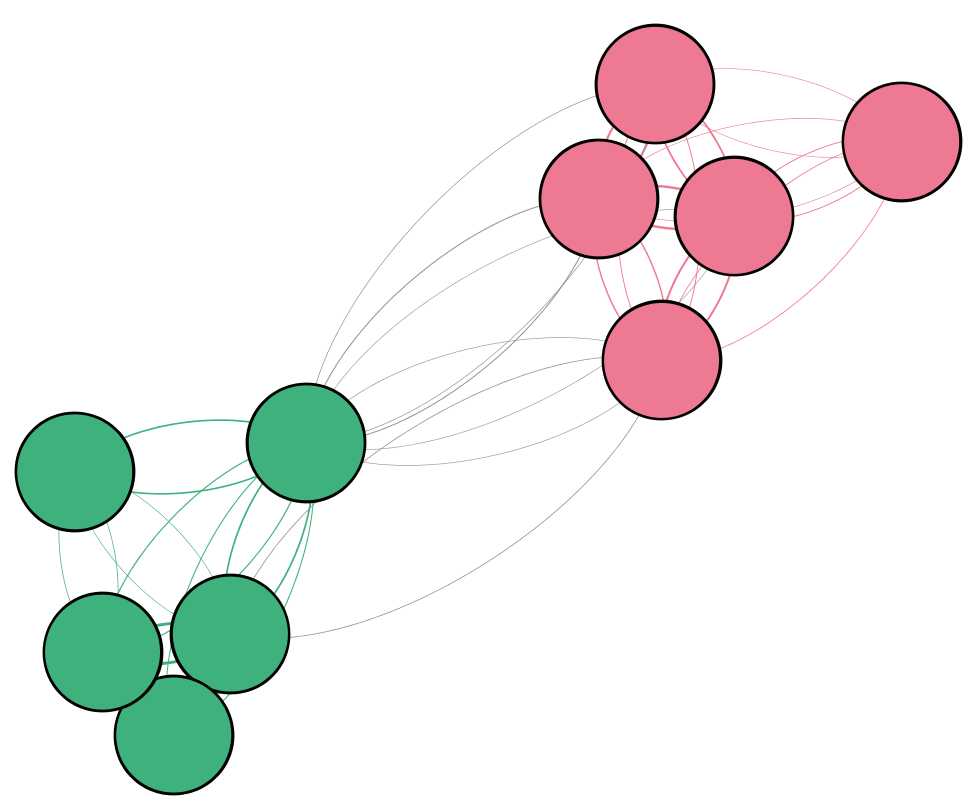} \\
			A3 &\hspace{-1cm} A4 & 
					B4 &\hspace{-1cm} B5 &\hspace{-1cm} B6 \\
			\includegraphics[height=.15\textwidth]{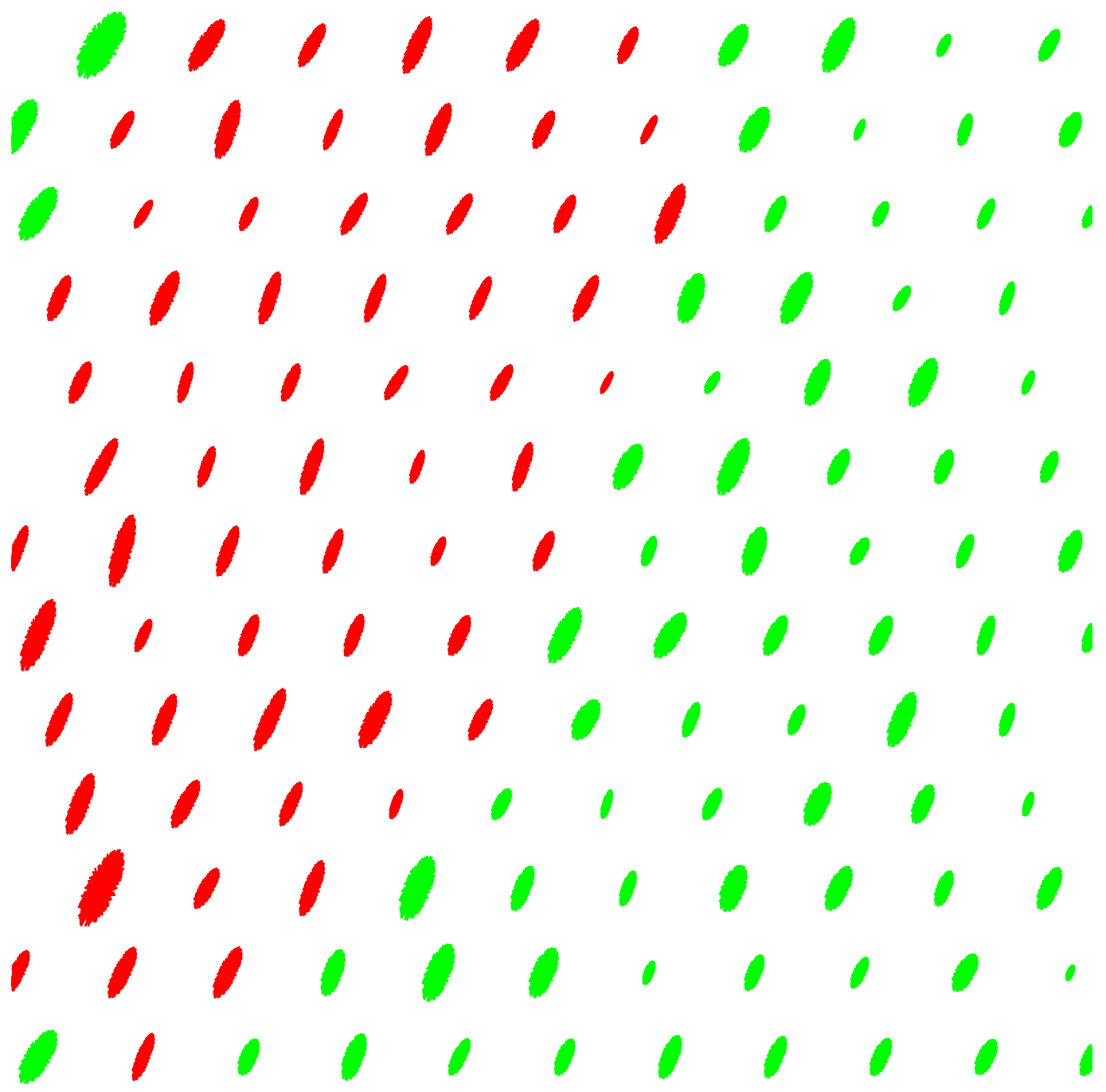} & \hspace{-.5cm}
			\includegraphics[height=.15\textwidth]{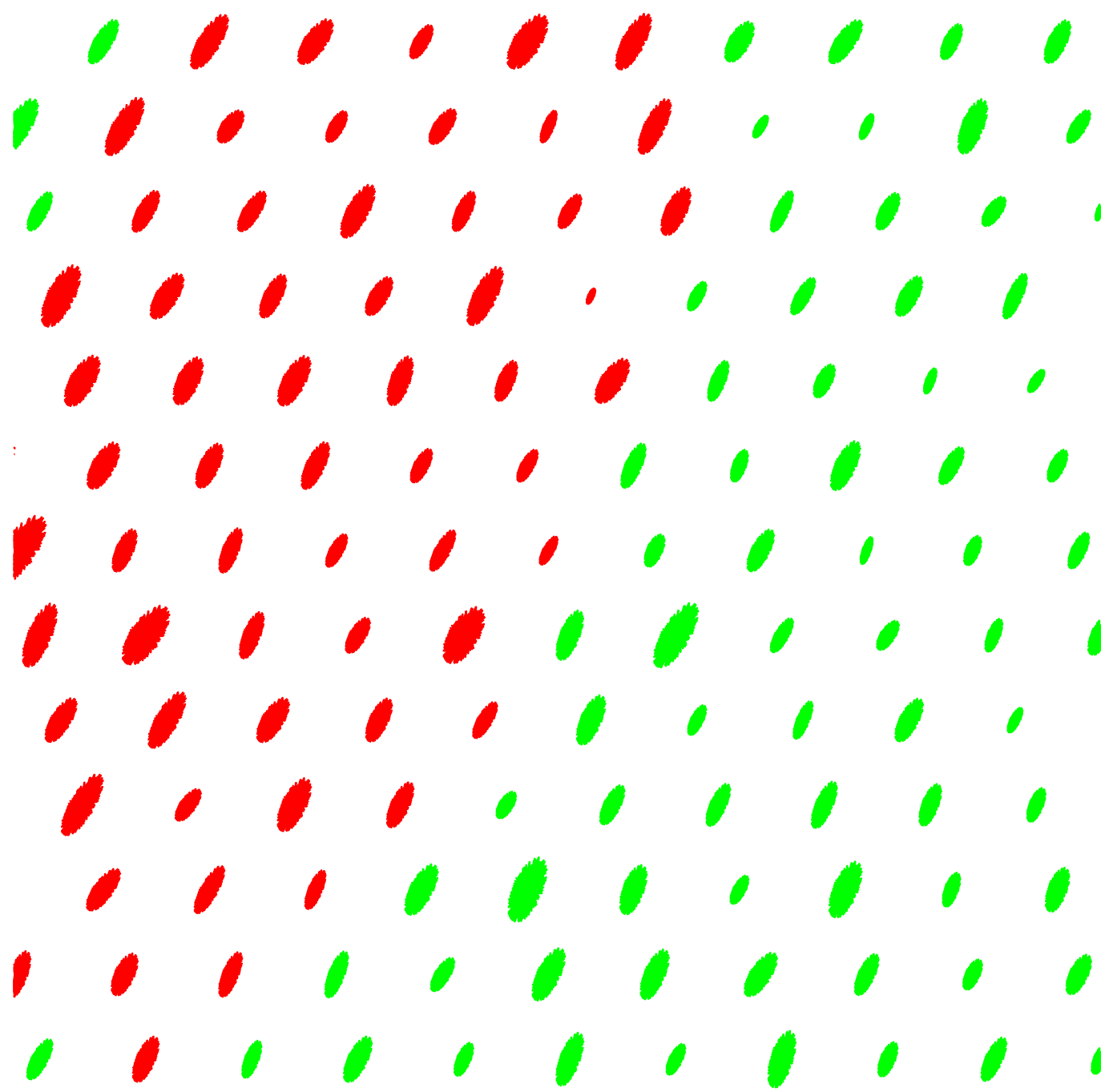} &
						\includegraphics[height=.15\textwidth]{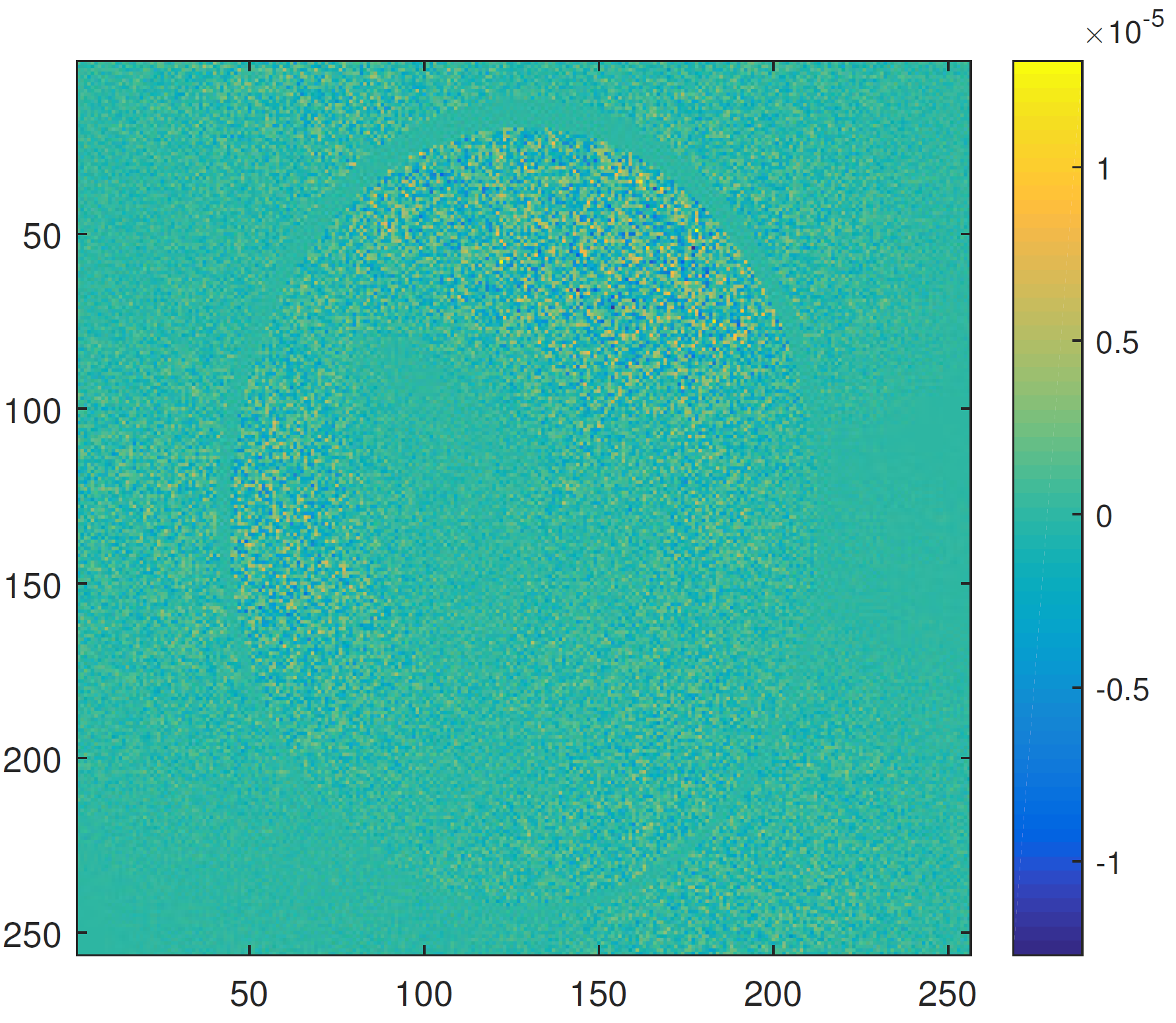} & \hspace{-.5cm}
						\includegraphics[height=.15\textwidth]{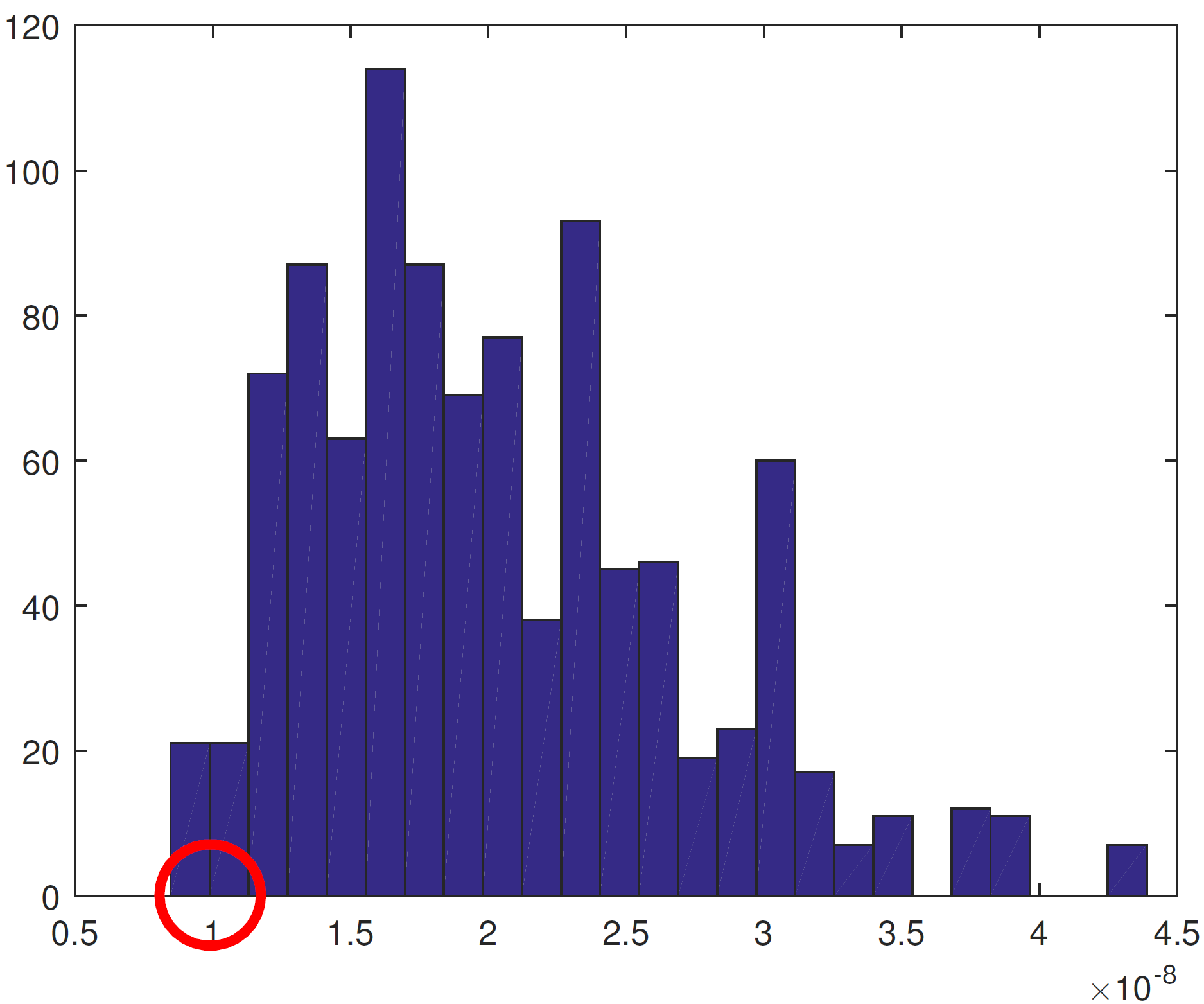} &
						\includegraphics[height=.13\textwidth]{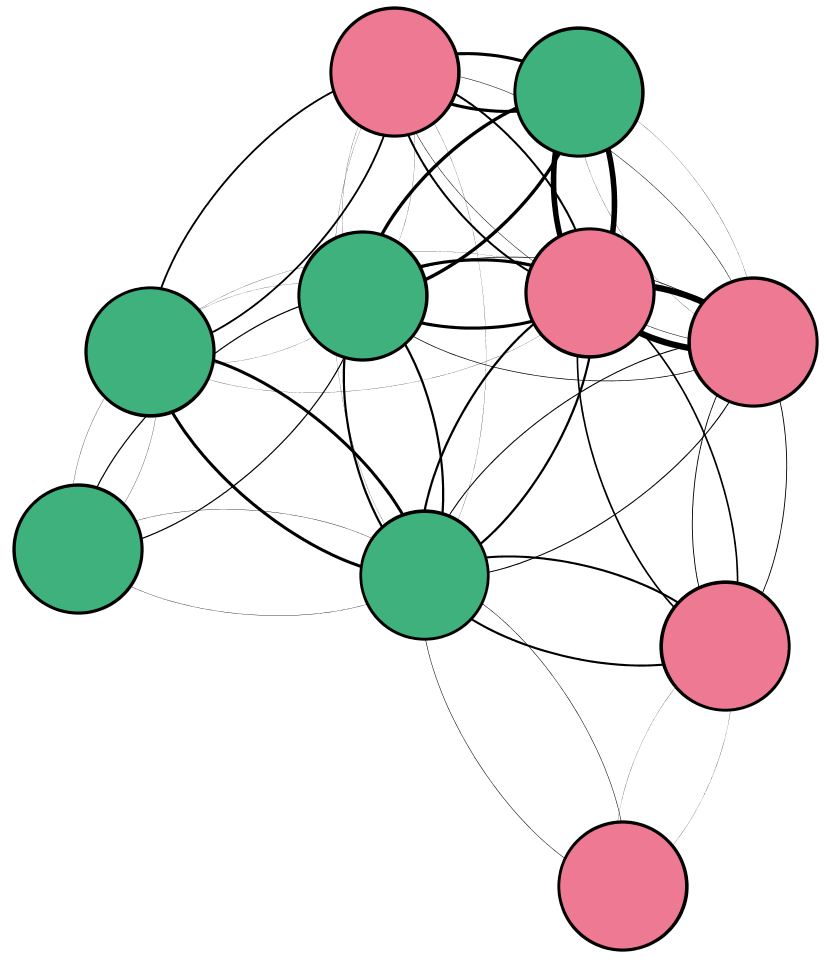} \\
			A5 &\hspace{-1cm} A6 & 
					B7 &\hspace{-1cm} B8 &\hspace{-1cm} B9 \\
		\end{tabular}
	\end{center}
	\caption{
	A1-A6: (top) synthetic brains with grayscale of diffusion tensor eccentricity, (middle) diffusion tensors with artificial coloring, (bottom) zoom in on area of difference. (left) Group 1, (right) Group 2.
	B1-B9: (left) witness function, (middle) permutation test, (right) pairwise graph from MMD distance. (top) anisotropic kernel under alternative hypothesis, (middle) isotropic kernel under alternative hypothesis, (bottom) anisotropic kernel under null hypothesis.
	}\label{fig:brainsOriginal}
\end{figure}

Diffusion weighted MRI is an imaging modality that creates contrast images via the diffusion of water molecules.  Various regions of the brain diffuse in different ways, and the patters can reveal details about the tissue architecture.  At a low level, each pixel in a 3D brain image generates a 3D diffusion tensor (i.e. covariance matrix) that describes the local flow of water molecules.  

An important question in diffusion MRI analysis is to identify regions of the brain that systemically differ between groups of healthy and sick individuals.  We attack this problem by comparing the distributions of the diffusion tensors in various regions of the brain, thus framing it as a multiple sample problem.  Every brain is co-registered so that the pixels overlap.  

Real images are around $200\times200\times200$, so the amount of memory needed to build a square symmetric kernel, even a sparse one, is completely prohibitive.  For this reason, we instead consider a set of reference pixels $R$ that subsamples the image, and construct a anisotropic kernel $a:(X,T)\rightarrow R \rightarrow [0,1]$, where $x_i$ is the location of pixel $i$ and $T_{x_i}$ is the $3\times3$ diffusion tensor at pixel $i$.  This kernel must enforce both locality in the pixel space, and the local behavior of the diffusion tensors.  The latter requires measuring covariance matrices over the space of positive semi-definite matrices.  

Fortunately, the study of covariance matrices on the space of positive semi-definite matrices is a well studied phenomenon \cite{basser2007spectral}.  The main takeaway is that there is an isomorphism from a $3\times 3$ diffusion tensor $T$ and its vectorized representation 
\begin{equation*}
\gamma = \begin{bmatrix} T_{11} & T_{22} & T_{33} & \sqrt{2}T_{12} & \sqrt{2}T_{13} & \sqrt{2}T_{23} \end{bmatrix}.
\end{equation*}
Thus, we can define a $6\times6$ covariance matrix $\Sigma_{T_i}$ on $\gamma$ and define the anisotropic kernel
\begin{equation}\label{eq:anisotropicDiffMRI}
a_{i,r} = exp\{ - (\gamma_i - \gamma_r) \Sigma_r^\dagger (\gamma_i - \gamma_r)^\intercal\} \cdot \mathbbm{1}_{\|x_i - r_i\| < \epsilon}.
\end{equation}

We examine MMD using a synthetic data set generated from the common brain phantom image.  We treat this as a 2D slice of the 3D image, and simply have all tensor variation in the z-direction constant and uncorrelated with the xy-direction of this slice.  We also assume that the brains have been co-registered.  The process of co-registration is an independent preprocessing issue which can be incorporated into our proposed methodology when working with real world data sets, but is outside the scope of this paper.

In Figure \ref{fig:brainsOriginal}, we show both a ``healthy'' brain and an ``unhealthy'' brain in which a small region has been removed.  The intensity of the image in this case will correspond to the eccentricity of the diffusion tensor at that point; the magnitude of the tensor will decrease from left to right in the same way for both images and the angle shift uniformly from left to right.  The eccentricity, magnitude, and angle all have iid Gaussian noise added to them with $\sigma = 0.05$.

We down-sample the brain by a factor of 5 for the reference points, and consider the mean embeddings $h_i(r)$ for $10$ realizations of a healthy brain (null hypothesis $H_0$), and for $5$ realizations of a healthy brain and $5$ realizations of an unhealthy brain (alternative hypothesis $H_1$).  Figure \ref{fig:brainsOriginal} shows the supervised witness function of regions of difference between the two groups, as well as a permutation test in which we permute group labels while maintaining the individual brain structure.  We also show the leading eigenvectors of the pairwise network generated by measuring the MMD between any two brains.

By using reference points, each brain is represented by a kernel that is $65536\times2601$, and the data adaptive MMD computation can be run on 10 brains in about $4.5$ minutes on a standard laptop.

We also compare the anisotropic kernel to one with an isotropic kernel of constant bandwidth.  We cannot compare to kernel MMD with a square symmetric kernel due to computational limits, so instead we compare to the modified asymmetric kernel as in \eqref{eq:anisotropicDiffMRI}, but without the covariance matrix.  Instead, we replace by a constant bandwidth, which is chosen to be $\sigma^2 = \E_{r\in R}[ trace(\Sigma_r)]$.

%

\section{Discussion and Remarks}\label{sec:discussion}

The paper studies kernel-based MMD statistics
with kernels of the form $\int a(r,x)a(r,y)d\mu_R(r)$,
and more generally, $\int \int L(r,r') a(r,x)a(r,y) d\mu_R(r) d\mu_R(r')$
where $L(r,r')$ is a ``filtering'' kernel over $r$. 
The statistics can be computed with a pre-defined reference set in time $O(n_R n)$,
 where $n$ is the number of samples and $n_R$ is the cardinal number of the reference set.
The power of the test against alternative distributions is analyzed using the spectral decomposition of the kernel with respect to the data distribution,
and the consistency is proved under generic assumptions.
The difference in the testing power of the kernels are determined by their spectral properties.
We apply the proposed methodology to flow cytometry and diffusion MRI data analysis,
where the goal of analysis is formulated as comparing the distribution of multiple samples. 

We close the section by a few remarks about the proposed approach as well as possible extensions:

{\bf The spectral coordinates.} 
The kernel-MMD distance being studied can be viewed as certain $L^2$ distance, 
weighted or unweighted, 
in the space of spectral embedding.
This is reflected in the construction of the kernel $k_{\text{spec}}$,
as well as in the power analysis. 
 Theoretically, the space of spectral embedding is infinitely dimensional, 
 however, in practice only finite dimensional coordinates may contribute to the RKHS MMD statistic 
 - in the sense of providing statistically significant departures. 
 Thus, the leading $K$ coordinates of spectral embedding (depends on the kernel) gives a mapping from $\R^d$ to $\R^K$, 
 where $K$ may be proportional or larger than $d$, 
 and one may consider the two-sample test in the new coordinates. 
The general alternative then becomes a mean-shift alternative in the new coordinates. 
This suggests other possible tests for mean-shift alternatives than the weighted $L^2$ distance test being studied. 

{\bf Weighting of bins.}
While spectral filtering introduces weighting in the generalized Fourier domain, another important variation is to introduce weighting in the ``real domain", namely weighting each bin centered at $r$ by a weight $w(r)$. Such weights may be computed from data, e.g. by certain local $p$-value (see below), 
where the dependence among $r$'s needs to be handled. 
Another interesting question is how to introduce multi-resolution systems in the context of the current paper:
Mapping points into spectral coordinates
 has certain advantage as analyzed in the paper, 
 however, spectral basis is global and may not be sensitive enough to local departure. 
On the other hand, histograms on local bins 
may have large variance and one needs to jointly analyze multiple bins. 
The shortcomings of both approaches may be overcome by considering a multi-scale basis on the reference-set graph.

{\bf Beyond $L^2$ distance.}
RKHS MMD considers the $L^2$ distance by construction, while other metrics have been studied in literature, particularly the Wasserstein metric which has the interpretation of optimal flow given the underlying geometry. Such ``geometric" distances are certainly useful in various  application scenarios. 
Using the reference set, a modification of $T_{L^2}$ will be 
\begin{eqnarray*}
	T_{\text{EMD}}(p,q) 
	&=& \min_{\pi} \int \pi(r,r') d_R(r,r') d\mu_R(r) d\mu_R(r'), \\
	\text{ s.t.}
	& &
	\int \pi(r,r')  d\mu_R(r') = h_p(r),\quad
	\int \pi(r,r')  d\mu_R(r) = h_q(r'), 
\end{eqnarray*}
where $h_p$, $h_q$ are the population histograms, and $d_R(r,r')$ is certain metric on reference set.  
It may also be possible to construct a metric which is equivalent to the Wasserstein metric by measuring the difference at reference points across multiple scales of covariance matrices, as is done with with Haar wavelet \cite{shirdhonkar2008approximate} and with diffusion kernels \cite{Leeb2016}. 
Efficient estimation scheme of the Wasserstein metric needs to be developed as well as the consistency analysis with $n$ samples.

{\bf Local $p$-value.} The current approach gives a global test and computes the $p$-value for the hypothesis of the distribution globally. 
In certain applications, especially differential analysis of flow cytometry data and other single-cell data, 
a more-important problem is to find the local region where the two samples differ corresponding to different biological conditions,
or in other words, to derive a ``local" $p$-value of the test. While the witness function introduced in our current approach can provide indication where $q\neq p$, 
a more systematically study of testing the hypothesis locally and controlling false discovery rate across bins is needed.

\section*{Acknowledgment}
We would like to thank Yuval Kluger for introducing the problem of flow cytrometry data analysis, and  Wade Schultz, Richard Torres, and Jon Astle for facilitating access to the Yale New Haven Hospital data.
We would also like to thank Carlo Pierpaoli, Neda Sadeghi, and Okan Irfanoglu for introducing the problem of diffusion MRI data analysis.  
Cloninger was supported by NSF grant DMS-1402254.

\bibliographystyle{plain}
\bibliography{mmd}

\newpage

%

\appendix

\setcounter{equation}{0}

%

\appendix

\setcounter{equation}{0}

\section{Proofs in Section 3}\label{app:A}

\subsection{Proofs of Propositions \ref{prop:equivalentA2}, \ref{prop:tildek}}

\begin{proof}[Proof of Proposition \ref{prop:equivalentA2}]
	By \eqref{eq:Tpopulation},\eqref{eq:kernel-spec},
	\[
	T(p,q) = \sum_k f_k ( \int  \psi_k(x) (p(x)-q(x))dx )^2
	\]
	and thus (i) $\Leftrightarrow$ (ii). 

	Recall that $a(r,x)$ has the singular value decomposition \eqref{eq:svd-a},
	and thus
	\[
	\int (h_p(r) - h_q(r) )^2 d\mu_R(r) 
	=  \sum_k \sigma_k^2 ( \int  \psi_k(x) (p(x)-q(x))dx )^2
	\]
	with $\sigma_k$ all strictly positive. 
	This means that (iii) is equivalent to $\int\psi_{k}(x)(p(x)-q(x))dx\ne0$
	for some $k$. When $f_{k}$ is all strictly positive, it is equivalent
	to (ii).
\end{proof}

\vskip 0.1in

\begin{proof}[Proof of Proposition \ref{prop:tildek}]
	We first verify the positive semi-definiteness of $\tilde{k}$:
	for any $f$ so that $\int f(x)^{2}p(x)dx<\infty$, by definition,
	\begin{eqnarray*}
	 &  & \int\int\tilde{k}(x,y)f(x)f(y)p(x)p(y)dxdy \\
	 & = & \int\int\int\int k(x,y)p(x')p(y')f(x)f(y)p(x)p(y)dxdydx'dy' \\
	 &    &   -\int\int\int\int k(x,y')p(x')p(y')f(x)f(y)p(x)p(y)dxdydx'dy'\\
	 &  & -\int\int\int\int k(x',y)p(x')p(y')f(x)f(y)p(x)p(y)dxdydx'dy' \\
	 &   & +\int\int\int\int k(x',y')p(x')p(y')f(x)f(y)p(x)p(y)dxdydx'dy'\\
	 & = & \int\int\int\int k(x,y)(p(x')p(y')p(x)p(y)f(x)f(y)-p(x')p(y)p(x)p(y')f(x)f(y')\\
	 &  & -p(x)p(y')p(x')p(y)f(x')f(y)+p(x)p(y)p(x')p(y')f(x')f(y'))dxdydx'dy'\\
	 & = & \int\int k(x,y)\left(\int(p(x')p(x)f(x)-p(x)p(x')f(x'))dx'\right) \\
	 &     & ~~ \left(\int(p(y')p(y)f(y)-p(y)p(y')f(y'))dy'\right)dxdy\\
	 & = & \int\int k(x,y)\tilde{f}(x)\tilde{f}(y)dxdy,
	\end{eqnarray*}
	where $\tilde{f}(x)=\int(p(u)p(x)f(x)-p(x)p(u)f(u))du$. 
	The quantity is nonnegative by that $k$ is PSD. 
	
	To prove (1): 
	Under Assumption \ref{assump:A1}, $k(x,x) \le 1$ and  $| k(x,y)| \le 1$ for any $x,y$. 
	This implies the boundedness of $k_p$ and $k_{pp} $in \eqref{eq:def-tildek}
	and leads to (1).

	(2) follows from the continuity of $k$. 
	
	The square integrability of $\tilde{k}$ then follows from
	boundedness of $k$, which makes the operator Hilbert-Schmidt with
	\begin{equation}\label{eq:sumk-lambdak2-bound}
	\sum_{k}\tilde{\lambda}_{k}^{2} = \int \tilde{k}(x,y)^2 p(x)p(y)dxdy \le 16,
	\end{equation}
	by (1). 
	Meanwhile,  by (1) again,
	\[
	\sum_{k}\tilde{\lambda}_{k}=\int\tilde{k}(x,x)p(x)dx\le 4,
	\]
	which proves that the operator is in trace class. 
	Mercer's Theorem applies to give the spectral expansion and the relevant properties,
	and $\int \tilde{\psi}_k(x) p(x)dx =0$ is by the centering so that $\int \tilde{k}(x,y)p(y)dy = 0$ for all $x$.
	This proves (3). 
	
	Finally, by the uniform convergence of Eqn. (\ref{eq:tildek-spec}), and (1),
	$4 \ge \int\tilde{k}(x,x)q(x)dx=\sum_{k}\tilde{\lambda}_{k} \int\tilde{\psi}_{k}(y)^{2}q(y)dy $,
	which proves (4). 
\end{proof}

\subsection{Proof of Theorems \ref{thm:limit1}, \ref{thm:consist1}}

\begin{lemma}[Replacement lemma]
	\label{lemma:replace}
	Let $\nu_k$ be a sequence of positive number so that $\sum_k \nu_k < \infty$. 
	Let $G_{k,n}$ be an array of random variables, $k=1,2,\cdots$, $n=1,2,\cdots$, s.t. for each $n$,
	\[
	\E G_{kn}=0,\, \forall k,n, 
	\quad
	\E G_{kn} G_{ln}=\Sigma^{(n)}_{kl}, \, \forall k,l,n,
	\]
	where $\Sigma^{(n)}_{kk} < \infty$. Furthermore, as $n \to \infty$,
	
	(i) $\Sigma^{(n)}_{kl} \to \Sigma_{kl}$ elementwise, 
	and for any finite $K$, $(G_{kn})_{1 \le k\le K} \overset{d}{\to} (G_k)_{1\le k \le K}$
	where $ (G_k)_{1\le k \le K} \sim {\cal N}( 0,  \{ \Sigma_{kl}) \}_{1\le k\le K, 1\le l \le K} )$. 
	
	(ii) There exists $B_1 > 0 $ s.t. \[
	\sum_k \nu_k \Sigma^{(n)}_{kk}  < B_1, \, \forall n
	\]
	and the convergence is uniform in $n$.
	
	Meanwhile, let $\alpha_{kn}$, $\beta_{kn}$, $\gamma_{kn}$ be three double arrays satisfying that
	
	(iii) $\alpha_{kn} \ge 0$, $|\gamma_{kn}| \le 1$; as $n\to \infty$, 
	\[
	\alpha_{kn} \to \alpha_k,
	\quad
	\beta_{kn} \to \beta_k,
	\quad
	\gamma_{kn} \to \gamma_k;
	\]
	
	(iv)  There exist $B_2, \, B_3 > 0 $ s.t. \[
	\sum_k \nu_k \alpha_{kn} < B_2,
	\quad
	\sum_k \nu_k \beta_{kn}^2 < B_3,
	\]
	for all $n$, and the convergence is uniform in $n$.
	
	Then, as $n\to \infty$, the random variable 
	\begin{equation}\label{eq:Un}
	U_n = \sum_k \nu_k ( \alpha_{k,n} +  \beta_{k,n} G_{kn} + \gamma_{k,n} G_{kn}^2 ),
	\end{equation}
	converges in distribution to $U$ defined as
	\begin{equation}\label{eq:U}
	U : = \sum_k \nu_k ( \alpha_{k} +  \beta_{k} G_k + \gamma_{k} G_k^2 ),
	\end{equation}
	and $U$ has finite mean and variance.
\end{lemma}

\begin{proof}
	Firstly, we verify that $U$ is well-defined and has finite variance:
	notice that (ii) implies that $\sum_k \nu_k \Sigma_{kk} < B_1$,
	(iii) implies that $\alpha_k \ge 0$, $|\gamma_k| \le 1$,
	and (iv) implies that $\sum_k  \nu_k \alpha_k \le B_2$ and $\sum_k   \nu_k \beta_k^2 \le B_3$. 
	For finite $K$,
	we define the truncated $U_K$ as taking the summation from $1$ to $K$ in Eqn. (\ref{eq:U}).
	Then 
	\[
	|U-U_K |
	\le 
	\sum_{k > K} \nu_k \alpha_k +\sum_{k>K} \nu_k |\beta_k| |G_k| +  \sum_{k>K} \nu_k G_k^2,
	\]
	and thus
	\begin{equation}\label{eq:step1}
	\E |U-U_K| \to 0,
	\quad
	\text{as $K\to \infty$},
	\end{equation}
	due to the summability of 
	$\sum_k \nu_k \alpha_k$, $\sum_k \nu_k \beta_k^2$ and  $\sum_k \nu_k \Sigma_{kk} $ (the 2nd term is bounded by Cauchy-Schwarz).
	Thus $U_K \to U$ with probability one, and $\E |U| < \infty$. Furtherly, \[
	\E (U-\E U)^2 = \E ( \sum_k \nu_k \gamma_k (G_k^2-1) )^2 
	\le \sum_k \nu_k \sum_k \nu_k \E (G_k^2-1)^2 <\infty
	\]
	by that $\E G_k^4$ is finite. This verifies that $U$ has finite mean and variance. Actually, by martingale convergence theorem one can show that $U_K \to U$ a.s.
	
	Secondly, using a similar argument, defining  $U_{n,K}$ to be the truncated $U_n$ in Eqn. (\ref{eq:Un}),
	one can show that 
	\begin{equation}\label{eq:step2}
	\E |U_n-U_{n,K}| \to 0
	\quad
	\text{as $K\to \infty$ uniformly in $n$},
	\end{equation}
	by that $\sum_{k>K} \nu_k \alpha_{kn}$, $\sum_{k>K} \nu_k \beta_{kn}^2$ and $\sum_{k>K} \nu_k \Sigma_{kk}^{(n)}$ all converges to zero uniformly in $n$, which is assumed in condition (ii) and (iv).
	
	Now we come to prove $U_n \overset{d}{\to}   U$. 
	By Levy's Continuity Theorem, 
	it suffice to show the pointwise convergence of the characteristic function, namely
	\[
	\E e^{it U_n} - \E e^{it U} \to 0, \quad \forall t.  
	\]
	Using the truncation of  $k$ up to $K$, we have that 
	\begin{equation}\label{eq:three-terms}
	\begin{split}
	|\E (e^{it U_n} - e^{it U}) |
	& \le 
	\E | e^{it U_n} - e^{it U_{n,K}}  |
	+ \E | e^{it U_{n,K}} - e^{it U_{K}}  | 
	+ \E | e^{it U_K} - e^{it U}  | \\
	& \le
	\E |t| |U_n - U_{n,K}  | + \E | e^{it U_{n,K}} - e^{it U_{K}} | +  \E |t|| U_K -   U |.
	\end{split}
	\end{equation}
	By Eqn. (\ref{eq:step1}) and (\ref{eq:step2}), for any $\varepsilon > 0$, and any $t$, we can choose sufficiently large $K$ s.t. the first and the third term are both less than $\frac{\varepsilon}{3}$ for any $n$.  
	To show that the second term can be made small, we introduce 
	\[
	\bar{U}_{n,K}=  \sum_{k=1}^K \nu_k ( \alpha_{k} +  \beta_{k} G_{kn} + \gamma_{k} G_{kn}^2 ),
	\]
	and by that $(G_{kn})_{1 \le k \le K}  \overset{d}{\to} (G_{k})_{1 \le k \le K} $ (condition (i)), 
	$\bar{U}_{n,K} \overset{d}{\to} U_K$.
	Meanwhile,
	\[
	| U_{n,K} - \bar{U}_{n,K} |
	\le  \sum_{k=1}^K \nu_k \{
	| \alpha_{kn}- \alpha_{k} |  
	+  |\beta_{kn} - \beta_k| |G_{kn}| 
	+ |\gamma_{kn}-\gamma_{k}| G_{kn}^2 \}.
	\]
	Since $K$ is finite, and $\E G_{kn}^2 = \Sigma_{kk}^{(n)}$ is uniformly bounded as $n$ increases for each $k$, 
	we have that 
	$\E | U_{n,K} - \bar{U}_{n,K} | \to 0$ as $n\to \infty$,
	by the convergence of $\alpha_{kn}$, $\beta_{kn}$ and $\gamma_{kn}$.
	Thus the second term can be bounded by 
	\begin{eqnarray}
	\E | e^{it U_{n,K}} - e^{it U_{K}} | 
	& \le &
	\E | e^{it U_{n,K}} - e^{it \bar{U}_{n,K}} | 
	+ \E | e^{it \bar{U}_{n,K}} - e^{it U_{K}} | \nonumber \\
	&\le &
	\E |t| |U_{n,K} -  \bar{U}_{n,K} | +  \E | e^{it \bar{U}_{n,K}} - e^{it U_{K}} | 
	:= \text{(I)} + \text{(II)},
	\end{eqnarray}
	where (I) can be made smaller than $\frac{\varepsilon}{6}$ for large $n$ ($t$ is fixed
	and $\E | U_{n,K} - \bar{U}_{n,K} | \to 0$), 
	and (II) can be made smaller than $\frac{\varepsilon}{6}$ as a result of $\bar{U}_{n,K} \overset{d}{\to} U_K$ which implies convergence of characteristic function.
	Putting together, 
	the l.h.s. of Eqn. (\ref{eq:three-terms}) can be made smaller than $\frac{\varepsilon}{3}+\frac{\varepsilon}{3}+\frac{\varepsilon}{3} = \varepsilon$ for large $n$, which proves the claim.
\end{proof}

\vskip 0.1in

\begin{proof}[Proof of Theorem \ref{thm:limit1}]
	We introduce
	\[
	\hat{h}_k : = \frac{1}{\sqrt{n_1}}\sum_{i=1}^{n_1} \tilde{\psi}_k(x_i), 
	\quad
	\hat{g}_k : = \frac{1}{\sqrt{n_2}}\sum_{j=1}^{n_2} (\tilde{\psi}_k(y_j) + \tilde{v}_k), 
	\]
	where, by definition, 
	\[\E_{x\sim p} \tilde{\psi}_k(x) = 0,
	\quad 
	\E_{y\sim q} (\tilde{\psi}_k (y) + \tilde{v}_k) = 0,
	\]
	and 
	\[
	 \E_{x\sim p} \tilde{\psi}_k(x) \tilde{\psi}_l(x) = \delta_{kl},
	 \quad 
	 \E_{y\sim q} (\tilde{\psi}_k(y)+\tilde{v}_k) (\tilde{\psi}_l(y) + \tilde{v}_l):=S^{(n)}_{kl}.
	 \]
	Using the above notations, we rewrite Eqn. (\ref{eq:Tn1}) as 
	\begin{equation}\label{eq:nTn1}
	nT_n = \sum_k \tilde{\lambda}_k 
	\left(
	-\tau_n \sqrt{n} c_k
	+  \frac{1}{\sqrt{\rho_{1,n}}} \hat{h}_k - \frac{1}{\sqrt{\rho_{2,n}}} \hat{g}_k
	\right)^2,
	\end{equation}
	where $\rho_{1,n}= \frac{n_1}{n}$ and $\rho_{2,n}= \frac{n_2}{n}$. 
	The random variables $\{\hat{h}_k\}_k$ are independent from $\{\hat{g}_k\}_k$,
	and both are asymptotically normal for finite many $k$'s.
	We will use the replacement lemma \ref{lemma:replace} 
	to substitute $\hat{h}_k$ and $\hat{g}_k$ by their normal counterparts, 
	and discuss scenarios (1)-(3) respectively.

	To apply Lemma \ref{lemma:replace}, 
	we set $\nu_k =\tilde{\lambda}_k$, and the summability follows (3) of Proposition \ref{prop:tildek};
	we set 
\begin{equation}\label{eq:set-Gkn}
G_{kn} = \frac{1}{\sqrt{\rho_{1,n}}} \hat{h}_k - \frac{1}{\sqrt{\rho_{2,n}}} \hat{g}_k,
\end{equation}
and then \eqref{eq:nTn1} becomes
\begin{equation}\label{eq:nTn2}
nT_n 
= \sum_k \tilde{\lambda}_k 
( 
- \tau \sqrt{n} c_k
+ G_{kn}
)^2
= \sum_k \tilde{\lambda}_k 
\left\{
(\tau \sqrt{n} c_k)^2
- 2 \tau \sqrt{n} c_k G_{kn}
+ G_{kn}^2
\right\}.
\end{equation}
We have that
\begin{eqnarray}
\E G_{kn} &= &0  
\label{eq:EGkn}\\
\Sigma_{kl}^{(n)} 
&=& \E G_{kn}G_{ln}  
		= \E \frac{1}{\rho_{1,n}} \hat{h}_k \hat{h}_l + \E \frac{1}{\rho_{2,n}} \hat{g}_k \hat{g}_l 
		= \frac{1}{\rho_{1,n}} \delta_{kl} + \frac{1}{\rho_{2,n}} S^{(n)}_{kl},
		\label{eq:EGknGln}
\end{eqnarray}	
	where, recalling that $c_k = \int \tilde{\psi}_k(y) g(y) dy$, $\tilde{v}_k = -\tau c_k$, and $q_1 := p+g$,
	\begin{eqnarray}
	S^{(n)}_{kl} 
	&=& \int \tilde{\psi}_k(y)  \tilde{\psi}_l(y) (p+\tau g)(y) dy -\tilde{v}_k \tilde{v}_l  \nonumber \\
	&=& (1-\tau) \delta_{kl} + \tau \int  \tilde{\psi}_k(y)  \tilde{\psi}_l(y) q_1(y) dy  -\tau^2 c_k c_l. 
	\label{eq:Sn_kl}
	\end{eqnarray}
	By that $\rho_{1,n} \to \rho_1 > 0 $ and $  \rho_{2,n} \to \rho_2 > 0$, both $\frac{1}{\rho_{1,n}}$ and $\frac{1}{\rho_{2,n}}$ are uniformly bounded, so $\Sigma_{kl}^{(n)}$ are bounded for each $(k,l)$. We now verify condition (i) and (ii) in  Lemma \ref{lemma:replace}: 
	
	Condition (i): 
	In case (1) and (2), $\tau \to 0$, 
	thus $S^{(n)}_{kl}  \to \delta_{kl}$ 
	and then 
	$\Sigma_{kl}^{(n)} \to \Sigma_{kl} = ( \frac{1}{\rho_1} + \frac{1}{\rho_2})\delta_{kl}$. 
	In case (3) $\tau \equiv 1$,
	thus $S^{(n)}_{kl}  \equiv S_{kl}$ and then $\Sigma_{kl}^{(n)} \to \Sigma_{kl} =  \frac{1}{\rho_1} \delta_{kl} + \frac{1}{\rho_2} S_{kl}$. 
	As for the limiting distribution of $ (G_{kn})_{1 \le k \le K}$ for any finite $K$, 
	we know that $(\hat{h}_{k})_{1 \le k \le K} \overset{d}{\to} {\cal N}(0, I_K)$, 
	and 
	$(\hat{g}_{k})_{1 \le k \le K} \overset{d}{\to} {\cal N}(0, \{S_{kl}\}_{1 \le k,l \le K})$
	by Lindeberg-Levy CLT (Theorem 1.9.1 B in \cite{serfling1981approximation}, extended to the case where the covariance matrix converges to a non-degenerate limit by Slutsky Theorem).
	By definition of $G_{kn}$ and that $(\hat{h}_{k})_{1 \le k \le K}$ and $(\hat{g}_{k})_{1 \le k \le K}$ are independent, 
	$ (G_{kn})_{1 \le k \le K} \overset{d}{\to} {\cal N}(0, \{\Sigma_{kl}\}_{1 \le k,l \le K}) $ where $\rho_{1,n}$ ($\rho_{2,n}$) 
	is replaced by $\rho_1$ ($\rho_2$) by Slutsky Theorem. The argument applies to all the three cases.

	Condition (ii):
	$\Sigma_{kk}^{(n)} = \frac{1}{\rho_{1,n}} +  \frac{1}{\rho_{2,n}} S^{(n)}_{kk} \le c_1 + c_2 S^{(n)}_{kk}$
	for some absolute positive constant $c_1$ and $c_2$. Meanwhile, by Eqn. (\ref{eq:Sn_kl}),
	$S^{(n)}_{kk} = (1-\tau) + \tau \int  \tilde{\psi}_k(y)^2 q_1(y) dy - \tau^2 c_k^2 \le 1+  \int  \tilde{\psi}_k(y)^2 q_1(y) dy $, 
	thus
	\[
	\sum_k \tilde{\lambda}_k \Sigma_{kk}^{(n)}
	\le
	\sum_k \tilde{\lambda}_k (c_1 + c_2 (1 + \int  \tilde{\psi}_k(y)^2 q_1(y) dy )) < \infty,
	\] 
	thanks to that $\sum_k \tilde{\lambda}_k < \infty$ and that
	$ \sum_k \tilde{\lambda}_k \int  \tilde{\psi}_k(y)^2 q_1(y) dy < \infty $ ((4) of Proposition \ref{prop:tildek}),
	and the convergence is uniform in $n$.
	
	We now consider the three scenarios respectively:
	
	(1) Let $U_n=nT_n$, by Eqn. (\ref{eq:nTn2}) we have 
	\[
	\alpha_{k,n}= (-\tau \sqrt{n} c_k)^2 \to a^2 c_k^2,
	\quad
	\beta_{k,n}= -2 \tau \sqrt{n} c_k \to -2ac_k,
	\quad
	\gamma_{k,n} = 1,
	\]
	thus (iii) holds. Condition (iv) can be verified by that $\sum_k \tilde{\lambda}_k c_k^2 < \infty$ (upper bounded by $ \le \sum_k \tilde{\lambda}_k \int  \tilde{\psi}_k(y)^2 q_1(y) dy $). 
	As analyzed above, 
	$\Sigma_{kl}^{(n)} \to \Sigma_{kl} = ( \frac{1}{\rho_1} + \frac{1}{\rho_2})\delta_{kl}$, 
	and condition (i) and (ii) hold,
	thus Lemma \ref{lemma:replace} applies to give that 
	\[
	U_n \overset{d}{\to} 
	U =\sum_k \tilde{\lambda}_k (  a^2 c_k^2 - 2ac_k G_k + G_k^2 )
	= \sum_k \tilde{\lambda}_k ( - ac_k + G_k)^2,
	\quad
	G_k \sim {\cal N}\left(0,\frac{1}{\rho_1} + \frac{1}{\rho_2}\right) \text{ i.i.d.}
	\]
	as claimed in the theorem.

	(2) Let $U_n$ be the l.h.s. of the statement, then
	\[
	\alpha_{k,n}= 0,
	\quad
	\beta_{k,n}= -2 c_k,
	\quad
	\gamma_{k,n} = n^{-\delta} \to 0,
	\]
	and condition (iii) and (iv) hold. 
	Same as in (1), 
	$\Sigma_{kl} = ( \frac{1}{\rho_1} + \frac{1}{\rho_2})\delta_{kl}$
	and (i) and (ii) hold,
	thus Lemma \ref{lemma:replace} gives that 
	\[
	U_n \overset{d}{\to} U = \sum_k \tilde{\lambda}_k (-2c_k) G_k, 
	\quad
	G_k \sim {\cal N}\left(0,\frac{1}{\rho_1} + \frac{1}{\rho_2}\right) \text{ i.i.d.}
	\]
	By the summability of $\tilde{\lambda}_k$, $U$ is in same distribution as ${\cal N}(0,\sigma_{(2)}^2)$ as defined in the theorem.

	(3) Similar to (2), let $U_n$ be the l.h.s. of the statement, then $\alpha_{k,n}$, $\beta_{k,n}$, $\gamma_{k,n}$ are same as in (2) where $\delta=\frac{1}{2}$, so they have the same limit, and (iii) and (iv) hold. As analyzed above,
	$\Sigma_{kl}^{(n)} \to \Sigma_{kl} =  \frac{1}{\rho_1} \delta_{kl} + \frac{1}{\rho_2} S_{kl}$,
	and (i) and (ii) hold. Thus Lemma \ref{lemma:replace} gives that 
	$U_n \overset{d}{\to} U = \sum_k \tilde{\lambda}_k (-2c_k) G_k$
	where $G_k$ has covariance $\Sigma_{kl}$. 
	By the summability of $\tilde{\lambda}_k$, $U$ is in same distribution as ${\cal N}(0,\sigma_{(3)}^2)$, where
	$\sigma_{(3)}^2 = 4 \sum_{kl} \tilde{\lambda}_k\tilde{\lambda}_l c_k c_l \Sigma_{kl}$ which equals the formula claimed in the theorem. 
\end{proof}

\vskip 0.1in

\begin{proof}[Proof of Theorem \ref{thm:consist1}]
We only prove (2), as (1) directly follows from Theorem \ref{thm:limit1} (1) 
and the form of the limiting density of $nT_n$ in this case. 

We first consider the ${\cal H}_1$ case, i.e. $\tau > 0$:
Due to that $g$ satisfies Assumption \ref{assump:A2}
and the equivalent forms of $T$ as in \eqref{eq:def-T-tildek}, \eqref{eq:T-lambdak-vk},
we have that
\[
\sum_k \tilde{\lambda}_k c_k^2 := t_{(1)} > 0.
\]
Notations as in the proof of Theorem \ref{thm:limit1}, by \eqref{eq:nTn2}, 
\[
\frac{n T_n}{ (\tau \sqrt{n})^2 }
=
\sum_k \tilde{\lambda}_k
\left\{
c_k^2
- \frac{2 c_k}{ \tau \sqrt{n}} G_{kn}
+ \frac{G_{kn}^2}{ (\tau \sqrt{n} )^2}
\right\}.
\]
We set $U_n$ to be the l.h.s., and apply Lemma \ref{lemma:replace}:
$\nu_k = \tilde{\lambda}_k$ and are summable as before. 
Conditions (i) (ii) are satisfied by $G_{kn}$ (defined in \eqref{eq:set-Gkn}),
as has been verified in the proof of Theorem \ref{thm:limit1}.
Since
\[
\alpha_{kn} = c_k^2,
\quad
\beta_{kn} = -\frac{2 c_k}{\tau \sqrt{n}},
\quad
\gamma_{kn} = (\tau \sqrt{n})^{-2},
\]
and  $\tau \sqrt{n} \to + \infty$, 
Condition (iii) holds 
with the limits as
\[
\alpha_k = c_k^2,
\quad 
\beta_k = 0,
\quad
\gamma_k = 0.
\]
Condition (iv) is also satisfied due the summability of $\sum_k \tilde{\lambda}_k c_k^2$, 
same as in the proof of Theorem \ref{thm:limit1}.
Thus Lemma \ref{lemma:replace} gives that 
\begin{equation}
\label{eq:d-converge-U1}
\frac{n T_n}{ (\tau \sqrt{n})^2 } 
\overset{d}{\to}
U_{(1)}
:=
\sum_k \tilde{\lambda}_k c_k^2 = t_{(1)},
\end{equation}
which is a single-point distribution at the positive constant $t_{(1)}$.

We then consider the  ${\cal H}_0$ case, i.e. $\tau  = 0$:
By Theorem \ref{thm:limit1} (1),
$nT_n \overset{d}{\to} U_{(0)}$ 
which is a continuous nonnegative random variable with finite mean and variance.
Thus for any chosen level $\alpha$, there exists $t_{(0)} < \infty$  s.t.
$\Pr [ U_{(0)} >  t_{(0)}] < \alpha$,
and then when $n$ is large enough,
\[
\Pr [ nT_n >  t_{(0)}
 | {\cal H}_0 ] < \alpha.
\]
This means that $ \frac{t_{(0)}}{n}$ is a valid threshold in \eqref{eq:pi-n-alpha},
and as a result, (shortening ``${\cal H}_1 \text{ with $q$}$'' as ${\cal H}_1$)
\[
\pi_n(q) \ge  \Pr [ T_n  >  \frac{t_{(0)}}{n} | {\cal H}_1 ].
\]

Putting together with \eqref{eq:d-converge-U1}, 
we then have that  for sufficiently large $n$,
\[
1- \pi_n(q) 
\le  
\Pr [ T_n  \le  \frac{t_{(0)}}{n} | {\cal H}_1 ]
\to
\Pr [ U_{(1)} \le \frac {t_{(0)} }{ (\tau \sqrt{n})^2} ]
\]
which converges to 0 since $t_{(0)}$ is an absolute constant and thus 
$\frac {t_{(0)} }{ (\tau \sqrt{n})^2} \to 0$,
and meanwhile $U_{(1)} \equiv  t_{(1)} > 0$.
This proves that $\pi_n(q) \to 1$. 
\end{proof}

\subsection{Proof of Theorem \ref{thm:consist2} }

\begin{proof}[Proof of Theorem \ref{thm:consist2}]
We use Chebyshev to control the deviation of the random variable $X:= nT_n$ from its mean,
under ${\cal H}_0$ and ${\cal H}_1$ respectively.

Under ${\cal H}_0$, $\tau=0$, by \eqref{eq:nTn2},
\[
X 
= \sum_k \tilde{\lambda}_k  G_{kn}^2.
\]
By \eqref{eq:EGknGln}, $\E G_{kn}^2  = \frac{1}{\rho_{1,n}}     + \frac{1}{\rho_{2,n}} = \frac{1}{ \rho_{1,n} \rho_{2,n}}$ ,
\begin{equation}\label{eq:EX|H0}
\E X 
 = \frac{1}{ \rho_{1,n} \rho_{2,n}} \sum_k \tilde{\lambda}_k  
 = C_4.
\end{equation}
We will prove later that
\begin{equation}\label{eq:Var(X)|H0-to-prove}
\text{Var}(X) 
 \le C_3 + 0.1,
\end{equation}
and then by Chebyshev we have that for any $t_1 > 0$,
\[
\Pr [   X > C_4 + t_1 | {\cal H}_0 ]
\le 
\frac{C_3 + 0.1}{  t_1^2}.
\]
Setting the r.h.s to be $\alpha$ gives $t_1 = \sqrt{ \frac{C_3 + 0.1}{\alpha} }$,
and thus
\begin{equation}\label{eq:Pr-bound-H0}
\Pr [   X > C_4 + \sqrt{ \frac{C_3 + 0.1}{\alpha} } | {\cal H}_0 ] \le \alpha.
\end{equation}

Under ${\cal H}_1$, by \eqref{eq:nTn2} and the definition that $T_1 := \sum_k \tilde{\lambda}_k c_k^2 $, 
\begin{align}
X 
&= \sum_k \tilde{\lambda}_k 
( - \tau \sqrt{n} c_k
+ G_{kn})^2
=
(\tau^2 n) T_1
+X_1,\\
X_1
& : =   \sum_k \tilde{\lambda}_k 
\left(
-2 \tau \sqrt{n} c_k G_{kn} + G_{kn}^2
\right).
\label{eq:def-X1}
\end{align}
Since $\E G_{kn} = 0$ (c.f. \eqref{eq:EGkn}), 
\[
\E X_1  =  \sum_k \tilde{\lambda}_k G_{kn}^2 \ge 0,
\]
we have that 
\begin{equation}\label{eq:EX|H1}
\E X 
\ge (\tau^2 n) T_1.
\end{equation}
We will prove later that
\begin{equation}\label{eq:Var(X)|H1-to-prove}
\text{Var}(X) 
 \le (\tau^2 n) C_1 + \tau C_2 + C_3 + 0.1,
 \end{equation}
 and then, using Chebyshev again, for any $t_2 > 0$,
 \[
 \Pr [   X \le   (\tau^2 n) T_1 - t_2 | {\cal H}_1]
\le
 \Pr [   X \le  \E X  - t_2 | {\cal H}_1]
 \le 
 \frac{ (\tau^2 n) C_1 + \tau C_2 + C_3 + 0.1 }{t_2^2}.
 \]
Combined with \eqref{eq:Pr-bound-H0} which shows that 
$ t : = C_4 + \sqrt{ \frac{C_3 + 0.1}{\alpha} }$ is a valid threshold to achieve level $\alpha$,
by setting $t_2 =  (\tau^2 n) T_1 - t $ (which is strictly positive under \eqref{eq:cond2-consist2}), 
this gives the bound \eqref{eq:bound-consist2}.

It remains to prove \eqref{eq:Var(X)|H0-to-prove} and \eqref{eq:Var(X)|H1-to-prove} to finish the proof.

\vskip 0.1in
\underline{Proof of \eqref{eq:Var(X)|H0-to-prove}}:
We will prove that
\begin{equation}\label{eq:Var(X)|H0-2}
\text{Var}(X) 
 \le 
 \frac{2}{ (\rho_{1,n} \rho_{2,n})^2}
  	\sum_k \tilde{\lambda}_k^2 
  + \frac{16}{n} \left(  \frac{1}{ \rho_{1,n}^3 } +   \frac{1}{ \rho_{2,n}^3 }  \right)
\end{equation}
where the first term $\le C_3$ since $\sum_k \tilde{\lambda}_k^2 \le 16$ (c.f. \eqref{eq:sumk-lambdak2-bound}),
and the second term $ < 0.1$ under the condition that $ n > \frac{16}{0.1}( \frac{1}{ \rho_{1,n}^3 } + \frac{4}{ \rho_{2,n}^3 }   )$. 
This gives \eqref{eq:Var(X)|H0-to-prove}.

Recall that $X= \sum_k \tilde{\lambda}_k  G_{kn}^2$, and $\E X = C_4$.
By the definition of $G_{kn}$ \eqref{eq:set-Gkn}, 
\begin{align}
\E X^2 
& = \E \sum_k \sum_l \tilde{\lambda}_k \tilde{\lambda}_l G_{kn}^2G_{ln}^2 
\nonumber \\
& = \E \sum_k \sum_l \tilde{\lambda}_k \tilde{\lambda}_l
    \left( \frac{1}{\rho_{1,n}} \hat{h}_k^2 +  \frac{1}{\rho_{2,n}} \hat{g}_k^2  -  \frac{2}{\sqrt{\rho_{1,n}\rho_{2,n}}} \hat{h}_k\hat{g}_k   \right) 
    \left( \frac{1}{\rho_{1,n}} \hat{h}_l^2 +  \frac{1}{\rho_{2,n}} \hat{g}_l^2  -  \frac{2}{\sqrt{\rho_{1,n}\rho_{2,n}}} \hat{h}_l\hat{g}_l  \right) 
    \nonumber \\
& =   \sum_k \sum_l \tilde{\lambda}_k \tilde{\lambda}_l
	\{
	\frac{1}{\rho_{1,n}^2} \E \hat{h}_k^2 \hat{h}_l^2 
	+   \frac{2}{\rho_{1,n} \rho_{2,n} } 
	+  \frac{1}{\rho_{2,n}^2} \E \hat{g}_k^2 \hat{g}_l^2
	+ \frac{4}{\rho_{1,n} \rho_{2,n} }  \delta_{kl}
	\} 
	 \nonumber \\ 
& = \frac{2}{\rho_{1,n} \rho_{2,n} }  (\sum_k \tilde{\lambda}_k )^2
 	+ \frac{4}{\rho_{1,n} \rho_{2,n} } \sum_k \tilde{\lambda}_k^2
	+  \frac{1}{\rho_{1,n}^2} \E \sum_k \sum_l \tilde{\lambda}_k \tilde{\lambda}_l \hat{h}_k^2 \hat{h}_l^2 
	+  \frac{1}{\rho_{2,n}^2} \E \sum_k \sum_l \tilde{\lambda}_k \tilde{\lambda}_l \hat{g}_k^2 \hat{g}_l^2  .	
	\label{eq:EX2-H0-1}
\end{align}
Note that
\[
\sum_k \tilde{\lambda}_k \hat{h}_k^2
=
\sum_k \tilde{\lambda}_k \left(  \frac{1}{\sqrt{n_1}} \sum_{i} \tilde{\psi}_k(x_i) \right)^2 
= \frac{1}{n_1}  \sum_{i, \, i'} \tilde{k}( x_i, x_{i'}),
\]
and then
\begin{align}
\text{3rd term in \eqref{eq:EX2-H0-1}}
& = \frac{1}{\rho_{1,n}^2}  
	\E (\sum_k \tilde{\lambda}_k \hat{h}_k^2)^2 
	\nonumber \\
& = \frac{1}{\rho_{1,n}^2}
	\E \left( \frac{1}{n_1}  \sum_{i, \, i'} \tilde{k}( x_i, x_{i'}) \right)^2 
	\nonumber  \\
& = \frac{1}{\rho_{1,n}^2}
	\frac{1}{n_1^2}
	\E \sum_{ i,\, i', \, j, \, j'} 	\tilde{k}( x_i, x_{i'})   \tilde{k}( x_j, x_{j'}).	
	\label{eq:EX2-H0-term3}
\end{align}
Recall that $\E_{x' \sim p} \tilde{k}( x, x') =0$ for any $x$, 
thus $\E \tilde{k}( x_i, x_{i'})   \tilde{k}( x_j, x_{j'})$ does not vanish only when the indices $\{ i,\, i', \, j, \, j' \}$ all equal or fall into two pairs.
Then
\begin{align}
\text{\eqref{eq:EX2-H0-term3}}
& = \frac{1}{\rho_{1,n}^2}
	\frac{1}{n_1^2}
	\{
	n_1 \E_{x \sim p} \tilde{k}( x, x)^2  
	+ n_1 (n_1 -1) ( \E_{x \sim p} \tilde{k}( x, x) )^2
	+ 2 n_1 (n_1 -1)  \E_{x, \,x' \sim p} \tilde{k}( x, x')^2
	\} \nonumber \\
& \le
	\frac{1}{\rho_{1,n}^2}
	\{
	( \E_{x \sim p} \tilde{k}( x, x) )^2 
	+ 2  \E_{x, \,x' \sim p} \tilde{k}( x, x')^2 
	+ \frac{1}{n_1}  \E_{x \sim p} \tilde{k}( x, x)^2  
	\}
	\label{eq:EX2-H0-term3-moment},
\end{align}
Recall that 
\begin{align}
 \E_{x \sim p} \tilde{k}( x, x)  
 & = \sum_k  \tilde{\lambda}_k\\
 \E_{x, \,x' \sim p} \tilde{k}( x, x')^2 
 & = \sum_k  \tilde{\lambda}_k^2,
 \end{align}
 and $  \tilde{k}( x, x)   \le 4$ (c.f. Proposition \ref{prop:tildek} (1)), 
the above line continues to give that
\begin{equation}\label{eq:EX2-H0-term3-bound}
\text{3rd term in \eqref{eq:EX2-H0-1}}
\le  
\frac{1}{\rho_{1,n}^2}
	\{
	( \sum_k  \tilde{\lambda}_k )^2 
	+ 2 \sum_k  \tilde{\lambda}_k^2
	+ \frac{16}{n_1}  
	\}.
\end{equation}
Similarly, the 4th term can be bounded by
\begin{equation}\label{eq:EX2-H0-term4-bound}
\text{4th term in \eqref{eq:EX2-H0-1}}
\le
\frac{1}{\rho_{2,n}^2}
	\{
	( \sum_k  \tilde{\lambda}_k )^2 
	+ 2 \sum_k  \tilde{\lambda}_k^2
	+ \frac{16}{n_2}  
	\}.
\end{equation}
Back to \eqref{eq:EX2-H0-1}, we have that 
\begin{align*}
\text{Var}(X) 
&= \E X^2  - (\E X)^2 \\
&\le 
	\frac{2}{\rho_{1,n} \rho_{2,n} }  (\sum_k \tilde{\lambda}_k )^2
 	+ \frac{4}{\rho_{1,n} \rho_{2,n} } \sum_k \tilde{\lambda}_k^2
	+ \frac{1}{\rho_{1,n}^2}
	\{
	( \sum_k  \tilde{\lambda}_k )^2 +  2 \sum_k  \tilde{\lambda}_k^2
	+ \frac{16}{n_1}  
	\} \\
&~~~~~~
	+ 
	\frac{1}{\rho_{2,n}^2}
	\{
	( \sum_k  \tilde{\lambda}_k )^2 +  2 \sum_k  \tilde{\lambda}_k^2
	+ \frac{16}{n_2}  
	\}
	- \left( \frac{1}{\rho_{1,n} \rho_{2,n}} \sum_k \tilde{\lambda}_k\right)^2 \\
& = 	( \sum_k  \tilde{\lambda}_k^2 )
	 \frac{2}{(\rho_{1,n}\rho_{2,n})^2}
	+ \frac{16}{n} \left(  \frac{1}{\rho_{1,n}^3} + \frac{1}{\rho_{2,n}^3} \right)
\end{align*}
which is \eqref{eq:Var(X)|H0-2}.

\vskip 0.1in
\underline{Proof of \eqref{eq:Var(X)|H1-to-prove}}:
Recall that $X =(\tau^2 n) T_1+X_1$,  $T_1$ being constant, thus $\text{Var}(X) = \text{Var}(X_1)$.
By \eqref{eq:def-X1},
\begin{align*}
\text{Var}(X_1) 
& = (1) - (2) + (3) - (4) \\
(1)
& := ( \tau \sqrt{n})^2  4
	\sum_{k,l} \tilde{\lambda}_k   \tilde{\lambda}_l c_k c_l \Sigma_{kl}^{(n)} \\
(2)
& := \E 4\sqrt{n} 
	\sum_{k,l} \tilde{\lambda}_k   \tilde{\lambda}_l
	(\tau c_k)  G_{kn} G_{ln}^2 \\
(3)
& := \E \sum_{k,l} \tilde{\lambda}_k   \tilde{\lambda}_l
	G_{kn}^2 G_{ln}^2 \\
(4)
& :=	(\E X_1)^2 
	= (\E \sum_k \tilde{\lambda}_k G_{kn}^2 )^2
\end{align*}

We will prove that
\begin{align}
(1) 	&  \le (\tau \sqrt{n} )^2 C_1, 
	\label{eq:varX-H1-1-bound} \\
|(2) | 	& \le \tau C_2, 
	\label{eq:varX-H1-2-bound}\\
(3)-(4) & \le C_3 + 0.1
	\label{eq:varX-H1-3-bound}
\end{align}
which, putting together, gives that
\[
\text{Var}(X_1) 
\le
  (\tau \sqrt{n} )^2 C_1 + \tau C_2 + C_3 + 0.1
\]
which proves \eqref{eq:Var(X)|H1-to-prove}.

\vskip 0.1in
Proof of \eqref{eq:varX-H1-1-bound}:
Recall that (c.f. \eqref{eq:EGknGln}), and $q = p + \tau g$,
\[
\Sigma_{kl}^{(n)} 
	= \frac{1}{\rho_{1,n}} \delta_{kl} 
	+ \frac{1}{\rho_{2,n}} S_{kl}^{(n)},
\quad
S_{kl}^{(n)} = \E \hat{g}_k \hat{g}_l = \E_{y \sim q} ( \tilde{\psi}_k(y) + \tilde{v}_k) (\tilde{\psi}_l(y) + \tilde{v}_l ).
\]
We define 
\begin{equation}\label{eq:def-k1}
k_1( x ) : = \int \tilde{k}( x, z) (p+g)(z) dz 
	= \sum_k \tilde{\lambda}_k  c_k \tilde{\psi}_k(x),
\end{equation}
and then 
\[
\E_{ y\sim q} k_1 (y) = - \sum_k \tilde{\lambda}_k  c_k \tilde{v}_k,
\]
thus
\begin{align}
\sum_{k,l} \tilde{\lambda}_k   \tilde{\lambda}_l c_k c_l S_{kl}^{(n)} 
& = \E_{y \sim q} 
	\sum_{k,l} \tilde{\lambda}_k   \tilde{\lambda}_l c_k c_l
	( \tilde{\psi}_k(y) + \tilde{v}_k) (\tilde{\psi}_l(y) + \tilde{v}_l )
	\nonumber \\
&= \E_{y \sim q} ( \sum_{k} \tilde{\lambda}_k c_k ( \tilde{\psi}_k(y) + \tilde{v}_k) )^2
	\nonumber \\
& = \E_{y \sim q} ( k_1(y) - \E_{ y\sim q} k_1 (y) )^2 
 \le \E_{y \sim q} k_1(y)^2
 \label{eq:sumkl-ckcl-Slk-bound}
\end{align}
Meanwhile, by Proposition \ref{prop:tildek} (1), $|\tilde{k}(x,y)| \le 4$ uniformly, 
this gives that  $|k_1 (x) | \le 4$ for any $x$. 
Thus,
\[
\sum_{k,l} 
	\tilde{\lambda}_k   \tilde{\lambda}_l c_k c_l S_{kl}^{(n)}
\le 	16.
\]
As a result,
\[
\sum_{k,l} \tilde{\lambda}_k   \tilde{\lambda}_l c_k c_l \Sigma_{kl}^{(n)}
= 	  \frac{1}{\rho_{1,n}} \sum_k \tilde{\lambda}_k^2 c_k^2 
        + \frac{1}{\rho_{2,n}}  \sum_{k,l} \tilde{\lambda}_k   \tilde{\lambda}_l c_k c_l S_{kl}^{(n)}
\le 	\frac{1}{\rho_{1,n}} \sum_k \tilde{\lambda}_k^2 c_k^2  
	+  \frac{16}{\rho_{2,n}}
\]
which proves \eqref{eq:varX-H1-1-bound}.

\vskip 0.1in
Proof of \eqref{eq:varX-H1-2-bound}:
By \eqref{eq:set-Gkn}, the independence of $\hat{h}_k$ from $\hat{g}_k$, and that $\E \hat{h}_k = 0 $, $\E \hat{g}_k = 0 $ for all $k$, 
\[
\E G_{kn} G_{ln}^2 
= \frac{1}{\rho_{1,n}^{3/2}} \E \hat{h}_k \hat{h}_l^2 
	- \frac{1}{\rho_{2,n}^{3/2}} \E \hat{g}_k \hat{g}_l^2.
\]
We then have that 
\begin{align*}
(2) 
&= (4 \sqrt{n})
	\E \sum_{k,l} \tilde{\lambda}_k   \tilde{\lambda}_l 
		\tau c_k \left(  \frac{1}{\rho_{1,n}^{3/2}} \E \hat{h}_k \hat{h}_l^2 
		- \frac{1}{\rho_{2,n}^{3/2}} \E \hat{g}_k \hat{g}_l^2  \right) 
		\nonumber \\
& = (4 \sqrt{n})
    \E  \left\{
     \frac{1}{\rho_{1,n}^{3/2}}  \textcircled{\small{1}} \textcircled{\small{2}}
    - \frac{1}{\rho_{2,n}^{3/2}}  \textcircled{\small{3}} \textcircled{\small{4}}
    \right\}
    \label{eq:(2)-fromula} \\
\textcircled{\small{1}} 
& = \sum_k   \tau c_k \tilde{\lambda}_k \hat{h}_k 
	= \frac{\tau}{ \sqrt{n_1}} \sum_i k_1(x_i) 
	\nonumber \\
\textcircled{\small{2}} 
& = \sum_k  \tilde{\lambda}_k \hat{h}_k^2 
	=  \frac{1}{ n_1} \sum_{i, \, i'} \tilde{k}(x_i, x_{i'})
	\nonumber \\
\textcircled{\small{3}} 
& = \sum_k   \tau c_k \tilde{\lambda}_k \hat{g}_k 
	= \frac{\tau}{ \sqrt{n_2}} \sum_j ( k_1(y_j) -  \E_{ y\sim q} k_1 (y)  )
	\nonumber \\
\textcircled{\small{4}} 
& = \sum_k  \tilde{\lambda}_k \hat{g}_k^2
	= \frac{1}{ n_2} \sum_{j, \, j'} \tilde{\tilde{k}}(y_j, y_{j'}),
\end{align*}
where $\tilde{\tilde{k}}$ is defined as
\begin{align}
\tilde{\tilde{k}}(y, y') 
& : = \tilde{k}(y, y') -  \tilde{k}_q(y) -  \tilde{k}_q(y') +  \tilde{k}_{qq}
    = \sum_k \tilde{\lambda}_k (\tilde{\psi}_k(y) + \tilde{v}_k)(\tilde{\psi}_k(y') + \tilde{v}_k) , 
\\
\tilde{k}_q(y)
& : = \int \tilde{k}(y, z)q(z)dz
    = - \sum_k \tilde{\lambda}_k  \tilde{v}_k \tilde{\psi}_k(y),  \\
\tilde{k}_{qq} 
& : = \E_{y\sim q} \tilde{k}_q(y)
    = \sum_k \tilde{\lambda}_k  \tilde{v}_k^2 .
\end{align}
We compute $\E \textcircled{\small{1}}  \textcircled{\small{2}} $ and $ \E \textcircled{\small{3}}  \textcircled{\small{4}} $ respectively:
Note that $\E_{x \sim p} k_1 (x) = 0$, thus
\[
\E \textcircled{\small{1}}  \textcircled{\small{2}}
= \frac{\tau}{ n_1^{3/2}} \sum_{i, j, j'} \E k_1(x_i) \tilde{k}(x_j, x_{j'})
= \frac{\tau}{\sqrt{n_1}} \E_{x \sim p} k_1(x) \tilde{k}(x, x)
\]
due to that $\E k_1(x_i) \tilde{k}(x_j, x_{j'})$ vanishes unless the three indices coincide.
By that $|k_1(x)| \le 4$ and $\tilde{k}(x, x) \le 4$,  we have that 
\begin{equation}
| \E \textcircled{\small{1}}  \textcircled{\small{2}}|
\le  16 \frac{\tau}{\sqrt{n_1}}.
\end{equation}
As for $ \E \textcircled{\small{3}}  \textcircled{\small{4}} $, by that $\E_{y' \sim q } \tilde{\tilde{k}}(y, y')  = 0$ for any $y$, 
a similar argument gives that 
\[
\E \textcircled{\small{3}}  \textcircled{\small{4}}
= \frac{\tau}{\sqrt{n_2}} \E_{y \sim q} ( k_1(y) -  \E_{ z \sim q} k_1 (z)  ) \tilde{\tilde{k}}(y, y).
\]
To proceed, by Cauchy-Schwarz,
\[
| \E_{y \sim q} ( k_1(y) -  \E_{ z \sim q} k_1 (z)  ) \tilde{\tilde{k}}(y, y) |^2
\le 
	\E_{y \sim q} ( k_1(y) -  \E_{ z \sim q} k_1 (z)  )^2
	\cdot	 \E_{y \sim q}  \tilde{\tilde{k}}(y, y)^2,
\]
where, same as in \eqref{eq:sumkl-ckcl-Slk-bound},
\[
\E_{y \sim q} ( k_1(y) -  \E_{ z \sim q} k_1 (z)  )^2
\le 
\E_{y \sim q}  k_1(y)^2 \le 16,
\]
and, 
by the uniform bound that $\tilde{\tilde{k}}(y, y) \le 16$,
we also have that
\begin{equation}\label{eq:Etildetildek(y,y)2-bound}
\E_{y \sim q}  \tilde{\tilde{k}}(y, y)^2
\le 16 \E_{y \sim q}  \tilde{\tilde{k}}(y, y)
= 16 (\E_{y \sim q}  \tilde{k}(y, y) - \tilde{k}_{qq})
\le 16 \E_{y \sim q}  \tilde{k}(y, y) \le 16 \times 4,
\end{equation}
Together, they give that
\[
| \E_{y \sim q} ( k_1(y) -  \E_{ z \sim q} k_1 (z)  ) \tilde{\tilde{k}}(y, y) |
\le \sqrt{ 16 \times 16 \times 4 } = 32.
\]
This means that
\begin{equation}
| \E \textcircled{\small{3}}  \textcircled{\small{4}}|
\le  32  \frac{\tau}{\sqrt{n_2}}.
\end{equation}
Back to \eqref{eq:(2)-fromula}, we have that
\[
|(2) |
\le 
(4 \sqrt{n})
   \left\{
     \frac{1}{\rho_{1,n}^{3/2}}  16 \frac{\tau}{\sqrt{n_1}}
    + \frac{1}{\rho_{2,n}^{3/2}}  32  \frac{\tau}{\sqrt{n_2}}
    \right\}
\le
4 \times 32 \tau \left(\frac{1}{\rho_{1,n}^{2}} + \frac{1}{\rho_{2,n}^{2}}  \right),
\]
namely \eqref{eq:varX-H1-2-bound}.

\vskip 0.1in
Proof of \eqref{eq:varX-H1-3-bound}:
Similar to \eqref{eq:EX2-H0-1},  
\begin{align*}
(3)
&=   \frac{1}{\rho_{1,n}\rho_{2,n}} \sum_{k,l}  \tilde{\lambda}_k \tilde{\lambda}_l (  S^{(n)}_{kk} + S^{(n)}_{ll} + 4  \delta_{kl} S^{(n)}_{kk} ) \\
& ~~~ ~~~ 
  +   \frac{1}{\rho_{1,n}^2} \E  \sum_{k,l} \tilde{\lambda}_k \tilde{\lambda}_l \hat{h}_k^2 \hat{h}_l^2 
 +   \frac{1}{\rho_{2,n}^2} \E  \sum_{k,l} \tilde{\lambda}_k \tilde{\lambda}_l \hat{g}_k^2 \hat{g}_l^2 \\
& := (3)_1 +   (3)_2 + (3)_3.
\end{align*}
We have that
\begin{equation}\label{eq:bound-3-1}
(3)_1
=  \frac{1}{\rho_{1,n}\rho_{2,n}}
	\left\{
    2 \sum_{k}  \tilde{\lambda}_k  S^{(n)}_{kk}   \cdot \sum_l \tilde{\lambda}_l 
    + 4  \sum_{k}  \tilde{\lambda}_k^2 S^{(n)}_{kk} 
    \right\},
\end{equation}
and by \eqref{eq:EX2-H0-term3-bound},
\begin{equation}\label{eq:bound-3-2}
(3)_2
\le \frac{1}{\rho_{1,n}^2} \left\{
	( \sum_k  \tilde{\lambda}_k )^2 
	+ 2  \sum_k  \tilde{\lambda}_k^2
	+ \frac{16}{n_1}  \right\}.
\end{equation}
Using a similar argument to derive \eqref{eq:EX2-H0-term3-moment},
one can verify that
\begin{align*}
(3)_3
= \frac{1}{\rho_{2,n}^2} 
	\E ( \sum_{k} \tilde{\lambda}_k \hat{g}_k^2)^2
= \frac{1}{\rho_{2,n}^2} 
	\frac{1}{n_2^2} \E \sum_{i, i', j, j'} \tilde{\tilde{k}}(y_i, y_{i'}) \tilde{\tilde{k}}(y_j, y_{j'}) \\
\le 	
 \frac{1}{\rho_{2,n}^2} 
 \left\{
 ( \E_{y \sim q} \tilde{\tilde{k}}(y,y) )^2
 + 
  2 \E_{y, y' \sim q} \tilde{\tilde{k}}(y,y')^2 
 + 
 \frac{1}{n_2} \E_{y \sim q} \tilde{\tilde{k}}(y,y)^2
 \right\}. 
\end{align*}
Meanwhile, 
\[
\E_{y \sim q} \tilde{\tilde{k}}(y,y) = \sum_k \tilde{\lambda}_k S^{(n)}_{kk},
\quad
 \E_{y, y' \sim q} \tilde{\tilde{k}}(y,y')^2  = \sum_{kl}\tilde{\lambda}_k  \tilde{\lambda}_l (S^{(n)}_{kl})^2
\]
and $\E_{y \sim q} \tilde{\tilde{k}}(y,y)^2 \le 64$ (c.f. \eqref{eq:Etildetildek(y,y)2-bound}).
This gives that
\begin{equation}\label{eq:bound-3-3}
(3)_3
\le 
\frac{1}{\rho_{2,n}^2} \left\{
	(\sum_k \tilde{\lambda}_k S^{(n)}_{kk})^2
	+ 2 \sum_{kl} \tilde{\lambda}_k  \tilde{\lambda}_l (S^{(n)}_{kl})^2
	+ \frac{64}{n_2}  \right\}.
\end{equation}

We also have that
\[
(4)
= \left( \sum_k \tilde{\lambda}_k \left( \frac{1}{\rho_{1,n}} +  \frac{1}{\rho_{2,n}} S^{(n)}_{kk} \right) \right)^2,
\]
together with \eqref{eq:bound-3-1}, \eqref{eq:bound-3-2}, \eqref{eq:bound-3-3},
this gives that
\[
(3) - (4)
\le
	\frac{4}{\rho_{1,n}\rho_{2,n}}
     \sum_{k}  \tilde{\lambda}_k^2 S^{(n)}_{kk} 
     + 
     \frac{2}{\rho_{1,n}^2} 
	 \sum_k  \tilde{\lambda}_k^2
     + \frac{2}{\rho_{2,n}^2} 
	\sum_{kl} \tilde{\lambda}_k  \tilde{\lambda}_l (S^{(n)}_{kl})^2
	+ \frac{16}{n_1}  \frac{1}{\rho_{1,n}^2} 
	+ \frac{64}{n_2}  \frac{1}{\rho_{2,n}^2} .
\]
Note that $ \sum_k \tilde{\lambda}_k S^{(n)}_{kk} = \E_{y \sim q} \tilde{\tilde{k}}(y,y) \le \E_{y \sim q} \tilde{k}(y,y) \le 4$,
and $\sum_k \tilde{\lambda}_k \le 4$, 
thus
\begin{align*}
\sum_{k}  \tilde{\lambda}_k^2 S^{(n)}_{kk}  
& \le 4 \sum_{k}  \tilde{\lambda}_k S^{(n)}_{kk} \le 4\times 4, 
\\
\sum_k  \tilde{\lambda}_k^2
& \le 16,
 \quad \text{( c.f. \eqref{eq:sumk-lambdak2-bound})}
\\
\sum_{kl} \tilde{\lambda}_k  \tilde{\lambda}_l (S^{(n)}_{kl})^2
& \le
\sum_{kl} \tilde{\lambda}_k  \tilde{\lambda}_l S^{(n)}_{kk} S^{(n)}_{ll}
= (\sum_{k}\tilde{\lambda}_k S^{(n)}_{kk})^2  \le 4^2,
\end{align*}
this gives that
\[
(3) - (4)
\le
16 \left( 
\frac{4}{\rho_{1,n}\rho_{2,n}}
+
 \frac{2}{\rho_{1,n}^2} 
+
\frac{2}{\rho_{2,n}^2} 
\right)
+ \frac{16}{n}  \left( \frac{1}{\rho_{1,n}^3}  + \frac{4}{\rho_{2,n}^3} \right)
\]
The first term equals $C_3$,
and the second term $< 0.1$ under the condition of the Theorem.
This proves \eqref{eq:varX-H1-3-bound}.
\end{proof}

\end{document}